%% file: icml2026_conference.tex
\documentclass{article}

\usepackage{microtype}
\usepackage{graphicx}
\usepackage{subcaption}
\usepackage{booktabs} %
\usepackage{multirow}

\usepackage{hyperref}

\usepackage[accepted]{icml2026}

\usepackage{amsmath}
\usepackage{amssymb}
\usepackage{mathtools}
\usepackage{amsthm}

\usepackage[capitalize,noabbrev]{cleveref}

\theoremstyle{plain}
\newtheorem{theorem}{Theorem}[section]
\newtheorem{proposition}[theorem]{Proposition}
\newtheorem{lemma}[theorem]{Lemma}

\theoremstyle{definition}
\newtheorem{definition}[theorem]{Definition}
\newtheorem{assumption}[theorem]{Assumption}
\theoremstyle{remark}

\usepackage[disable,textsize=tiny]{todonotes}

\icmltitlerunning{Gradient Regularization Mitigates Reward Hacking}

\usepackage{pgfplots}
\DeclareUnicodeCharacter{2212}{−}

\usepgfplotslibrary{groupplots,dateplot}
\usetikzlibrary{patterns,shapes.arrows,shapes.geometric,arrows.meta,calc,
positioning,decorations.pathreplacing,decorations.pathmorphing,plotmarks}
\pgfplotsset{compat=newest}

\pgfplotsset{
  every axis/.append style={
    tick label style={
        font=\footnotesize,
        /pgf/number format/1000 sep={},
    },
    label style={font=\footnotesize},
    },
}

\DeclareMathOperator{\expect}{\mathbb{E}}
\newcommand{\sharpset}{S_{K,\delta}}
\newcommand{\prew}{\widetilde{R}}
\usepackage{fvextra}

\usepackage{csquotes} 
\usepackage{dirtytalk}

\input{math_commands}

\usepackage[T1]{fontenc}

\usepackage{listings}
\usepackage{xcolor}
\renewcommand{\sectionautorefname}{Section}

\usepackage[cachedir=minted-cache]{minted}
\setminted{
  fontsize=\small,
  breaklines=true,
  linenos=false,
  frame=lines,
  framesep=6pt
}

\definecolor{noregcolor}{RGB}{31,119,180}
\definecolor{gradregcolor}{RGB}{255,127,14}
\definecolor{klcolor}{RGB}{44,160,44}
\definecolor{resetcolor}{RGB}{214,39,40}
\definecolor{sftcolor}{RGB}{148,103,189}

\begin{document}

\twocolumn[
    \icmltitle{Gradient Regularization Mitigates Reward Hacking \\ in
        Reinforcement Learning from Human Feedback and Verifiable Rewards}

    \icmlsetsymbol{equal}{*}

    \begin{icmlauthorlist}
        \icmlauthor{Johannes Ackermann}{utokyo,riken}
        \icmlauthor{Michael Noukhovitch}{mila,montreal}
        \icmlauthor{Takashi Ishida}{riken,utokyo}
        \icmlauthor{Masashi Sugiyama}{riken,utokyo}
    \end{icmlauthorlist}

    \icmlaffiliation{utokyo}{The University of Tokyo}
    \icmlaffiliation{riken}{RIKEN AIP}
    \icmlaffiliation{mila}{Mila}
    \icmlaffiliation{montreal}{Universit\'e de Montr\'eal}

    \icmlcorrespondingauthor{Johannes Ackermann}{ackermann@ms.k.u-tokyo.ac.jp}

    \icmlkeywords{Reinforcement Learning, Post-Training, Gradient Regularization, RLHF, Reasoning, RLVR, LLMs}

    \vskip 0.3in
]

\printAffiliationsAndNotice{}  %

\begin{abstract}
    Reinforcement Learning from Human Feedback (RLHF) or Verifiable Rewards (RLVR) are two key steps in the post-training of modern Language Models (LMs). A common problem is reward hacking, where the policy may exploit inaccuracies of the reward and learn an unintended behavior. Most previous works address this by limiting the policy update with a Kullback-Leibler (KL) penalty towards a reference model. We propose a different framing: Train the LM in a way that biases policy updates towards regions in which the reward is more accurate. First, we derive a theoretical connection between the accuracy of a reward model and the flatness of an optimum at convergence. Gradient regularization (GR) can then be used to bias training to flatter regions and thereby maintain reward model accuracy. We confirm these results by showing that the gradient norm and reward accuracy are empirically correlated in RLHF. 
    We then empirically show that Reference Resets of the KL penalty find flatter regions with a higher reward accuracy. We further improve on this by proposing to use explicit GR with an efficient finite-difference estimate. Empirically, GR performs better than a KL penalty across a diverse set of RL experiments with LMs. GR achieves a higher GPT-judged win-rate in RLHF, avoids overly focusing on the format in rule-based math rewards, and prevents hacking the judge in LLM-as-a-Judge math tasks. 
\end{abstract}

\section{Introduction}
\begin{figure*}[t]
    \centering
    \includegraphics[width=0.27\linewidth]{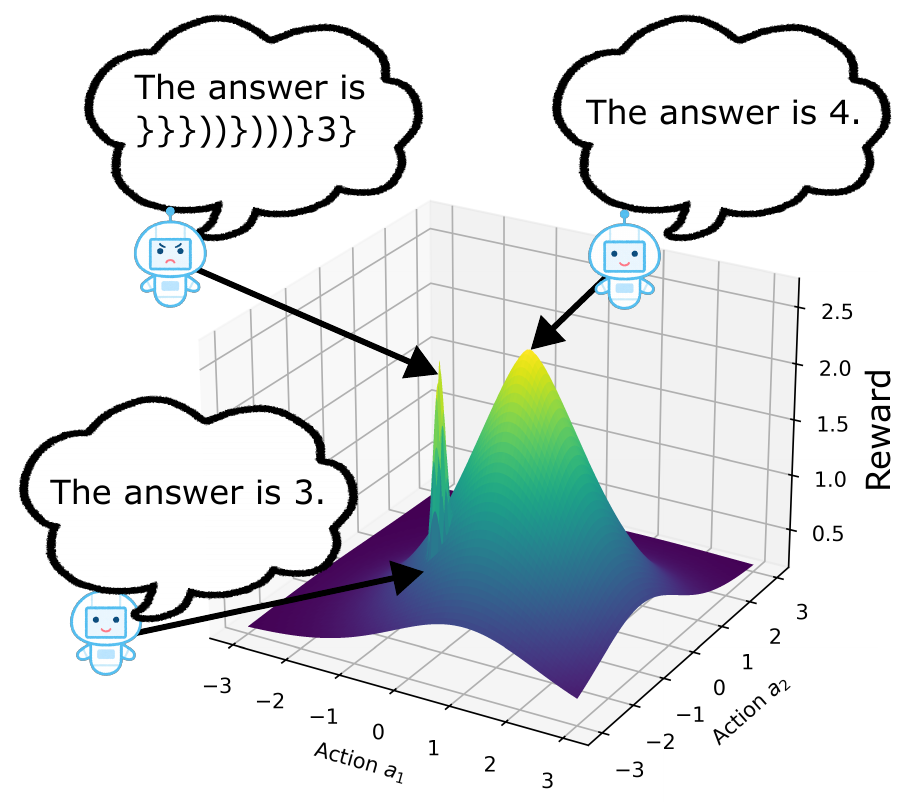}
    \hfill
    \newcommand{\figonewidth}{0.24\linewidth}
    \newcommand{\figoneheight}{4.5cm}
    \input{figures/llm_judge/llmjudge_threepanel_bothmethods.tex}
    \caption{We argue that reward hacking often corresponds to exploiting sharp maxima in action space, as illustrated by the conceptual figure (left).
        For example, an LLM judge may be confused and assign a high reward to a wrong answer with specific formatting.
        In the LLM-as-a-Judge training run shown on the right, the increase in gradient norm coincides with reward hacking, resulting in true reward collapsing.
        By using gradient norm regularization, we can prevent this issue and obtain a better model, as seen by the improved true reward.
        The examples show Qwen2.5-0.5B models trained on GSM8K with a Qwen2.5-1.5B-Instruct judge with access to the true answer.
    }
    \label{fig:figure2}
\end{figure*}
\input{intro_gradreg_rlhfvr}

\section{Background}

\subsection{Reinforcement Learning}
As is common in the RL post-training literature \citep{deepseek_grpo_2024,ahmadian_back_2024}, we consider an episode consisting of a single state $s$, representing the prompt, and a single action $a$, representing the reply.
We denote the state set by $\mathcal{S}$ and the action set by $\mathcal{A}$.
As policy $\pi_\phi$ with parameters $\phi$, we consider an autoregressive LM with conditional probability $\pi_\phi(a| s)=\prod_t\pi_\phi(\mathfrak{a}_t|s,\mathfrak{a}_{<t})$, where $\mathfrak{a}_t$ is the $t$-th response token, $\mathfrak{a}_{<t}$ refers to the response tokens before $\mathfrak{a}_t$ and $a$ refers to the entire response.
We assume that there is a true reward function $R^*: \mathcal{S} \times \mathcal{A} \to \mathbb{R}$ and our goal is to maximize the return of the policy $\pi_\phi$ on this reward $J^*(\phi)=\expect_{s\sim P(s), a\sim\pi_\phi}[R^*(s,a)]$.
However, we do not have access to $R^*$.
Instead, we use a proxy reward (PR) $\widetilde{R}$, and train the policy to maximize its return $J(\phi,\theta)=\expect_{s\sim P(s), a\sim\pi_\phi}[\prew_\theta(s,a)]$.
To optimize the return we use REINFORCE policy gradient updates \mbox{\citep{williams_simple_1992}}
\begin{equation}
    \begin{aligned}
        \nabla_\phi J(\phi,\theta)
         & =
        \\
        \expect_{s\sim P(s), a\sim \pi_\phi}
         & \left[(\prew_\theta(s,a)-b(s))\nabla_\phi\log\pi_\phi(a|s)\right] \,,
    \end{aligned}
    \label{eq:policygrad}
\end{equation}
where $b(s)$ is a baseline.
In this work, we use a variant of Group Relative Policy Optimization (GRPO) \citep{deepseek_grpo_2024} called GRPO Done Right (Dr.GRPO) \citep{liu_understanding_2025}, which removes normalization terms from GRPO.
As we do not do multiple updates per batch, Dr.GRPO simplifies to using REINFORCE with the baseline being the sample average over multiple actions drawn for the same state, i.e., $b(s)=\frac{1}{N}\sum_{i=1}^N \prew(s,a^i)$, with $a^i\sim\pi_\phi(a|s)$.
GRPO was originally presented with a KL penalty $D_\mathrm{KL}(\pi_\phi;\pi_{\phi^1})$, as commonly used in RLHF \citep{stiennon_learning_2020}.
This KL-penalty, weighted by a hyper-parameter $\beta$, is intended to keep the model close to the initial policy $\pi_{\phi^1}$, preventing reward hacking \citep{stiennon_learning_2020}, but is sometimes omitted in more recent works \citep{team_glm-45_2025, olmoteam_olmo_2025}.

\subsection{Proxy Rewards}
We consider three types of PRs:
Trained reward models $R^\theta$, rule based rewards $R^\mathrm{R}$, and LLM-as-a-Judge rewards $R^\mathrm{LM}$.

\paragraph{Reward Models}
(RMs) are commonly used in RLHF \citep{christiano_deep_2017,stiennon_learning_2020} to represent complex preferences which are hard to turn into rule-based rewards.
They generally assume the Bradley-Terry (BT) model of preference \citep{bradley_rank_1952} where the probability $P$ of preferring one option $a_1$ over another option $a_0$ is the logistic function ${\sigma(x)= 1 / (1+e^{-x})}$ of the difference of the true reward for each action, i.e.,
\begin{equation}
    P(a_1 > a_0\mid s) = \sigma\left(R^*(s,a_1) - R^*(s,a_0)\right) \,.
    \label{eq:bradleyterry}
\end{equation}
Pairs of responses $(a_0,a_1)$ for each prompt $s$ are collected from an initial model $\pi_{\phi^1}$.
Human annotators then choose their preferred (winning) reply $a_\mathrm{w}$ and not preferred (losing) reply $a_\mathrm{l}$, creating a dataset $\mathcal{D}_\mathrm{RM}=\{(s^j,a_\mathrm{w}^j,a_\mathrm{l}^j)\}_{j=1}^{{N_\mathrm{RM}}}$.
We can then use the BT assumption to train an RM $R_\theta : \mathcal{S} \times \mathcal{A} \to \mathbb{R}$ with parameters $\theta\in\Theta$ by minimizing the cross-entropy loss based on the BT model:
\begin{equation}
    \begin{aligned}
         & \mathcal{L}_\mathrm{BT}(\theta,\phi)
        =                                                             \\
         & -\expect_{(s,a_\mathrm{w},a_\mathrm{l}) \sim P_{\pi_\phi}}
        [\log\sigma(R_\theta(s,a_\mathrm{w})-R_\theta(s,a_\mathrm{l}))] \,,
    \end{aligned}
    \label{eq:bradleyterryCrossentropyLoss}
\end{equation}
where $P_{\pi_\phi}=P(s)\pi_\phi(a_0|s)\pi_\phi(a_1|s)P(a_1>a_0|s)$ is the probability of the $(s,a_\mathrm{w},a_\mathrm{l})$ triplet under the policy $\pi_\phi$.
To train the RM we minimize $\mathcal{L}_\mathrm{BT}(\theta,\phi^1)$ and the expectation is replaced by the sample average from $\mathcal{D}_\mathrm{RM}$.

\paragraph{Rule-Based Rewards}
\citep{havrilla_teaching_2024} are deterministic checks whether a final answer matches the ground truth.
To encourage reasoning, and to allow a discrimination of a final answer against occurrences of the answer during reasoning, rule-based rewards typically require the answer to follow a specific format, for example using the LaTeX tag \verb+\boxed{}+ \citep{hendrycks_measuring_2021} or the HTML tags \citep{deepseek-ai_deepseek-r1_2025} \verb+<think>...</think><answer>...</answer>+.
Thus, oftentimes one reward term $R^\mathrm{F}$ checks whether the format matches and another reward term $R^\mathrm{C}$ checks for correctness, to give the rule-based reward $R^\mathrm{R}=R^\mathrm{C}+R^\mathrm{F}$.

\paragraph{LLM-as-a-Judge}
\citep{zheng_judging_2023} prompts an LLM and uses its textual output to check whether an answer is correct. Frequently designing a rule-based reward can be challenging due to many possible accurate solutions, e.g., the correctness of a proof or many equivalent ways to write a math answer.
Instead, prompting an LLM-as-a-Judge with a description of scoring criteria, the question, and the correct answer (if available) can allow us to capture solutions more robustly.
LLM-as-a-Judge can also allow for more complicated rewards with a combination of objective (e.g., correct reply) and subjective (e.g., clear reasoning) criteria.

\subsection{Gradient Regularization (GR)}
Flat minima of a loss function $\mathcal{L}(\phi)$ are connected to better generalization in supervised learning \cite{hochreiter_flat_1997,foret_sharpness-aware_2021}, i.e., a smaller difference between the population/test loss $\mathcal{L}(\phi)=\E_{(x,y)\sim P(x,y)}[\ell(f_\phi(x),y)]$ and
its empirical approximation
on a finite training dataset $\mathcal{D}$, consisting of i.i.d. input-label pairs $(x,y)\sim P(x,y)$, for a loss function $\ell$.
A way to obtain flat minima is by regularizing the gradient norm of the objective $\mathcal{L}$, i.e., the squared Euclidean norm $\|\nabla_\phi \mathcal{L}(\phi)\|^2$  \citep{zhao_penalizing_2022}.
Adding this term to our loss function, we need to calculate its gradient, which can be approximated with a parameter perturbation \citep{karakida_understanding_2023}
\begin{equation}
    \Delta_\phi\|\nabla_\phi \mathcal{L}(\phi)\|^2 =
    \frac
    {\nabla_\phi \mathcal{L}(\phi+\varepsilon\nabla_\phi \mathcal{L}(\phi)) - \nabla_\phi \mathcal{L}(\phi)}
    {\varepsilon}\,.
    \label{eq:gradreg_fd_estimate}
\end{equation}
The model parameters are then updated as
\begin{equation}
    \phi \gets \phi - \eta \nabla_\phi \mathcal{L}(\phi) - \eta\frac{\gamma}{2} \Delta_\phi\|\nabla_\phi \mathcal{L}(\phi)\|^2\,,
\end{equation}
where $\eta\in\R^+$ is the learning rate and $\gamma\in\R$ is a hyper-parameter controlling the strength of the GR and $\varepsilon$ controls the strength of the parameter perturbation.

\section{Accurate Proxy Rewards via Gradient Regularization}
\label{sec:gradreg_theory}

\subsection{Problem Formulation and Overview}
As mentioned above, the goal of RL is to learn a policy $\pi_\phi$ that maximizes the expected true reward $R^*$.
As we do not have access to $R^*$, we instead have to use the PR $\prew$ to update our policy.
As the PR is generally an approximation of the true reward $R^*$, it may be prone to reward hacking and we need to ensure that the PR stays accurate during training.
For PRs that have been trained or designed based on samples from an initial policy $\pi^1$, such as RMs in RLHF, the most common solution is to use a KL penalty \cite{stiennon_learning_2020,deepseek_grpo_2024} ensuring that the policy stays close to $\pi^1$ and thus the PR stays accurate.
However, the KL penalty also limits how much the policy can learn.
\begin{figure*}[t]
    \centering
    \input{figures/theory_tikz_accurate}
    \caption{Conceptual illustration of our theoretical argument: (left) Regularizing the gradient norm biases optimization toward flat basins in parameter space, and (right) under action-smoothness, a flat maximum makes $\delta$-close pairs unlikely to have a reward gap larger than $K$, i.e., decreases the probability of overly sharp action pairs $a_1,a_2:\|a_1-a_2\|\le\delta,|\prew(s,a_1)-\prew(s,a_2)|>K$.
    Under the assumption of a Lipschitz-continuous true reward $R^*$, each such pair implies an incorrect proxy reward $\prew$.
    }
    \label{fig:theory_illust}
\end{figure*}
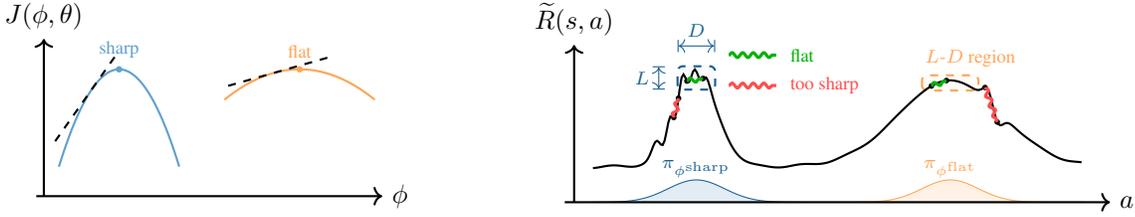
Instead of changing the PR $\prew$ or constraining $\pi_\phi$, we aim to update the policy $\pi_\phi$ such that it obtains a high reward provided by the PR $\prew$, while also biasing the policy update towards regions in which the PR $\prew$ is accurate.
Our goal is thus
\begin{equation}
    \max_\phi \E_{a\sim\pi_\phi} [\prew(s,a)]  - \gamma \mathcal{L}_\mathrm{BT}(\theta,\phi) \,,
\end{equation}
where $\gamma>0$ is a hyper-parameter and we use the BT loss $\mathcal{L}_\mathrm{BT}(\theta,\phi)$ of a PR parameterized with $\theta$ on actions sampled from $\pi_\phi$.
However, we cannot directly evaluate the BT loss $\mathcal{L}_\mathrm{BT}(\theta,\phi)$, as it requires pairwise comparisons for actions drawn from $\pi_\phi$.
As we will show below, under some assumptions, overly sharp optima of the objective $J(\phi,\theta)$ correspond to overly sharp maxima of the PR $\prew$, which imply an excess BT loss $\mathcal{L}_\mathrm{BT}$ of the PR.
Regularizing the policy gradient norm during training biases the optimization towards flat optima \cite{zhao_penalizing_2022}, avoiding this problem.
Instead of $\mathcal{L}_\mathrm{BT}(\theta,\phi)$, we thus use the policy gradient norm $\|\nabla_\phi J(\phi,\theta)\|^2$, yielding the practically optimizable
\begin{equation}
    \max_\phi \E_{a\sim\pi_\phi} [\prew(s,a)]  - \gamma \left\|\nabla_\phi J(\phi,\theta)\right\|^2 \,.
\end{equation}

\subsection{Gradient Regularization Improves Proxy Reward Accuracy}
Our argument consists of three steps:
GR biases optimization towards flat maxima.
Flat maxima imply pairwise robust policies.
Pairwise robust policies correspond to accurate PRs, under the assumption of a flat reward.
We provide an illustration in \autoref{fig:theory_illust}.
We note that our argument in this section focuses on continuous action spaces, while LMs use discrete action spaces. 
We discuss this discrepancy at the end of this section and provide experimental evidence in the LM setting in \autoref{sec:ref_resets} and \autoref{sec:explicit_gradreg}.

\paragraph{GR Favors Flat Maxima in Parameter Space}
It is well known that GR favors flat minima in the parameter space \citep{barrett_implicit_2021,karakida_understanding_2023}.
One way to see this is the argument of \citet{zhao_penalizing_2022}, connecting the gradient norm to Lipschitz continuity, which we reproduce for completeness in Appendix \ref{app:gradreg_lipschitzness}.
It is also well known in the supervised learning literature \citep{hochreiter_flat_1997} that flat minima lead to better generalization, i.e., a smaller difference between the population loss $\mathcal{L}(\theta)$ and the loss on the training set $\widehat{\mathcal{L}}(\theta)$.
Equivalently, in RL, we would directly expect GR to ensure a similar PR score obtained on the training prompts $\widehat{J}(\phi,\theta)=\E_{s\sim D_\mathrm{tr},a\sim\pi_\phi}[\prew_\theta(s,a)]$ and on the test prompts $J(\phi,\theta)=\E_{s\sim D_\mathrm{te},a\sim\pi_\phi}[\prew_\theta(s,a)]$.
Instead of considering generalization, we consider the PR accuracy, or more precisely the BT loss $\mathcal{L}_\mathrm{BT}$, necessitating the next two steps.

\paragraph{Flat Maxima in Parameter Space imply Robust Maxima in Action Space}
We need to connect the flatness of an optimum in parameter space with its flatness in action space, often referred to as robustness.
\citet{lee_flat_2025} previously investigated this connection under the assumption of a ball with constant reward, which we use as a starting point.
While they assume a constant return on a ball $B(\phi^*,\mathcal{E})\coloneq \{\epsilon:\|\epsilon\|\leq\mathcal{E}\}$ around the optimum $\phi^*$, we allow the reward to decrease by at most $\hat{L}$:
\begin{definition}[$\mathcal{E}-\widehat{L}$ flat reward maximum]
    For a reward function $R(s,a)$ and policy $\pi_\phi(a|s)$, parameterized by $\phi$, a maximum $\phi^*$ is $\mathcal{E}-\widehat{L}$-flat if the following holds:
    \begin{equation}
        \begin{aligned}
             & \text{For all }\epsilon\in\mathbb{R}^m\text{ s.t. }\|\epsilon\|\le \mathcal{E},
            \\
             & \E_{a \sim \pi_{\phi^*}(s)}
            \left[
                \prew(s,a)
                \right]
            -
            \E_{a \sim \pi_{\phi^*+\epsilon}(s)}
            \left[
                \prew(s,a)
                \right]
            \leq \widehat{L}
        \end{aligned}
    \end{equation}
    \label{def:e-l-flatrew}
\end{definition}

We also define the concept of a $(\delta,K,\rho)$-pairwise robust policy, which measures the probability of action pairs $(a_1,a_2)$ sampled from the policy violating $K/\delta$-Lipschitz-continuity, i.e., $|\prew(s,a_1)-\prew(s,a_2)|>K,\|a_1-a_2\|<\delta$:
\begin{definition}[$(\delta,K,\rho)$-pairwise robust policy]
    \label{def:pairwise-robust}
    For a policy $\pi_\phi(a|s)$ and a reward $R(s,a)$, define the sharpness set $S_{K,\delta}(R):=\{(s,a_1,a_2):\|a_1-a_2\|\leq\delta,|R(s,a_1)-R(s,a_2)|>K\}$, i.e., the set of $\delta$-close action pairs for which the reward changes by more than $K$.
    A policy $\pi_\phi$ is $(\delta,K,\rho)$-pairwise robust for a reward $R$ if
    \begin{equation}
        P(S_{K,\delta}(R)|\pi_\phi) \leq \rho\,,
    \end{equation}
    with $P(S_{K,\delta}(R)|\pi_\phi)$ being the probability under $s\sim P(s)$ and $(a_1,a_2)\stackrel{\mathrm{i.i.d.}}{\sim}\pi_\phi(a|s)$.
    \label{def:deltakrho-robustpol}
\end{definition}

We obtain the following proposition, linking flatness in parameter space to $(\delta,K,\rho)$ in action space, under the assumption of a $\beta$-smooth PR:

\begin{proposition}[ $\mathcal{E}-\widehat{L}$ flat reward implies $(\delta,K,\rho)$ robust policy]
    Assume a Gaussian policy with fixed covariance $\Sigma$, where we denote the policy noise $Z\sim\mathcal{N}(0,\Sigma)$, and that $\prew(s,a)$ is $\beta$-smooth in $a$.
    Further, $\mathfrak{J}(\phi^*) := \nabla_{\phi}\,\mu_{\phi}(s)\bigm|_{\phi=\phi^*}$ is the Jacobian matrix of the mean action $\mu_{\phi}(s)$.
    If $\phi^*$ is an $\mathcal{E}-\widehat{L}$ flat reward maximum, then the PR action gradient $\|\nabla_a\prew(s,\mu_\phi(s))\|$ is bounded by
    \[
        G := \frac{\widehat{L}}{D^*} + \frac{\beta}{2}D^* + \beta\,\E[\|Z\|]\,,
    \]
    with radius $D^* \le \|\mathfrak{J}(\phi^*)\| \,\mathcal{E} +\mathcal{O}(\mathcal{E}^2)$.
    For a given $\delta>0$ and $K>0$, with $K/\delta>G$, we then know that no pairwise violations can occur within the radius $r:=\frac{1}{\beta}\bigl(\frac{K}{\delta}-G\bigr)$, and thus the policy $\phi^*$ is $(\delta,K,\rho)$-robust with
    \[
        P(S_{K,\delta}(\prew)\mid \pi_{\phi^*}) \leq 2\,P(\|Z\|>r) \coloneq \rho\,.
    \]
    \label{prop:flat_return_implies_deltarobust}
\end{proposition}
The proposition follows first linking flatness in parameter space to flatness in action space \citep{lee_flat_2025}.
Then, under a $\beta$-smooth PR, by a gradient bound we can ensure the absence of non-Lipschitz action pairs within a radius $r$ around the mean.
Thus a violation requires at least one of the actions to fall outside this region $\Pr(\|Z\|>r)$ and the final result follows from a union bound.
A full derivation is shown in Appendix \ref{app:proofs}.
Intuitively, for a given maximum, as the sensitivity to disturbances $\hat{L}$ increases, the probability of overly sharp actions $P(S_{K,\delta}|\pi_\phi)$ increases as well.
As the radius of the robustness $\mathcal{E}$ increases, $P(S_{K,\delta}|\pi_\phi)$ decreases or stays constant, as we can freely pick a smaller $E'<E$.
Thus, flatter, wider minima decrease the risk of sharp action pairs.
Next we will show that a larger $P(S_{K,\delta}|\pi_\phi)$ incurs a larger excess BT loss $\mathcal{L}_\mathrm{BT}$.

\paragraph{Non-Robust Policies Imply Inaccurate Proxy Reward} %
To connect $(\delta,K,\rho)$-robustness and the BT loss, we make the assumption of an $L$-Lipschitz true reward, $|R^*(s,a_1) - R^*(s,a_2)| \leq L \|a_1-a_2\|$. We then obtain
\begin{proposition}
    For prompts $s\sim P(s)$, pairs of actions $(a_1,a_2)$, $L$-Lipschitz true reward function $R^{*}$, proxy reward $\prew$, policy $\pi_\phi$, and $K>L\delta$, the excess BT loss can be lower bounded as
    \begin{equation}
        \mathcal{L}_\mathrm{BT}(\prew) - \mathcal{L}_\mathrm{BT}(R^*) \geq 2 (\sigma(K) - \sigma(L\delta))^2 P(\sharpset|\pi_\phi) \,.
    \end{equation}
\end{proposition}
The proof is shown in Appendix \ref{sec:app:btsharpnessbound}.
A non-robust policy in action space at least incurs an excess BT loss proportionate to the probability of overly sharp pairs $P(S_{K,\delta})$ and the magnitude of the sharpness $(\sigma(K) - \sigma(L\delta))^2$.
By changing the policy $\phi$ to decrease the ratio of violating pairs $P(S_{K,\delta})$ or the magnitude of the violations $K$, we obtain a policy that induces a smaller excess BT loss.
As shown above, we can bias the policy updates towards such policies with GR.

\paragraph{Limitations}
While GR itself can be applied to LMs regardless of whether the action space $\mathcal{A}$ is discrete or continuous, our theoretical argument assumes a Gaussian policy and requires a distance function, which is difficult to define for LMs.
As we assume the true reward to be Lipschitz under this distance, a distance under a representation that captures semantic closeness $\phi(a): \mathcal{A} \to \R^d$, such as the hidden space of the LM we are training, would be an appealing option.
For example in the illustrative Figure \autoref{fig:figure2} (left), the distance in action corresponds to semantic similarity, not formatting.
Further, we only address excess BT loss incurred by overly sharp maxima, i.e., overly sharp PR maxima.
We do not show whether or not GR prevents convergence to flat but incorrect regions of the PR.

To empirically validate our theory, we next show that implicit GR, via Reference Resets, can improve RLHF. In \autoref{sec:explicit_gradreg} we leverage explicit GR for further improvements.

\section{Reference Reset as Gradient Regularization}
\label{sec:ref_resets}
\begin{figure*}[t]
    \begin{minipage}{0.67\linewidth}
        \centering
        \input{figures/reset_rlhf/combined_reward_sharpness_btloss_gradnorm}
        \captionof{figure}{\textbf{When gradient norm decreases in a reset iteration, so do sharpness, and BT Loss $\mathcal{L}_\mathrm{RM}$.} Evolution of reward, sharpness and BT loss during training on TL;DR with Pythia 1B using GRPO+reference resets, resets shown as grey dashed lines. After initially spiking in an iteration, gradient norm decreases along with the sharpness of the parameters and the BT-loss under the current policy. We show moving averages over 30 steps. 
        }
        \label{fig:refresetbehavior}
    \end{minipage}
    \hfill
    \begin{minipage}{0.3\linewidth}
        \centering
        \input{figures/reset_rlhf/klsweep_kl_div_gold.tex}
        \captionof{figure}{\textbf{Reference Resets outperform all possible weights $\beta$ of KL penalty.} Oracle evaluation (Gold Model Score) vs KL from initial model for Pythia 1B on the TL;DR test set.%
        }
        \label{fig:sched_comp}
    \end{minipage}
\end{figure*}

Instead of performing explicit GR, we can rely on the implicit GR inherent to stochastic gradient descent \citep{barrett_implicit_2021}. We propose to leverage Reference Resets \citep{liu_prorl_2025} where the KL penalty is changed by iteratively resetting its reference to the current policy $\pi_\phi$ during training i.e., the update is penalized with $D_\mathrm{KL}(\pi_\phi;\pi_{\phi'})$, where every $R$ steps we set $\phi'\gets\phi$.
Since implicit GR usually occurs during the later stages of training, we find Reference Resets with a sufficient number of gradient steps per iteration to be an effective way of obtaining flat maxima.
While \citet{liu_prorl_2025} reset the policy when the reward has stagnated, we instead choose to perform resets every $R$ gradient steps similar to prior work using resets in RL \citep{nikishin_primacy_2022,noukhovitch_language_2023}.
We find that it is important to train beyond reward stagnation, as the gradient norm decreases significantly only after stagnation.

\paragraph{Setup} We experiment on the well-known TL;DR summarization task \citep{stiennon_learning_2020} of Reddit posts with human summaries. %
Following \citep{gao_scaling_2023,tang_understanding_2024} we make this a controlled synthetic setup, where the preference and evaluation data is relabeled by a ``gold'' reward model, Skywork-RewardLlama-3.1-8B-v0.2 \citep{liu_skywork-reward_2024}.
Compared to noisy human preferences, this setup gives us an oracle that enables consistent evaluations. We run experiments with models from both Pythia \citep{biderman_pythia_2023} and Qwen 2.5 \citep{qwen_qwen25_2025} families for different scales. See full details in Appendix~\ref{app:rlhf-details}.

\paragraph{Gradient Norm Tracks Sharpness and RM Accuracy}
To test our theory of GR, we plot the gradient norm against three important empirical values: 1) The training reward from the RM. 2) The sharpness of the current policy parameters $\phi$, which we predict is tied to reward-hacking. The sharpness is estimated by sampling 32 perturbations $\{\epsilon_i\}_{i=1}^{32}$ and evaluating $S(\phi,\theta) = \max_i J(\phi, \theta) -J(\phi+\epsilon_i,\theta)$, for each checkpoint. And 3) the BT loss $\mathcal{L}_\mathrm{BT}(\theta,\phi)$, which represents how accurate our reward model is for our current policy. We sample completions from our model, label with the gold RM, and get the BT loss under our training RM. In this way, we can empirically check whether we are training in a regime where our PR is accurate.
We train a Pythia 1B model with GRPO+Reference Resets and show results in \autoref{fig:refresetbehavior} with dashed vertical lines representing resets. 
In each iteration the gradient norm initially spikes and the reward increases quickly.
After the reward stabilizes, the gradient norm decreases. With it the sharpness of the parameters and the BT loss also decrease. This demonstrates that the gradient norm is tied to both the sharpness and accuracy of the PR. 
Continuing training after the reset, we now start out with a more accurate RM. This enables training in a regime with a good PR, leading to a better final policy.

\begin{table}
    \begin{center}
        \caption{Win rate vs reference response on the TL;DR summarization task, as judged by the gold RM.}
        \label{tab:baselineComparisonTLDR}
        \resizebox{\columnwidth}{!}{
            \begin{tabular}{l c c c c }
                \toprule
                Model Family      & \multicolumn{2}{c}{Pythia} & \multicolumn{2}{c}{Qwen 2.5}                                      \\
                \cmidrule(lr){2-3} \cmidrule(lr){4-5}
                Size              & 1B                         & 2.8B                         & 0.5B             & 1.5B            \\
                \midrule
                SFT               & 17.8\%                     & 25.6\%                       & 16.4\%           & 28.1\%          \\
                DPO               & 45.9\%                     & 68.0\%                       & 48.5\%           & 71.6\%          \\
                TR-DPO            & 45.5\%                     & 66.6\%                       & 49.7\%           & 73.4\%          \\
                \midrule
                GRPO              & 62.2\%                     & 68.5\%                       & 54.1\%           & \textbf{82.5\%} \\
                GRPO + Ref. Reset & \textbf{78.1\%}            & \textbf{76.4\%}              & \textbf{ 70.2\%} & 82.0\%          \\
                \bottomrule
            \end{tabular}
        }
    \end{center}
\end{table}

\paragraph{Reference Resets Outperform Standard KL}
We run an extensive comparison across both families and two sizes per model on standard baselines. 
The initial models are SFT-trained on TLDR data. 
From there we compare our method to DPO \citep{rafailov_direct_2023}, DPO with reference reset, also known as Trust Region DPO (TR-DPO) \citep{gorbatovski_learn_2025}, and standard GRPO with a fixed KL penalty. 
Our results are shown in \autoref{tab:baselineComparisonTLDR}.
GRPO + Reference Resets strongly outperforms all other methods for all but one setup.
Notably, a Pythia 1B model trained with Reference Resets performs better than a Pythia 2.8B model trained with standard GRPO, even with a finely-tuned $\beta$, as shown in Appendix \ref{app:expdetails}. 
In our experiments, TR-DPO's Reference Resets do not seem to afford the same performance improvements.
We speculate that this is because DPO does not use an RM, thus there can be no sharp action-space PR maxima which GR would help to avoid.

\paragraph{Is a Scheduled or Weaker $\beta$ a Sufficient Alternative?}
Originally, Reference Resets were proposed \citep{liu_prorl_2025} not to improve RM accuracy but to prevent the KL term from dominating the reward sum.
If this was their main mechanism, decreasing the strength of the KL penalty $\beta$ should have a similar effect.
Further, an analysis of the optimal policy under iterated KL constrained optimization (see Appendix \ref{sec:asymptotic_analysis}) shows that the optimal policy for $i$ reset iterations is equivalent to the optimal policy with a KL constraint to $\pi^1$ with a lower penalty strength $\beta' = \beta/i$.
However, \autoref{fig:sched_comp} shows that decreasing $\beta$ is not able to match Reference Resets.
This demonstrates the necessity for our novel insight that Reference Resets change the optimization dynamics via implicit gradient regularization during training and its effects on the PR accuracy.

\section{Explicit Gradient Regularization}
\label{sec:explicit_gradreg}
Reference Resets are an indirect way of regularizing the gradient norm, require many more gradient steps, and do not provide a direct, controllable way to trade-off reward maximization, adherence to the initial policy, and gradient regularization. 
We therefore propose a novel application of explicit GR methods to RLHF and RLVR, specifically finite-difference GR \citep{karakida_understanding_2023}.
To improve training stability, we implement parameter perturbations only on the transformer blocks, leaving the embedding layer and output head untouched, and clip the intermediate gradients.
To make GR training efficient, we reuse the actions $a\sim \pi_\phi(a|s)$ to calculate both the gradients $\nabla J(\phi,\theta)$ and $\nabla J(\phi + \varepsilon \nabla J(\phi,\theta),\theta)$.
In principle, this would require new actions $a\sim \pi_{\phi + \varepsilon \nabla J(\phi,\theta)}(a|s)$ or correction by importance sampling. However, we empirically found this to be unnecessary and reusing actions reduces computation overhead.\footnote{In our MATH reasoning experiments on 8 GH200 GPUs, generating the actions took on average 7.4s, while the policy update took 60ms without GR and 150ms with GR.}
PyTorch-style pseudocode, implementation details and experimental details are shown in Appendix \ref{app:expdetails}.
Our implementation is available at \url{https://github.com/JohannesAck/gradientregularization_trl}.

\subsection{GR Mitigates Hacking Reward Models in RLHF}
\begin{table}
    \begin{center}
        \caption{Win rate vs reference response on AlpacaFarm dataset, Qwen 2.5, judged by GPT4.1-Nano with early stopping. The 1.5B experiment is run for three different SFT policies and RMs.}
        \label{tab:gradreg_alpaca_qwen}
        \begin{tabular}{l c c c c c}
            \toprule
            Model Size       & 0.5B            & 1.5B            & 3B              \\
            \midrule
            SFT Model        & 12.8\%          & 20.6\%          & 26.3\%          \\
            \midrule
            No Reg           & 12.8\%          & 21.7\%          & 44.2\%          \\
            KL Reg           & 16.9\%          & 27.6\%          & 52.8\%          \\
            Reference Resets & 17.4\%          & 27.1\%          & 49.2\%          \\
            GR     & \textbf{18.5\%} & \textbf{29.2\%} & \textbf{59.2\%} \\
            \bottomrule
        \end{tabular}
    \end{center}
\end{table}

We first investigate whether GR can fully replace the standard KL penalty in RLHF.
Scaling up from TL;DR, we run experiments on the AlpacaFarm dataset \citep{dubois_alpacafarm_2023}, with preference feedback from GPT 4.1-Nano.
We again evaluate with winrate: generating completions on the test set and judging them against reference completions with GPT4.1-Nano.
We train models from the Qwen 2.5 family, running for 1000 gradient steps in the 1.5B experiments and 500 gradient steps in the 0.5B and 3B experiments and early-stop based on the training winrate.
We evaluate and ablate GR in four settings: 1) no KL penalty, no GR 2) KL penalty, 3) Reference Resets + KL, 4) GR without KL.
We do a grid-search to find the optimal strength for each regularization method for each model size and show results in \autoref{tab:gradreg_alpaca_qwen}.
In most cases using RL without regularization only yields a modest improvement over the initial SFT policy. RL with a KL penalty works decently well though 
Reference Resets is sometimes better. But explicit GR consistently performs best, demonstrating that it can replace the KL penalty and improve overall performance.

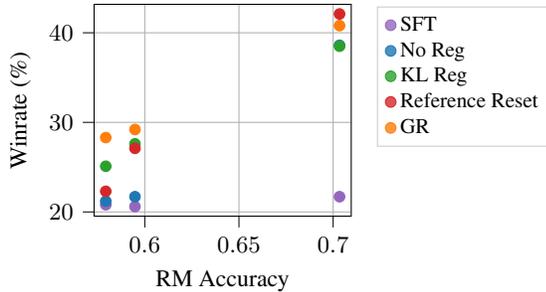
\begin{figure}
    \centering
    \input{figures/expl_gradreg/rmAccByModelSizeNoRegVSsftProperAx2}
    \caption{\textbf{Explicit GR performs well even with inaccurate RMs.} RM accuracy on SFT data vs GPT 4.1 Accuracy for different SFT and RM models, corresponding to different random seeds for full RLHF pipeline. The x-axis scale is nonlinear.}
    \label{fig:varied_rm_acc_vs_gptwinrate}
\end{figure}

\paragraph{Explicit GR is robust to RM accuracy} RLHF performance can vary heavily depending on the accuracy of the RM and performance of the initial policy \citep{huang_n_2024}.
To demonstrate robustness, we rerun our whole pipeline two more times for the 1.5B model: initial SFT, dataset sampling, RM training, GRPO with hyper-parameter tuning.
In \autoref{fig:varied_rm_acc_vs_gptwinrate} we show the winrate of each trained model against the accuracy of the RM with which it trained.
We observe that GR performs significantly better than Reference Resets or a KL penalty when the RM is weaker, demonstrating better robustness. Reference Resets do perform slightly better than GR with the strongest RM, where reward-hacking is less prevalent.

We also find that GR is robust to choices of hyper-parameters $\gamma$ and $\varepsilon$, and show a learning rate sweep in Appendix \ref{app:hparam_robustness}.
Finally, we discover that strong GR can even compensate in robustness for a weak RM by training in a regime where the RM is more accurate than on its training distribution, see Appendix \ref{app:gradient_reg_bt_loss}.

\subsection{GR Prevents Focus on Easy Rule-Based Rewards}
Recent work in RLVR has generally removed the KL penalty to allow for a stronger deviation from the base model \citep{team_glm-45_2025,olmoteam_olmo_2025}, though  others have kept it for training stability \cite{kimi_team_kimi_2025}.
As GR does not constrain the divergence from the base model, but may provide the desired training stability, we investigate whether it can be used in RLVR to enable more flexibility while preventing reward hacking.
\begin{table}
    \begin{center}
        \caption{Test accuracies on GSM8K after training with GRPO and different regularization methods, with LLM judge or rule-based reward. 0.5B experiments are trained with three random seeds each, shown are the mean and standard deviation.}
        \label{tab:reasoning_qwen_results}
        \resizebox{\columnwidth}{!}{
            \begin{tabular}{l c c  c c }
                \toprule
                Feedback-Type & \multicolumn{2}{c}{Rule-Based} & \multicolumn{2}{c}{LLM Judge}                                     \\
                \cmidrule(lr){2-3}                    \cmidrule(lr){4-5}
                Qwen 2.5 Size & 0.5B                           & 1.5B                          & 0.5B            & 1.5B            \\
                \midrule
                Base model    & 3.0\%                          & 37.5\%                        & 3.0\%           & 37.5\%          \\
                \midrule
                No Reg        & $52.1\pm0.3$\%                         & 72.9\%                        & $20.7\pm3.8$\%           & 1.9\%           \\
                KL Reg        & $44.2\pm1.2$\%                         & 72.4\%                        & $24.2\pm1.7$\%          & 55.5\%          \\
                GR  & $\mathbf{56.2\pm1.0}$\%                & \textbf{75.7}\%               & $\mathbf{41.0\pm2.8}$\% & \textbf{67.8}\% \\
                \bottomrule
            \end{tabular}
        }
    \end{center}
\end{table}
\begin{figure}
    \centering
    \newcommand{\reasonwidth}{0.45\columnwidth}
    \newcommand{\reasonheight}{4.5cm}
    \input{figures/reasoning/gradreg_reason_accuracy.tex}
    \hfill
    \input{figures/reasoning/gradreg_reason_formatrew.tex}
    \caption{\textbf{GR prevents overly focusing on formatting reward.} Qwen2.5-0.5B on GSM8K, test set accuracy (left) and formatting reward (right), the dashed line shows the optimal formatting reward. Without regularization, the policy focuses overly on the formatting reward, resulting in worse accuracy.}
    \label{fig:qwen05Bgsm8k}
\end{figure}
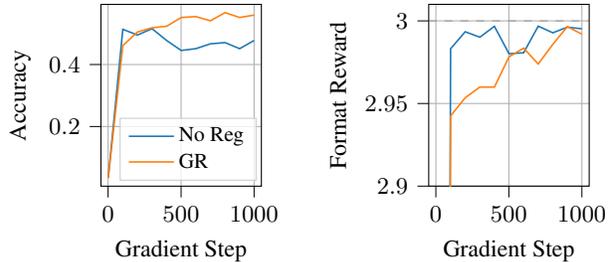

We perform experiments with Qwen 2.5 0.5B-Instruct and 1.5B-Instruct on GSM8K \citep{cobbe_training_2021} with the standard combination of formatting and correctness reward, see full details in Appendix \ref{app:expdetails}.
We indeed observe more stable training and improved final accuracy, as shown in \autoref{fig:qwen05Bgsm8k} (left) and \autoref{tab:reasoning_qwen_results}.
Notably, improved performance comes at the expense of a slightly worse adherence to the formatting reward.
In the presence of both rewards, we can see the excessive focus on the easier formatting reward as a kind of reward hacking, even when neither reward is hackable on its own.
This demonstrates how GR can be effective in situations with a combination of rule-based rewards.

\begin{table}
    \centering
    \definecolor{softred}{RGB}{214,39,40}  %
    \definecolor{softgreen}{RGB}{44,160,44}  %
    \caption{\textbf{GR prevents focus on easy questions.} Accuracy on MATH with rule-based reward, by difficulty of the question category, grouped by base-model accuracy.  Colored depending on \textcolor{softgreen}{increase} or \textcolor{softred}{decrease} after Step 250. Without regularization, after initial improvement, the policy improves on the easy questions but worsens on the hard questions.}
    \label{tab:math_ray_intereference}
    \resizebox{\columnwidth}{!}{
        \begin{tabular}{lllllll}
            \toprule
            Step                       &            & 0    & 250  & 500                         & 750                         & 1000                        \\
                                       & Init. Acc. &      &      &                             &                             &                             \\
            \midrule
            \multirow[t]{5}{*}{GR}     & 00-20\%      & 11.7 & 15.4 & \textcolor{softred}{15.2}   & \textcolor{softgreen}{16.2} & \textcolor{softgreen}{16.2} \\
                                       & 20-40\%      & 31.9 & 39.8 & \textcolor{softgreen}{40.2} & \textcolor{softgreen}{41.1} & \textcolor{softgreen}{41.8} \\
                                       & 40-60\%      & 49.1 & 57.9 & \textcolor{softgreen}{58.4} & \textcolor{softgreen}{59.4} & \textcolor{softgreen}{58.7} \\
                                       & 60-80\%      & 68.0 & 75.7 & \textcolor{softgreen}{76.4} & \textcolor{softgreen}{76.7} & \textcolor{softgreen}{76.5} \\
                                       & >80\%      & 83.0 & 87.9 & \textcolor{softgreen}{89.0} & \textcolor{softgreen}{89.6} & \textcolor{softgreen}{89.2} \\
            \midrule
            \multirow[t]{5}{*}{No Reg} & 00-20\%      & 11.7 & 15.2 & \textcolor{softred}{14.9}   & \textcolor{softred}{13.9}   & \textcolor{softred}{14.3}   \\
                                       & 20-40\%      & 31.9 & 39.0 & \textcolor{softred}{38.5}   & \textcolor{softred}{36.7}   & \textcolor{softred}{37.6}   \\
                                       & 40-60\%      & 49.1 & 57.7 & \textcolor{softgreen}{58.3} & \textcolor{softred}{55.8}   & \textcolor{softred}{56.2}   \\
                                       & 60-80\%      & 68.0 & 76.2 & \textcolor{softred}{76.0}   & \textcolor{softred}{74.3}   & \textcolor{softred}{73.9}   \\
                                       & >80\%      & 83.0 & 87.0 & \textcolor{softgreen}{88.0} & \textcolor{softgreen}{87.2} & \textcolor{softgreen}{87.9} \\
            \bottomrule
        \end{tabular}
    }
    \vspace{-0.2cm}
\end{table}

\begin{figure}
    \centering
    \input{figures/llm_judge/gradreg_llmjudge_acc_sidebyside3methods.tex}
    \vspace{-0.3cm}
    \caption{\textbf{GR prevents reward hacking with LLM-as-a-Judge} Results when training Qwen2.5-0.5B-Inst. on GSM8K with Qwen2.5 1.5B-Inst. as judge. Left: LLM-Judge and rule-based accuracy over time, showing reward hacking without regularization.}
    \vspace{-0.3cm}
    \label{fig:llmjudge}
\end{figure}
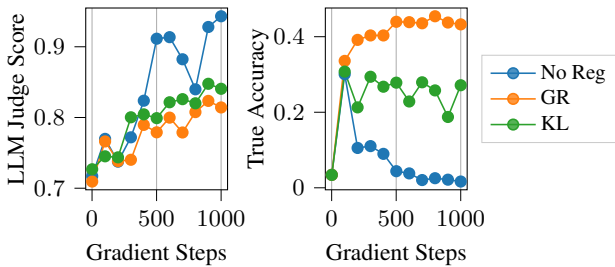

Next, we show that a form of reward hacking is also possible even within a single reward and how GR can mitigate it. We train a Qwen 2.5-1.5B-Instruct model with a rule-based reward on MATH \cite{hendrycks_measuring_2021}.
The base model achieves 46.3\% pass@1 accuracy and GRPO+GR (57.6\%) clearly outperforms standard GRPO (54.8\%).
In order to discern reward hacking, we investigate the accuracy more closely. We divide performance into quintiles of difficulty so that test-set questions are split based on the initial accuracy of their category and level-labels, as provided by the dataset.
As shown in \autoref{tab:math_ray_intereference}, 
without regularization the performance on the easiest quintile of questions continues to improve past 250 steps. But after the first 250 steps, performance on harder questions actually degrades. This demonstrates how RL can focus on learning only the easiest questions in order to hack even a rule-based reward.
In contrast, training with GR more evenly improves the accuracies across difficulties, avoiding the focus on the easier questions.

\subsection{GR Mitigates Hacking LLM-as-a-Judge in RLVR}

\begin{figure}
    \centering
    \input{figures/llm_judge/llmjudge_gradnormcomp_withrulebasedrew_sidebyside}
    \vspace{-0.3cm}
    \caption{\textbf{Gradient norm increases as reward hacking occurs.} True accuracy and gradient norm when training Qwen2.5 0.5B-Instruct on GSM8K with Qwen2.5 1.5B-Instruct LLM judge.}
    \vspace{-0.3cm}
    \label{fig:llmjudge_gradient}
\end{figure}
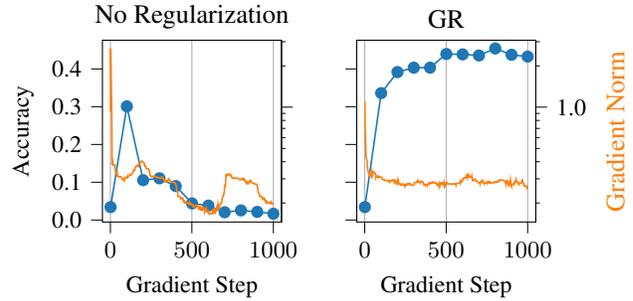

\begin{figure}
    \centering
    \input{figures/llm_judge/judge_ablation}
    \captionof{figure}{\textbf{GR allows usage of cheaper/worse judges to reach same performance.} Test accuracy when training a Qwen2.5-0.5B model on GSM8K with different judges.}
    \label{fig:llmjudge_ablation}
\end{figure}

Finally, we investigate RL with LLM-as-a-Judge.
As LLM-as-a-Judge can be susceptible to adversarial attacks \citep{zhao_one_2025}, we presume it can also be reward-hacked as a PR.
To investigate, we re-run the previous GSM8K experiments but replace the rule-based correctness reward with an LLM judge based on Qwen2.5 1.5B-Instruct \citep{qwen_qwen25_2025}.
The judge receives the problem description, true answer, model response, model reasoning, and is instructed to output a score for correctness (1 to 5) which acts as the reward.
As shown in \autoref{fig:llmjudge}, without regularization the model quickly starts to ``hack'' the judge LLM. LLM judge score goes sharply up while pass@1 accuracy on the test set peaks quite early. Empirically, we observed the model outputting excessive brackets and new HTML tags to fool the judge. In contrast, both GR and KL show much more reasonable train rewards and prevent excessive reward hacking, with GR resulting in a better final performance.
Confirming results in \autoref{sec:ref_resets}, we also see that, without regularization, an increase in gradient norm occurs when reward hacking begins, shown in \autoref{fig:llmjudge_gradient} (left). 
In \cref{fig:llmjudge_ablation} we use different LLM judges, showing that with GR we can use a weaker judge to obtain the same result as a stronger judge without regularization.
We provide additional ablations, experiments, and hyper-parameter sweeps in Appendix \ref{app:additional_experiments}.

\paragraph{Limitations}
We would like to mention two limitations of our experiments and approach.
While GR is empirically effective in the reasoning tasks, the assumption of a Lipschitz-continuous true reward highly depends on the chosen distance between actions.
While we believe it is reasonable to consider distances in a ``semantic space'' as illustrated in \cref{fig:figure2} (left), it could also be argued that the true reward in reasoning should be nonzero only for a unique true answer with specific formatting, violating continuity.

Another issue is that GR may inadvertently favor flat but incorrect maxima of the PR.
If the PR has such maxima, using GR could even cause reward hacking and monitoring the gradient norm during training could be misleading.
We did not observe this issue in our experiments and are not aware of PRs with this characteristic, but this possibility should be considered when using GR in a new setting.

\section{Conclusion}
We have investigated the problem of RL with proxy rewards (PR) and proposed a novel perspective: learning a policy in a regime where the PR is accurate.
For this purpose, we have derived a theoretical connection between the flatness of an optimum and the Bradley-Terry loss of the PR at this optimum.
By regularizing the gradient norm during training, we can bias the RL updates towards such a flat optimum.
We first validated our theoretical analysis by using implicit gradient regularization (GR) via Reference Resets, showing they improve upon a KL penalty.
We then proposed to use explicit GR based on an efficient implementation of a finite-difference estimate.
Explicit GR allows us to mitigate reward model hacking of RMs in RLHF, reduce focus on easy rule-based rewards in RLVR, and alleviate format hacking with LLM-as-a-judge.
We believe that GR is a promising candidate to completely replace KL penalties and improve training runs that currently eschew regularization.

\section*{Acknowledgements}
We would like to thank Soichiro Nishimori and Thanawat Lodkaew for helpful discussions.
MS was supported by JST ASPIRE Grant Number JPMJAP25B1.
TI was supported by KAKENHI Grant Number 22K17946.

\section*{Impact Statement}
We propose a method to prevent reward hacking in RL post-training of LLMs.
We believe that preventing reward hacking is likely to have a beneficial impact in general.
However, our theory only considers certain specific kinds of reward hacking, thus there is a risk that by overly relying on our method, users may miss other kinds of reward hacking, which should be monitored independently.

\bibliography{mybib}
\bibliographystyle{icml2026}

\newpage
\appendix
\renewcommand{\sectionautorefname}{Appendix}
\renewcommand{\subsectionautorefname}{Appendix}
\onecolumn

\section{Extended related work}
\label{app:extended_related_work}
\paragraph{Reference Resets}
We here provide an overview of alternative explanations for reference-resets and related methods.
\citet{noukhovitch_language_2023} proposed Elastic Resets, a method which maintains an exponential moving average of the weights $\bar{\phi}$ during the updates. After each $R$ steps, the current policy parameters are replaced with the EMA weights $\bar{\phi}$ and the EMA weights are reinitialized to the initial policy $\phi^1$.
Next, \citet{gorbatovski_learn_2025} proposed Trust-Region Direct Preference Optimization (TR-DPO), which proposes a method very similar to reference resets, applied to DPO \citep{rafailov_direct_2023}.
Similar to Reference Resets, the reference policy is set to the current policy after each $R$ updates.
While DPO is meant to increase the probability of preferred responses $\pi(a_\mathrm{w}|s)$ while decreasing the probability of un-preferred responses $\pi(a_\mathrm{l}|s)$, in practice it often decreases both $\pi(a_\mathrm{w}|s)$ and $\pi(a_\mathrm{l}|s)$, degrading performance.
This issue is known as likelihood displacement \citep{razin_unintentional_2025}.
\citet{gorbatovski_learn_2025} motivate the mechanism of TR-DPO by preventing likelihood displacement, however likelihood displacement does not occur in on-policy alignment methods such as GRPO, which we investigate in this work.
This mechanism can thus not explain why Reference Resets improve over standard GRPO.

\citet{liu_prorl_2025} motivate the need for Reference Resets in reasoning training due to the magnitude of the KL penalty overwhelming the magnitude of the reward after a certain number of training steps.
However, if this was the reason, we could simply use a smaller KL penalty $\beta$ from the beginning, which we show in \autoref{fig:app:refresetklstrength} does not work.
Another reasonable approach following the reasoning of \citep{liu_prorl_2025} would be to simply decrease $\beta$ when the reward stagnates, but we also show in the experiments (\autoref{fig:app:reset_ablation_resetrate}) that this does not work as well as Reference Resets.
Instead, we provide a theoretical analysis explaining its success and provide experiments validating it specifically in the RLHF setting.
\citet{ackermann_off-policy_2025} contemporaneously with \citet{liu_prorl_2025} proposed reference resets as an ablation called \say{PPO + New KL}.

\paragraph{Replacing the KL penalty in Post-training}
The most common way to prevent reward hacking with proxy rewards is a KL penalty, introduced in RLHF by \citet{stiennon_learning_2020} and RLVR by \citet{havrilla_teaching_2024}.
Some recent reports about large model training, for example GLM4.5 \citep{team_glm-45_2025} and Olmo3 \citep{olmoteam_olmo_2025} remove the KL penalty, while others such as Kimi K2 \citep{kimi_team_kimi_2025} still use it.
Our experiments show that both in RLHF and math experiments using gradient regularization performs better than either the KL penalty or removing the KL penalty.
Direct alignment methods such as Direct Preference Optimization (DPO) \citep{rafailov_direct_2023}, Kahneman-Tversky Optimization (KTO) \citep{ethayarajh_kto_2024}, do not use an explicit KL penalty during training, however they indirectly optimize the KL-constrained RL objective.
Instead of just a KL penalty \citep{ouyang_training_2022} uses \textit{PPO-ptx} which adds an additional behavior cloning term to the loss, on data sampled from a pretraining dataset, to prevent regressions on standard NLP benchmarks during RLHF training.

\paragraph{Gradient Regularization}
While \citep{karakida_understanding_2023} performs experiments with convolutional neural networks, \cite{zhao_penalizing_2022,zhao_when_2024} further apply GR to vision transformer models.
Sharpness-aware minimization (SAM) \citep{foret_sharpness-aware_2021} has been shown to correspond to gradient regularization with a specific choice of hyper-parameters \citep{karakida_understanding_2023}.
\citep{bahri_sharpness-aware_2022,zhang_ga-sam_2022} apply SAM to transformer pre-training in order to improve generalization.
SAM has also been used in RL, particularly by \citet{lee_plastic_2023} to improve sample efficiency when training a policy for Atari games, and by \citep{lee_flat_2025} to obtain robust policies in continuous control tasks.
Our theoretical analysis in part uses the argument provided \citep{lee_flat_2025}, relating parameter flatness to action robustness.
Our work is, to our knowledge, the first one to investigate the relation of GR to the accuracy of proxy rewards, as well as the first work to use GR in RLHF/RLVR post-training of LMs.

\section{Proofs}
\label{app:proofs}
Our argument can be summarized as follows:
\citet{lee_flat_2025} showed that flatness in parameter space is related to flatness in action space, we slightly extend the argument to a maximum error of $\widehat{L}$ in Proposition \ref{prop:flat_return_implies_robust_policy}.
In our case this controls the expected proxy reward inside, when assuming a Gaussian policy with fixed covariance and full row-rank Jacobian (Assumption \ref{ass:gaussian_policy_fixed_cov_full_rank}).
Under the assumption of $\beta$ smoothness of the expected proxy reward (Assumption \ref{ass:action_smooth_proxy_reward}), this provides a bound on the gradient norm of the expected proxy reward (Lemma \ref{lem:flat_smooth_imply_boundgrad}).
We then relate the gradient norm of the expected proxy reward to the point-wise norm of the action-gradient in Lemma \ref{lem:pointwise_vs_smoothed_grad}.
Then, as in an area with a bounded gradient norm Lipschitz continuity is guaranteed, in Lemma \ref{lem:local_smooth_means_pairwise_robust} we show that overly sharp action pairs can only occur if at least one of the actions is outside of a certain radius from the maximum.
Putting it all together yields Proposition \ref{prop:flat_return_implies_deltarobust}.
Finally, we show that, under a Lipschitz continuous true reward (Assumption \ref{ass:lipschitz_true_reward}), sharp minima incur an excess BT error.

\subsection{Flatness in Parameter Space implies $(\delta,K,\rho)$-pairwise robustness}
We will first show that a $\mathcal{E}-\widehat{L}$ flat reward implies a $D-\widehat{L}$ robust policy, closely based on the proof proposed by \citet{lee_flat_2025}.
We also need to define a $D-\widehat{L}$ action robust policy:
\begin{definition}[$D-\widehat{L}$ action robust policy, slightly modified from \citet{lee_flat_2025}]
    \label{def:action-robust}
    For a reward function $R(s,a)$ and policy $\pi_\phi(a|s)$, parameterized by $\phi$, a maximum $\phi^*$ is $D$-$\widehat{L}$ action robust if the following holds:
    \begin{equation}
        \begin{aligned}
             \text{For all }\delta\in\mathbb{R}^{|A|}\text{ s.t. }\|\delta\|\le D: \quad \quad \quad
             \E_{a\sim\pi_{\phi^*}}\left[\prew(s,a)\right] - \E_{a\sim\pi_{\phi^*}}\left[\prew(s,a + \delta) \right] \leq \widehat{L}
        \end{aligned}
    \end{equation}
    \label{def:d-l-robustpol}
\end{definition}
We will also need the following assumptions during this section:
\begin{assumption}[Gaussian policy with fixed covariance and full row-rank Jacobian]
    \label{ass:gaussian_policy_fixed_cov_full_rank}
    For each $s$, $\pi_{\phi^*}(\cdot\mid s)=\mathcal{N}(\mu_{\phi^*}(s),\Sigma)$ with a fixed positive definite covariance $\Sigma$. The Jacobian matrix $\mathfrak{J}(\phi^*)$ of the mean action $\mu_\phi(s)$ w.r.t. $\phi$, evaluated at $\phi^*$, has full row rank. %
\end{assumption}

\begin{assumption}[Action-smooth proxy reward]
    \label{ass:action_smooth_proxy_reward}
    For all $s\in\mathcal{S}$, the PR $\prew(s,a)$ is $\beta$-smooth in $a$, i.e.,
    $\|\nabla_a \prew(s,a_1)-\nabla_a \prew(s,a_2)\|\le \beta\|a_1-a_2\|$ for all $a_1,a_2\in\mathcal{A}$.
\end{assumption}

\begin{assumption}[Lipschitz-continuous true reward]
    \label{ass:lipschitz_true_reward}
    For all $s\in\mathcal{S}$, the true reward $R^*(s,a)$ is Lipschitz in $a$, i.e., $|R^{*}(s,a_1)-R^{*}(s,a_2)| \leq L \|a_1-a_2\|$ for all $a_1,a_2$.
\end{assumption}

We now can show that an $\mathcal{E}-\widehat{L}$ flat return implies $D-\widehat{L}$ robust policy:

\begin{proposition}[ $\mathcal{E}-\widehat{L}$ flat return implies $D-\widehat{L}$ robust policy, slightly modified from \citet{lee_flat_2025}]
    \label{prop:flat_return_implies_robust_policy}
    If $\phi^*$ is an $\mathcal{E}-\widehat{L}$ flat return maximum of a policy under assumption \ref{ass:gaussian_policy_fixed_cov_full_rank}, then the policy  $\phi^*$ is $D^*-\widehat{L}$ robust, where:
    \begin{equation}
        D^* \le \bigl\lVert \mathfrak{J}(\phi^*)\bigr\rVert \,\mathcal{E} +\mathcal{O}(\mathcal{E}^2),
    \end{equation}
    and
    \[
        \mathfrak{J}(\phi^*) := \nabla_{\phi}\,\mu_{\phi}(s)\bigm|_{\phi=\phi^*}
    \]
    is the Jacobian matrix of the mean action $\mu_{\phi}(s)$ with respect to~$\phi$, evaluated at~$\phi^*$.
\end{proposition}
\begin{proof}
    Assume a Gaussian policy with fixed covariance $\Sigma$, such that $\pi_\phi(a|s) = \mathcal{N}(a;\mu_\phi(s),\Sigma)$, where $\mu_\phi(s)$ is the mean of the Gaussian we are training.
    A Taylor expansion yields
    $\mu_{\phi+\epsilon}(s)=\mu_\phi(s)+J(\phi)\epsilon+O(||\epsilon||^2)$.

    Define the change in policy for a given state $s$ due to the perturbation $\epsilon$ as
    \[\delta(s) = \mu_{\phi+\epsilon}(s) - \mu_{\phi}(s) = J(\phi)\epsilon+O(||\epsilon||^2)\,.
    \]
    Then by the triangle inequality, the Cauchy-Schwarz inequality and definition $\|\epsilon\|<\mathcal{E}$, we have
    \[
        ||\delta(s)|| \leq ||J(\phi)|| ||\epsilon|| + O(||\epsilon||^2) \leq ||J(\phi)||\mathcal{E} + O(\mathcal{E}^2)\,.
    \]

    The sub-optimality of this action perturbation $\delta$ is
    \begin{align}
         & \E_{s \sim P(s), a \sim \pi_{\phi^*}(s)}\left[\prew(s,a) \right]
        -
        \E_{s \sim P(s), a \sim \pi_{\phi^*}(s)}\left[\prew(s,a + \delta(s)) \right]
        \\
        \stackrel{\text{Def. $\delta$}}{=}
         & \E_{s \sim P(s), a \sim \pi_{\phi^*}(s)}\left[\prew(s,a) \right]
        -
        \E_{s \sim P(s), a \sim \pi_{\phi^* + \epsilon}(s)}\left[\prew(s,a)\right]
        \stackrel{\text{Def \ref{def:e-l-flatrew}}}{\leq}
        \widehat{L}
    \end{align}

    Thus a $\mathcal{E}-\widehat{L}$ flat reward maximum implies a $D^*-\widehat{L}$ robust policy with $D^* \leq ||J(\phi^*)||\mathcal{E} + \mathcal{O}(\mathcal{E}^2)$
\end{proof}

As we want to show an excess BT loss based on a Lipschitz-continuity assumption, which is defined over action pairs, we need to relate $D-\widehat{L}$ robustness to $P(S_{K,\delta}|\pi_\phi)$.
For this purpose we need \ref{ass:action_smooth_proxy_reward} following two simple lemmas:

\begin{lemma}[Flatness and $\beta$-smoothness imply bounded gradient]
    \label{lem:flat_smooth_imply_boundgrad}
    Let $f: \mathbb{R}^d\to\mathbb{R}$ be differentiable and $\beta$-smooth on a ball $B(c,r)$, i.e.,
    $\|\nabla f(x) - \nabla f(y)\| \leq \beta \| x -y \| \quad\forall x,y\in B(c,r)$\,.
    Assume a "flat maximum" in $c$, such that $f(c)-f(c+u)\leq \widehat{L} \quad \forall u:\|u\|\leq r $.
    Then the gradient norm at c is bounded as
    \[
        \|\nabla f(c) \| \leq \frac{\widehat{L}}{r} + \frac{\beta}{2}r \,.
    \]
\end{lemma}
\begin{proof}
    Based on the standard quadratic upper bound based on $\beta$-smoothness, we know
    \[
        f(c+u) \leq f(c) + \nabla f(c)^Tu + \frac{\beta}{2}\|u\|^2 \quad \forall c+u\in B(c,r)\,.
    \]
    We can rearrange this to
    \[
        f(c)-f(c+u) \geq - \nabla f(c)^Tu - \frac{\beta}{2}\|u\|^2\,.
    \]
    By choosing the worst case $u=-r\frac{\nabla f(c)}{\|\nabla f(c)\|}$ (if $\|\nabla f(c)\|=0$ the bound is trivially true), we have $\|u\|=r$ and $-\nabla f(c)^T u= r \|\nabla f(c)\|$. Plugging this into the inequality and using $f(c)-f(c+u)\leq \widehat{L}$ yields
    \[
        \|\nabla f(c) \| \leq \frac{\widehat{L}}{r} + \frac{\beta}{2}r \,.
    \]
\end{proof}
Note that robustness with radius $r$ implies robustness with any $r'<r$ and we could further improve this bound by picking $r'=\min(\sqrt{2\beta \hat{L}})\coloneq r^*$, if $r^*\leq r$.
We do not do this here for ease of exposition.
We next need to connect the gradient bound of the expected $\E[\prew(s,a+Z)]$ to a bound on the gradient of the pointwise $\prew(s,a)$:
\begin{lemma}[Pointwise action-gradient is controlled by Gaussian-smoothing]
    \label{lem:pointwise_vs_smoothed_grad}
    Fix $s$ and let $f(a)=\prew(s,a)$. Under Assumption \ref{ass:action_smooth_proxy_reward} and for $Z\sim\mathcal{N}(0,\Sigma)$, define the Gaussian-smoothed reward
    \[
        \bar{f}(c) := \E[f(c+Z)]\,.
    \]
    Then $\nabla \bar f(c)=\E[\nabla f(c+Z)]$ and for all $c$,
    \[
        \|\nabla f(c)\| \le \|\nabla \bar{f}(c)\| + \beta\,\E\left[\|Z\|\right]\,.
    \]
\end{lemma}
\begin{proof}
    Since $f$ is $\beta$-smooth, it is continuously differentiable with Lipschitz gradient, and we can interchange gradient and expectation to obtain $\nabla \bar f(c)=\E[\nabla f(c+Z)]$.
    Then, by Jensen's inequality and $\beta$-smoothness,
    \[
        \|\nabla f(c)-\nabla \bar f(c)\|
        =
        \Bigl\|\E[\nabla f(c)-\nabla f(c+Z)]\Bigr\|
        \leq
        \E\bigl[\|\nabla f(c)-\nabla f(c+Z)\|\bigr]
        \leq
        \beta\,\E\|Z\|\,,
    \]
    then, by the triangle inequality and the previous line,
    \begin{align}
        \|\nabla f(c)\|
         & =
        \|\nabla f(c) - \nabla f(c+Z) + \nabla f(c+Z)\|
        \leq \|\nabla f(c+Z)\| + \|\nabla f(c) - \nabla f(c+Z) \| \\
         & \leq
        \|\nabla f(c+Z) \| + \beta \E \left[ \| Z \| \right]
        =
        \|\nabla\bar{f}(c) \| + \beta \E \left[ \| Z \| \right]\,.
    \end{align}

\end{proof}
Then, we need to connect local smoothness and our knowledge of the gradient norm $\|\nabla f(c)\|$ at the center to violations of Lipschitz-continuity:
\begin{lemma}[Pairwise robustness is bounded by local smoothness]
    \label{lem:local_smooth_means_pairwise_robust}
    Fix a state $s$. Denote $f(a)=\prew(s,a)$. Assume that there exists a center $c\in A$ and radius $r>0$ such that $f$ is $\beta$-smooth on the ball $B(c,r):=\{a:\|a-c\|\leq r\}$.
    Let $a,a_1,a_2\stackrel{\mathrm{i.i.d.}}{\sim}\pi(a|s)$.

    Then for any $\delta>0$ and $K>0$, we have
    \[
        P(\|a_1-a_2\|\leq\delta,|f(a_1)-f(a_2)| > K) \leq 2P(\|a-c\|>r) \,,
    \]
    for $r= \frac{1}{\beta}\left( \frac{K}{\delta} - \|\nabla f(x)\| \right)$
\end{lemma}
\begin{proof}
    For any $x\in B(c,r)$, by $\beta$-smoothness,
    \[
        \|\nabla f(x)\|
        \leq
        \| \nabla f(c)\| + \| \nabla f(x) - \nabla f(c)\|
        \leq
        \|\nabla f(c)\| + \beta \|x-c \|
        \leq
        \|\nabla f(c)\| + \beta r \,.
    \]
    For any two points $x,y\in B(c,r)$,
    \[
        f(x)-f(y) = \int_0^1\nabla f(y+t(x-y))^T(x-y)dt\,.
    \]
    taking the norm, using the Cauchy-Schwarz inequality, and the bounded gradient, we get
    \[
        |f(x)-f(y)| \leq \int_0^1\|\nabla f(y+t(x-y))\|\|x-y\|dt\, \leq (\|\nabla f(c)\| + \beta r) \|x - y\|\,.
    \]
    Therefore, if $x,y\in B(c,r)$ and $\|x-y\|\leq \delta$, then
    \[
        |f(x)-f(y)| \leq (\|\nabla f(c)\| + \beta r)\delta \,.
    \]
    Thus, the event $\{\|a_1-a_2\|\leq \delta, |f(a_1)-f(a_2)|>K\}$ cannot occur when $a_1,a_2\in B(c,r)$ and $(\|\nabla f(c)\| + \beta r)\delta<K$.
    It thus can only occur if at least one of $a_1$ or $a_2$ is not in $B(c,r)$, which occurs with probability $P(\|a-c\|>r)$. By a union bound of the two events we get the result.
\end{proof}

Finally, we can put it all together:

\begin{proposition}[From $D$-$\widehat{L}$ robustness to pairwise sharpness control]
    \label{app:flat_to_pairwise}
    Assume a Gaussian policy with full row-rank Jacobian (Assumption \ref{ass:gaussian_policy_fixed_cov_full_rank}), a $\beta$-smooth proxy reward $\prew$ (\ref{ass:action_smooth_proxy_reward}), and fix a state $s$.
    If $\phi^*$ is $D$-$\widehat{L}$ action robust, then for any $\delta>0$ and $K>0$ such that $K/\delta>G$,
    we have gradient magnitude bound $G$ and non-violating radius $r$,
    \[
        G := \frac{\widehat{L}}{D}+\frac{\beta}{2}D+\beta\,\E\|Z\|\,,
        \quad
        r:=\frac{1}{\beta}\bigl(\frac{K}{\delta}-G\bigr)\,,
    \]

    such that
    \[
        P(S_{K,\delta}(\prew)\mid \pi_{\phi^*}) \le 2\, P(\|Z\|>r)\,.
    \]
\end{proposition}
\begin{proof}
    Fix $s$ and denote $c=\mu_{\phi^*}(s)$.
    Denote the PR as $f(a)=\prew(s,a)$, mean action as $c=\mu_{\phi^*}(s)$, and the smoothed proxy reward $\bar{f}(c)=\E_{Z\sim\mathcal{N}(0,\Sigma)}[f(c+Z)]$.
    By Assumption \ref{ass:gaussian_policy_fixed_cov_full_rank}, $a\sim\pi_{\phi^*}(a|s)$ can be written as $a=c+Z$ with $Z\sim\mathcal{N}(0,\Sigma)$.
    By Lemma \ref{lem:flat_smooth_imply_boundgrad} and, applied to $\bar{f}$, which we know fulfills Assumption \ref{ass:action_smooth_proxy_reward}, we obtain on $B(c,D)$ the gradient bound
    \[
        \|\nabla \bar f(c)\| \le \frac{\widehat{L}}{D} + \frac{\beta}{2}D\,.
    \]
    From Lemma \ref{lem:pointwise_vs_smoothed_grad}, we know that the gradient norm of $f$ can then be bounded as
    \[
        \|\nabla f(c)\|
        \leq
        \underbrace{\frac{\widehat{L}}{D} + \frac{\beta}{2}D+\beta\,\E[\|Z\|]}_{\coloneq G} \,.
    \]
    We then know from Lemma \ref{lem:local_smooth_means_pairwise_robust}, that in $B(c,r)$ with non-violating radius $r=\frac{1}{\beta}(\frac{K}{\delta}-G)$ there can be no action pairs $a_1,a_2$ with $\|a_1-a_2\|\leq\delta$ and $|\prew(s,a_1)-\prew(s,a_2)|>K$, thus such violations can only occur if at least one action is outside $B(c,r)$:
    \[
        P(\|a_1-a_2\|\le\delta,\ |f(a_1)-f(a_2)|>K) \le 2\,P(\|a_1-c\|>r)\,.
    \]
    Finally, $a_1-c\sim\mathcal{N}(0,\Sigma)$, so $P(\|a_1-c\|>r)=P(\|Z\|>r)$.
\end{proof}

\subsection{BT-Loss Lower Bound Based on Sharpness}
\label{sec:app:btsharpnessbound}
Now that we have connected $D-\widehat{L}$ robustness to the probability of actions pairs violating a flatness assumption $P(
    \sharpset|\pi_\phi)$, we now analyze the incurred excess BT loss $\mathcal{L}_\mathrm{BT}$ due to these violations.
For this purpose, we make the assumption of a Lipschitz continuous true reward (Assumption \ref{ass:lipschitz_true_reward})
\begin{proposition}
    For a prompt $s\sim P(s)$, a pair of actions $(a_1,a_2)$, $L$-Lipschitz true reward function $R^{*}$, reward model $\prew$, the excess-risk can be lower bounded as
    \begin{equation}
        \mathcal{L}_\mathrm{BT}(\prew) - \mathcal{L}_\mathrm{BT}(R^{*}) \geq 2 (\sigma(K) - \sigma(L\delta))^2 P(\sharpset) \,,
    \end{equation}
    where $S_{K,\delta}:=\left\{(s,a_1,a_2): \|a_1-a_2\|\leq \delta, |\prew(s,a_1) - \prew(s,a_2)| > K \right\}$ is the set of action pairs for which the RM $\prew$ is not $K$-Lipschitz continuous and $\Pr(S_{K,\delta})$ is the probability of an action pair sampled $(a_1,a_2)\sim \pi_\phi(\cdot|s)$ being in this set.
\end{proposition}

Let $s \sim P(s)$ denote prompts, $a\in \mathcal{A}$ denote actions.
Let $R^{*}(s,a)$ be the true reward.
Under the BT model, the preference probability is $p:=\Pr(Y=1\mid s,a_1,a_2)=\sigma(\Delta_*)$, $\Delta_*:=R^{*}(s,a_1)-R^{*}(s,a_2)$, with the logistic function $\sigma(x)=\frac{1}{1 + e^{-x}}$, where $Y=1$ means that $a_1$ is preferred over $a_2$.
A parametric reward model $\prew(s,a)$ induces the predicted pairwise probability $q:=P_{\prew}(Y=1|s,a_1,a_2)=\sigma(\Delta_\theta)$,$\Delta_\theta:=\prew(s,a_1)-\prew(s,a_2)$.
The BT loss in \eqref{eq:bradleyterryCrossentropyLoss} can be rewritten as
\begin{equation}
    \mathcal{L}_\mathrm{BT}^{\pi^1} (\theta) = \E_{(s,a_w,a_l)}\left[-\log\sigma\left(\prew(s,a_\mathrm{w})-\prew(s,a_\mathrm{l}\right)) \right] \,,
\end{equation}
but it can also be rewritten to explicitly consider a preference label $Y$, where we have actions $a_1,a_2$; then draw $Y\sim\mathrm{Bernoulli}(p)$ with $p=\Pr(Y=1|s,a_1,a_2)$.
Taking expectations first over $Y|s,a_1,a_2$ and then over $(s,a_1,a_2)$, we get,
\begin{equation}
    \mathcal{L}_\mathrm{BT}^{\pi^1} (\theta) = \E_{(s,a_1,a_2)}\E_{Y\sim\mathrm{Bernoulli}(p)}\left[\ell(q;Y)\right] \,,
\end{equation}
for $Y\in\{0,1\}$ and $\ell(q;Y)=-\left[Y\log q+(1-Y)\log(1-q)\right]$. We proceed with this version.

\paragraph{Lower Bound} Condition on $(s,a_1,a_2)$ so that $p$ is fixed. Then
\begin{equation}
    \E_{Y\sim\mathrm{Bernoulli}(p)}\left[\ell(q;Y)\right]=H(p,q), \quad \E_{Y\sim\mathrm{Bernoulli}(p)}\left[\ell(p;Y)\right]=H(p),
\end{equation}
where $H(p,q)$ is the cross-entropy and $H(p)$ is the entropy.
Averaging over $(s,a_1,a_2)$ and subtracting yields
\begin{equation}
    \mathcal{L}(\prew) - \mathcal{L}(R^{*})=\E_{(s,a_1,a_2)}\left[D_\mathrm{KL}(\mathrm{Bernoulli}(p);\mathrm{Bernoulli}(q))\right]\,.
    \label{eq:excess_loss_kl}
\end{equation}
For a fixed locality parameter $\delta>0$ and margin threshold $K>0$, we define the sharpness set
\begin{equation}
    \sharpset:=\left\{ (s,a_1,a_2): \|a_1-a_2\|\leq\delta, |\Delta_\theta|>K   \right\}\,.
\end{equation}
Assume the true reward is $L$-Lipschitz:
\begin{equation}
    |R^{*}(s,a_1)-R^{*}(s,a_2)| \leq L \|a_1-a_2\| \Rightarrow |\Delta_*|\leq L\delta \quad \text{whenever}\quad \|a_1-a_2\|\leq\delta\,.
\end{equation}

Furthermore, we study the interval separation on $\sharpset$: On $\sharpset$ we have $\Delta_*\leq L\delta$ and $\Delta_\theta\geq K$.
By monotonicity and symmetry of $\sigma$,
\begin{equation}
    \begin{aligned}
         & p=\sigma(\Delta_*) \in I_*:=[\sigma(-L\delta),\sigma(L\delta)] = [1-\sigma(L\delta),\sigma(L\delta)], \\
         & q=\sigma(\Delta_\theta) \in I_\theta:=[0,1-\sigma(K)] \cup [\sigma(K),1], \,.
    \end{aligned}
\end{equation}

If we assume $K>L\delta$, the minimum distance between these sets becomes:
\begin{equation}
    \inf_{p\in I_*,q\in I_\theta}|p-q|=
    \min\left\{
    \sigma(K)-\sigma(L\delta), (1-\sigma(L\delta))-(1-\sigma(K))
    \right\}
    =\sigma(K)-\sigma(L\delta)\,.
\end{equation}

Consequently, for every $(s,a_1,a_2)\in\sharpset$,
\begin{equation}
    |p-q| \geq \sigma(K) - \sigma(L\delta)\,.
    \label{eq:prob_diff}
\end{equation}
Since we can use the inequality  $D_\mathrm{KL}(\mathrm{Bernoulli}(p);\mathrm{Bernoulli}(q))\geq2|p-q|^2$ for $p,q\in(0,1)$, combining with \eqref{eq:prob_diff} yields:
\begin{equation}
    D_\mathrm{KL}(\mathrm{Bernoulli}(p);\mathrm{Bernoulli}(q))\geq2(\sigma(K) - \sigma(L\delta))^2 \quad\text{for every} \quad (s,a_1,a_2)\in\sharpset \,.
    \label{eq:bound_member}
\end{equation}

Let $\mathbf{1}_S:=\mathbf{1}\{Z\in \sharpset\}$ indicate membership of $\sharpset$.
We can make \eqref{eq:bound_member} valid for all $(s,a_1,a_2)$ by multiplying the RHS with the indicator:

\begin{equation}
    D_\mathrm{KL}(\mathrm{Bernoulli}(p);\mathrm{Bernoulli}(q))\geq2(\sigma(K) - \sigma(L\delta))^2 \mathbf{1}_S \quad\text{for all} \quad (s,a_1,a_2) \,.
    \label{eq:bound_all}
\end{equation}
By taking the expectation of \eqref{eq:bound_all} and using \eqref{eq:excess_loss_kl}, we obtain the desired lower bound
\begin{equation}
    \mathcal{L}_\mathrm{BT}(\prew) - \mathcal{L}_\mathrm{BT}(R^{*}) \geq 2 (\sigma(K) - \sigma(L\delta))^2 P(\sharpset) \,.
\end{equation}

\subsection{Connection of Gradient Regularization and Lipschitz Continuity}
\label{app:gradreg_lipschitzness}
For completeness, we reproduce the argument of \citet{zhao_penalizing_2022}, which explicitly connects gradient regularization to Lipschitzness in parameters $\theta$.

By the mean value theorem for differentiable $L$, we have
$L(\theta_1)- L(\theta_2) = \nabla L(\zeta)^T(\theta_1-\theta_2)\,,$
with $\zeta=c\theta_1 + (1-c)\theta_2$, with some $c\in[0,1]$ and the Cauchy-Schwarz inequality then yields
$\|L(\theta_1)- L(\theta_2)\| \leq \|\nabla L(\zeta)\|\|(\theta_1-\theta_2)\|\,.$
Here, $||\nabla L(\zeta)||$ takes the role of the Lipschitz constant and as $\theta_2\to \theta_1$, $||\nabla L(\zeta)||$ becomes $||\nabla L(\theta)||$.
Thus we can see that gradient regularization leads to local Lipschitzness in parameter space, i.e., a flat local minimum.

\section{Experiment Details}
\label{app:expdetails}
In this section we provide additional experimental details.

\subsection{Gradient regularization implementation}
\begin{figure}
    \centering
    \begin{minipage}{0.8\columnwidth}
        \begin{minted}{python}
phi = model.state_dict()
actions = model.generate(states)
rewards = reward_fn(states, actions)

grad1 = torch.zeros_like(phi)
for idx in range(batch_size / accumulation_steps):
    # mb = microbatch
    loss = grpo_loss(states_mb, actions_mb, rewards_mb)
    loss.backward()
    grad1 += model.grad
grad1 = norm_clip(grad1)

phi_2 = phi + varepsilon * grad1
grad2 = torch.zeros_like(phi)
model.set_state_dict(phi_2)
for idx in range(batch_size / accumulation_steps):
    loss = grpo_loss(states_mb, actions_mb, rewards_mb, phi_2)
    loss.backwards()
    grad2 += model.grad
grad2 = norm_clip(grad2)

comb_grad = grad1 + gamma * (grad2 - grad1) / varepsilon
model.set_state_dict(phi)
model.grad = comb_grad
optimizer.step()
\end{minted}
    \end{minipage}
    \caption{Implementation of finite-difference gradient regularization with GRPO in PyTorch}
\end{figure}

In our experiments, we use the GR method \citep{karakida_understanding_2023} based on the finite difference estimate
$\Delta_\phi\|\nabla_\phi \mathcal{L}(\phi)\|^2 =
    \frac
    {\nabla_\phi \mathcal{L}(\phi+\varepsilon\nabla_\phi \mathcal{L}(\phi)) - \nabla_\phi \mathcal{L}(\phi)}
    {\varepsilon}\,.$
We, thus, need to perturb the parameters $\phi$.
Empirically, we found it to be beneficial to perturb only the parameters of the transformer blocks, including attention matrices, MLP weights and layer norm parameters, but not perturb the embedding layer or final output layer.
In principle, to calculate $\|\nabla_\phi J(\phi + \nabla_\phi J(\phi))\|$ we would also need to sample new actions $a_i\sim \pi_{\phi + \nabla_\phi J(\phi)}(a|s)$ and estimate the gradient using these.
In practice, the computational overhead of this would be large, we reuse the same actions.
This introduces some bias which could be corrected using importance sampling, however, empirically we found this to be unnecessary.
We also need to choose a perturbation strength $\varepsilon$.
While \citet{karakida_understanding_2023} found a relatively large $\varepsilon\approx0.05$ to perform best, initial experiments showed $\varepsilon=10^{-3}$ to perform well in our setting. We thus used it through-out our experiments.
We further use gradient clipping, both for the disturbance $\nabla_\phi J(\phi)$ and the gradient $\|\nabla_\phi J(\phi + \nabla_\phi J(\phi))\|$, each to 10.
We found this to prevent gradient spikes from destabilizing training.
The final combined gradient is then again clipped to 1.0 within DeepSpeed, as is done for the non GR methods as well.
We train our models using DeepSpeed ZeRO 2 \cite{rajbhandari_zero_2020} and use gradient accumulation.

\subsection{RLHF Details}
\label{app:rlhf-details}
\begin{table}[]
    \caption{GRPO hyper-parameters in RLHF experiments. We tuned the learning rate for each method on Qwen 2.5-1.5B experiments from $1\times 10^{-6}, 3\times 10^{-6}, 5 \times 10^{-6}$.}
    \centering
    \begin{tabular}{l c c}
        \toprule
        Method              & GR                                          & KL, Resets, No Reg \\
        \midrule
        Optimizer           & \multicolumn{2}{c}{Adam \citep{Kingma2014}}                      \\
        LR                  & $5\times10^{-6}$                            & $3\times10^{-6}$   \\
        Adam $\beta_1$      & \multicolumn{2}{c}{0.9}                                          \\
        Adam $\beta_2$      & \multicolumn{2}{c}{0.999}                                        \\
        Batchsize           & \multicolumn{2}{c}{256}                                          \\
        Rollouts per Prompt & \multicolumn{2}{c}{8}                                            \\
        Temperature         & \multicolumn{2}{c}{0.7}                                          \\
        GR $\varepsilon$    & $1\times10^{-3}$                            & -                  \\
        Gradient Clipping   & \multicolumn{2}{c}{1.0}                                          \\
        Output Length       & \multicolumn{2}{c}{106}                                          \\
        \bottomrule
    \end{tabular}
    \label{tab:hyperparam_rlhf}
\end{table}

For our RLHF experiments, we first need to train an SFT model, sample a training dataset $D_\mathrm{RM}$, and train an RM.
For both, we use the code and hyper-parameters provided by \citet{huang_n_2024}, which use the AdamW optimizer \citep{loshchilov_decoupled_2019} with weight decay.
The SFT models are trained on the SFT dataset for one epoch.
For the summarization experiments, we use the SFT dataset provided by \citep{huang_n_2024}, for Alpaca experiments we use the Alpaca-Instructions dataset \citep{dubois_alpacafarm_2023}, but filter it by length following \citep{ackermann_off-policy_2025}.
This length filtering to a maximum length of 512 tokens significantly decreases the computational cost.
Further, while Alpaca-Instructions contains separate splits for SFT, RM and RL training, we combine them to a single dataset as used in the summarization setting.
To obtain the RM training dataset, we sample pairs of responses from the SFT model and label them with either the Gold RM \say{Skywork-RewardLlama-3.1-8B-v0.2} \citep{liu_skywork-reward_2024} for summarization or with GPT4.1-Nano for Alpaca experiments.
For the summarization task we use 278,496 preference pairs, for Alpaca Experiments we use 43,008 pairs for the 3B model and 86,016 pairs for the 0.5B and 1.5B models.
With the smaller data amount the 0.5B and 1.5B models did not produce sufficiently accurate RMs to use for subsequent GRPO updates.
We train the RMs, initialized from the SFT model, for one epoch on the dataset with the hyper-parameters as recommended by \citep{huang_n_2024}.
We then use the Dr.GRPO implementation provided by TRL \citep{vonwerra2022trl} with the hyper-parameters as listed in Table \ref{tab:hyperparam_rlhf}.
Additionally, we tuned the KL penalty strength $\beta$ and GR strength $\gamma$ as listed Tables \ref{tab:kl_optimization_tldr} and \ref{tab:kl_optimization_alpaca}.
We use the standard KL penalty estimate provided by TRL, including the KL penalty in the loss rather than in the reward.

\begin{table}[h]
    \centering
    \caption{KL penalty values considered in hyperparameter optimization for Reference Resets Tl;DR.}
    \begin{tabular}{l c}
        \toprule
        Experiment                       & KL values                                \\
        \midrule
        GRPO Pythia 1B TL;DR             & $\{0.03, 0.04, 0.05, 0.06, 0.07, 0.08\}$ \\
        GRPO Pythia 2.8B TL;DR           & $\{0.04, 0.06, 0.08, 0.10\}$             \\
        \midrule
        GRPO+Ref Reset Pythia 1B TL;DR   & $\{0.2, 0.25, 0.3, 0.35, 0.4\}$          \\
        GRPO+Ref Reset Pythia 2.8B TL;DR & $\{0.25, 0.3, 0.35, 0.4\}$               \\
        \bottomrule
    \end{tabular}
    \label{tab:kl_optimization_tldr}
\end{table}

\subsection{Gradient Regularization on Alpaca Farm}
We perform on the AlpacaFarm dataset \citep{dubois_alpacafarm_2023} with Qwen2.5 models \citep{qwen_qwen25_2025}.
Similar to the TL;DR experiments we first train an SFT model, sample pairs of responses from it, and obtain pairwise comparisons from GPT4.1 Nano.
We then train an RM based on these comparisons, initialized from the SFT model.

\begin{table}[h]
    \centering
    \caption{KL penalty values considered in hyperparameter optimization for Alpaca GPT4.1 Nano experiments.}
    \begin{tabular}{l c}
        \toprule
        Experiment             & Hyperparameters                                                            \\
        \midrule
        GR Qwen 2.5 0.5B       & $\gamma \in \{1 \times 10^{-1},\, 1 \times 10^{-2},\, 1 \times 10^{-3}\}$  \\
        GR Qwen 2.5 1.5B       & $\gamma \in \{1 \times 10^{-1},\, 1 \times 10^{-2},\,  1 \times 10^{-3}\}$ \\
        GR Qwen 2.5 3B         & $\gamma \in \{3 \times 10^{-3}\}$                                          \\
        \midrule
        KL Qwen 2.5 0.5B       & $\beta \in \{0.03, 0.05, 0.07, 0.1, 0.15\}$                                \\
        KL Qwen 2.5 1.5B       & $\beta \in \{0.03, 0.05, 0.07, 0.1, 0.15\}$                                \\
        KL Qwen 2.5 3B         & $\beta \in \{0.05, 0.1\}$                                                  \\
        \midrule
        KL+Reset Qwen 2.5 0.5B & $\beta \in \{0.2, 0.4, 0.5\}$                                              \\
        KL+Reset Qwen 2.5 1.5B & $\beta \in \{0.3, 0.4, 0.5\}$                                              \\
        KL+Reset Qwen 2.5 3B   & $\beta \in \{0.2, 0.3, 0.4\}$                                              \\
        \bottomrule
    \end{tabular}
    \label{tab:kl_optimization_alpaca}
\end{table}

\subsection{Reasoning Experiments}
\begin{table}[]
    \caption{Hyperparameters in reasoning experiments with GR}
    \centering
    \begin{tabular}{l c c}
        \toprule
        Dataset             & GSM8K                                & MATH \\
        \midrule
        Optimizer           & \multicolumn{2}{c}{Adam}                    \\
        LR                  & \multicolumn{2}{c}{$5\times10^{-6}$}        \\
        Adam $\beta_1$      & \multicolumn{2}{c}{0.9}                     \\
        Adam $\beta_2$      & \multicolumn{2}{c}{0.999}                   \\
        Batchsize           & 256                                  & 1024 \\
        Rollouts per Prompt & \multicolumn{2}{c}{8}                       \\
        Temperature         & \multicolumn{2}{c}{0.7}                     \\
        GR $\varepsilon$    & \multicolumn{2}{c}{$10^{-3}$}               \\
        GR $\gamma$         & \multicolumn{2}{c}{$10^{-3}$}               \\
        Gradient Clipping   & \multicolumn{2}{c}{1.0}                     \\
        Output Length       & 768                                  & 1024 \\
        \bottomrule
    \end{tabular}
    \label{tab:hyperparam_reason_gr}
\end{table}

\begin{table}[]
    \caption{Hyperparameters in reasoning experiments with KL penalty or no penalty}
    \centering
    \begin{tabular}{l c c}
        \toprule
        Dataset             & GSM8K                                & MATH \\
        \midrule
        Optimizer           & \multicolumn{2}{c}{Adam}                    \\
        LR                  & \multicolumn{2}{c}{$3\times10^{-6}$}        \\
        Adam $\beta_1$      & \multicolumn{2}{c}{0.9}                     \\
        Adam $\beta_2$      & \multicolumn{2}{c}{0.999}                   \\
        Batchsize           & 256                                  & 1024 \\
        Rollouts per Prompt & \multicolumn{2}{c}{8}                       \\
        Temperature         & \multicolumn{2}{c}{0.7}                     \\
        KL penalty $\beta$  & 0.05/ 0.0                            & 0.0  \\
        Gradient Clipping   & \multicolumn{2}{c}{1.0}                     \\
        Output Length       & 768                                  & 1024 \\
        \bottomrule
    \end{tabular}
    \label{tab:hyperparam_reason_kl}
\end{table}
We use the hyperparameters as stated in Table \ref{tab:hyperparam_reason_gr} and Table \ref{tab:hyperparam_reason_kl}.
We did not tune $\beta$, $\epsilon$ or $\gamma$ and simply used conservative values from the RLHF experiments.
For both GR and no regularization we tried learning rates $3\times 10^{-6}$ and $5\times 10^{-6}$.
For GR, the larger learning rate was beneficial.

We follow the setup of \citet{wei_redit_2025}, with GSM8K using one-shot prompting, math uses chain-of-thought prompting, both with the default Qwen2.5-Instruct system instructions.
We also use the same reward terms, which are based on the rewards used by Open R1 \citep{openr1}.
For GSM8K, the prompt is \verb+Respond in the following format: <reasoning>...</reasoning><answer>...</answer>+, followed by a one-shot example.
The reward consists of a rule-based correctness reward ($[0,2]$), a formatting term checking whether the answer is an integer formatting reward ($\{0, 0.5\}$), a formatting term checking whether the formatting is exactly followed including whitespace ($\{0, 1.0\}$), and excluding whitespace ($\{0, 1.0\}$), and a formatting reward for matching XML tags ($\{0, 0.5\}$).

On MATH the prompt is simply the problem statement followed by \verb+Let's think step by step and output the+ \verb+final answer within \\boxed{}.+.
The used reward consists of a correctness reward ($[0,2]$) and a formatting reward rewarding up to three of the following with $1/3$ each: \textit{Step \#} keywords, Numbered lists, bullet points, \textit{First, Second, Next, Finally} keywords.
This is done to encourage reasoning \citep{openr1}.

\subsection{LLM Judge Setup}
In the main text we use Qwen2.5 1.5B-Instruct \citep{qwen_qwen25_2025} as a judge with the following prompt:

\begin{minted}{text}
Judge the correctness of the answer and reasoning for the given problem.
The format is as follows:

<problem>
...
</problem>
<model_answer>
<reasoning>
...
</reasoning>
<answer>
...
</answer>
</model_answer>
<correct_solution>
...
</correct_solution>

You will reply with the following XML format:
<judgement>
...
</judgement>
<correctness_score>
...
</correctness_score>
<coherence_score>
...
</coherence_score>

The model_answer may contain mistakes in the reasoning, the final answer, and in the format.

Give a one sentence judgement on the model_answer, then you will give scores from 1 to 5 for correctness and for coherence of the reasoning trace.
\end{minted}

We additionally use 1-shot prompting with a correct example.
During training the agent is provided the correctness score scaled to $[0,2]$, along with the same formatting rewards used in the rule-based-reward setting.

\section{Additional Experiments}
\label{app:additional_experiments}

\subsection{Gradient Regularization $\gamma$, $\varepsilon$, and learning rate sweeps}
\label{app:hparam_robustness}

\begin{figure}
    \centering
    \begin{minipage}{.67\columnwidth}
        \centering
        \includegraphics[width=0.47\linewidth]{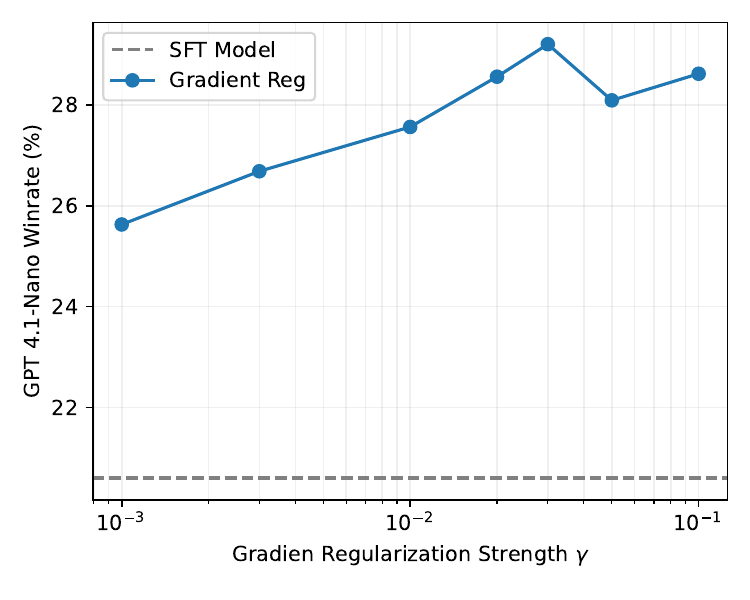}
        \includegraphics[width=0.47\linewidth]{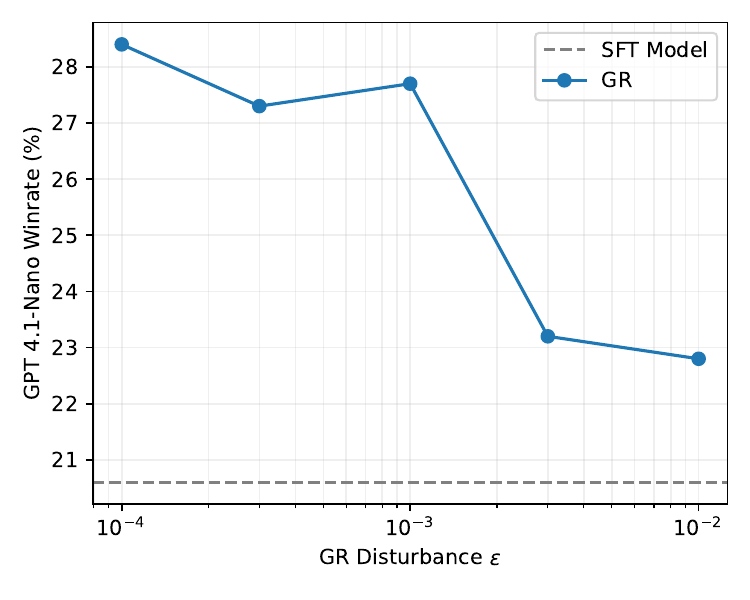}
        \captionof{figure}{\textbf{GR is rather predictable in choice of hyper-parameter.} Qwen2.5-1.5B on AlapcaFarm with different GR strengths $\gamma$ and fixed $\varepsilon=10^{-3}$ (left) with, and different disturbance strengths $\varepsilon$ and fixed $\gamma=3\times10^{-2}$(right) on the AlpacaFarm dataset, with early stopping.}
        \label{app:fig:gradregstrength_alpaca}
    \end{minipage}
    \hfill
    \begin{minipage}{.3\columnwidth}
        \centering
        \includegraphics[width=\linewidth]{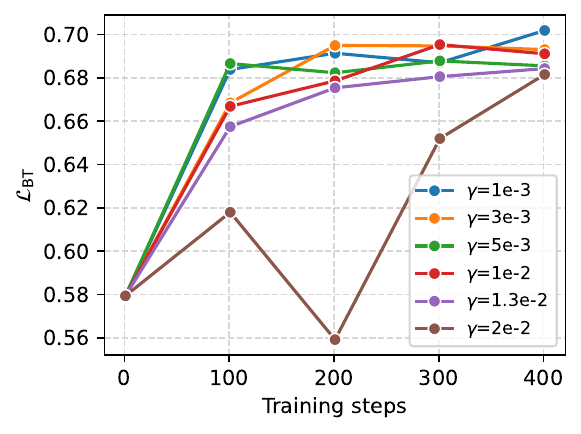}
        \captionof{figure}{\textbf{Strong GR can decrease BT loss below initial value.} BT loss $\mathcal{L}_\mathrm{BT}(\phi,\theta)$ during training of a Qwen 2.5 0.5B model on the TL;DR task with GR, using the Gold reward model.}
        \label{fig:qwentldrgradregstrength_btloss}
    \end{minipage}
\end{figure}
\begin{figure}
    \begin{minipage}{.45\columnwidth}
        \centering
        \includegraphics[width=0.7\columnwidth]{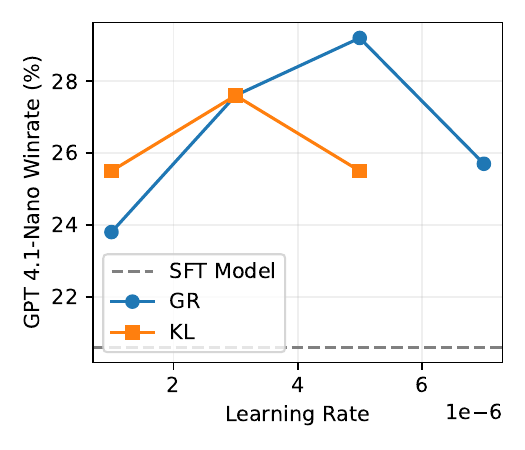}
        \caption{Learning rate sweep for KL penalty and GR when training a Qwen 2.5-1.5B model on the Alpacafarm dataset. KL $\beta=0.1$, GR $\gamma=3\times10^{-2}$}
        \label{fig:app:gpt_lr_sweep}
    \end{minipage}
    \hfill
    \begin{minipage}{.45\columnwidth}
        \centering
        \centering
        \includegraphics[width=0.7\linewidth]{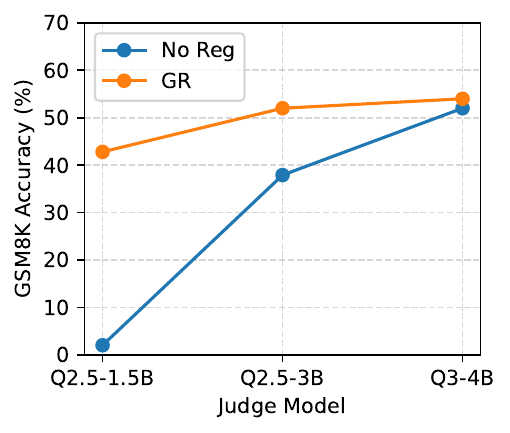}
        \captionof{figure}{\textbf{GR allows usage of cheaper judges to reach same performance.} Test accuracy when training a Qwen2.5-0.5B model on GSM8K with different judges.}
        \label{fig:app:llmjudge_ablation}
    \end{minipage}
\end{figure}

To evaluate the sensitivity of GR to the hyper-parameters controlling strength of the regularization $\gamma$ and the strength of the perturbation $\varepsilon$, we performed experiments on the Qwen2.5 1.5B model in the Alpaca GPT4.1 Nano setting.
We fix either $\gamma=3\times10^{-2}$ or $\varepsilon=10^{-3}$ and vary the respective other parameter.
We show the results in Figure \ref{app:fig:gradregstrength_alpaca}.
We find the performance to be rather predictable, facilitating hyper-parameter tuning.
We used the same disturbance strength $\varepsilon=10^{-3}$ for all experiments and did not find it necessary to tune it per task or model.
In the reasoning experiments we did not tune $\gamma$ and simply used $\gamma=10^{-3}$, however, in the RLHF experiments we found it necessary to tune $\gamma$ as described above, just like we found it necessary to tune $\beta$ for the KL penalty.

We also perform a learning rate sweep for the KL penalty and GR in the same Alpaca, Qwen2.5 1.5B setting. 
We use the best performing KL hyper-parameters from our main hyper-parameter optimization, i.e., $\beta=0.1$ for the KL penalty and $\gamma=0.03$ for GR.
The results are shown in \cref{fig:app:gpt_lr_sweep}. 

\subsection{Gradient regularization decreases BT loss}
\label{app:gradient_reg_bt_loss}

We also perform experiments in the synthetic TL;DR gold model setup, training a Qwen 2.5 0.5B model with different GR strengths $\gamma$.
The results in Figure \ref{fig:qwentldrgradregstrength_btloss} show that stronger regularization leads to a lower BT model loss $\mathcal{L}_\mathrm{BT}(\phi,\theta)$, evaluated on 4096 action pairs with labels from the gold reward model.
Interestingly, training with strong GR can result in a decreasing BT loss $\mathcal{L}_\mathrm{BT}$ beyond the initial BT loss, which we did not observe when utilizing either KL regularization or Reference Resets.
This illustrates the practical strength of the connection between gradient norm and PR accuracy.

\subsection{LLM judge ablation}
\label{app:llm_judge_ablation}
In the main text we are using Qwen2.5 1.5B-Instruct as judge. 
As we are using similarly sized policy models, we believe this could be a useful proxy for experiments in which LLMs are trained with equally large judge models.
However, to see whether GR is also useful with comparatively stronger judges, we additionally run experiments with Qwen2.5 3B-Instruct and Qwen3 4B Instruct-2507 \cite{yang_qwen3_2025} as judges.
We train a Qwen2.5 0.5B-Instruct model on GSM8K.
For Qwen3 we use \verb+Judgement:... Correctness_score:... Coherence_score:...+ as reply format instead of the xml tags, as we found Qwen3 to perform better with this format.
As shown in \cref{fig:app:llmjudge_ablation}, GR enables us to reach the same performance using a cheaper judge, potentially saving total computational cost.
In our setup with 8 GH200 GPUs, training with GR and the 1.5B judge took 84 minutes, while training without GR and the 4B judge took 88 minutes.
The additional cost of GR can thus be amortized by being able to use a cheaper judge.

\subsection{KL penalty schedule}
\begin{figure}
        \hspace{2cm}
        \includegraphics[height=4.5cm]{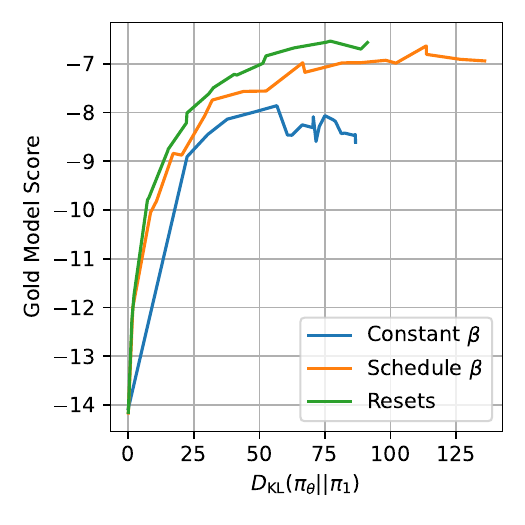}
        \hfill
        \includegraphics[height=4.5cm]{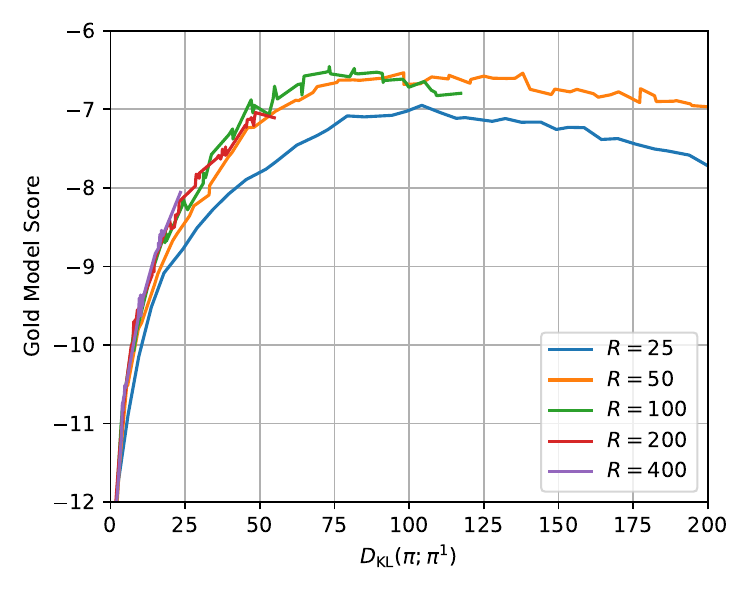}
        \hspace{2cm}
        \captionof{figure}{Pythia 1B on TL;DR task. Left: Reference Resets schedule ablation. A scheduled $\beta$ performs better than a constant value, however, it does not match the performance of full Reference Resets.
        Right: Steps per reset $R$ for GRPO + Reference Resets. A larger R is generally beneficial, but requires significantly more gradient steps.}
        \label{fig:app:reset_ablation_resetrate}
\end{figure}
While we have shown in the main text that Reference Resets perform better in RLHF than simply decreasing the strength of the KL penalty $\beta$ from the beginning, another hypothesis might be that decreasing the KL strength $\beta$ during training will match the results of Reference Resets.
As an additional baseline, we thus decrease the strength of the KL penalty to $\beta'=\beta/i$ in iteration $i$, while keeping the reference as $\pi^1$.
Thus in each iteration the optimal policy under Reference Resets and the scheduled $\beta$ is the same, assuming the previous iterations converged to their respective optimal policies.
As shown in Figure \ref{fig:app:reset_ablation_resetrate} (left), the schedule indeed yields a notable improvement over a fixed $\beta$, but also does not match the performance of Reference Resets.
We attribute this to the KL penalty explicitly keeping the policy close to the good region of the RM found in previous iteration, which a scheduled $\beta$ does not ensure.

\subsection{Steps per reset $R$}
Our theoretical derivation suggests that flatter minimum corresponds to a more accurate reward model.
Experiments show that the gradient norm keeps decreasing within each iteration even after the PR score is saturated.
Thus, we expect training for more steps $R$ per reset to improve performance at the cost of a higher computational expense.
We evaluate different values for $R$ when training a Pythia 1B model on the summarization task and show the results in Figure \ref{fig:app:reset_ablation_resetrate} (right), training each model for 1500 total steps.
Indeed, we find more steps $R$ to lead to a better KL-Gold-Reward tradeoff and perhaps a better asymptotic reward.
However, for high values such as $R=400$ the computational cost becomes prohibitively expensive, such that we use $R=200$ in experiments unless otherwise specified.

\subsection{Ablations}
\begin{table}[]
\caption{We evaluate different action reuse methods when training a Pythia 1B for the TL;DR task. In our main experiments we reuse actions without IS correction or resampling, which we find to perform best in experiments}
\label{tab:app:action_reuse}
\centering
\begin{tabular}{lcccc}
\toprule
GR Method       & $\gamma=3\times 10^{-4}$           & $\gamma=10^{-3}$           & $\gamma=3\times10^{-3}$           & $\gamma=10^{-2}$           \\
\midrule
Action Reuse    & 48.7\%          & \textbf{58.4\%} & \textbf{55.5\%} & \textbf{47.6\%} \\
Action Reuse IS & 52.2\%          & 49.4\%          & 54.6\%          & 42.9\%          \\
Action Resample & \textbf{53.0\%} & 49.3\%          & 48.5\%          & 39.3\% \\
\bottomrule
\end{tabular}
\end{table}

\begin{table}[]
\caption{In our main experiments we perturb only the transformer blocks, not the embedding and lm heads. In this experiment we train a Pythia 1B model for the TL;DR with either all parameters or only the transformer blocks being perturbed.}
\label{tab:app:allblock_perturb}
\centering
\begin{tabular}{lcccc}
\toprule
     & $\gamma=3\times 10^{-4}$ & $\gamma=10^{-3}$ & $\gamma=3\times10^{-3}$ & $\gamma=10^{-2}$      \\
     \midrule
Transformer blocks only & 48.7\%          & \textbf{58.4\%} & \textbf{55.5\%} & \textbf{47.6\%} \\
Perturb all parameters  & \textbf{48.9\%} & 49.1\%          & 53.6\%          & 47.0\% \\
\bottomrule
\end{tabular}
\end{table}

In our implementation of GR we made two key design decisions:
We reuse the actions generated by the unperturbed policy without resampling or an IS correction, and we only apply perturbations to the transformer blocks of the LLM.
We thus performed two additional experiments with Pythia 1B on the TL;DR task with the Gold model setup and one random seed each.

To investigate action reuse we evaluate two alternatives.
The first alternative is to sample new actions $a\sim \pi_{\phi + \varepsilon \nabla J(\phi,\theta)}(a|s)$.
The other option we evaluate is to correct for the difference between $\pi_{\phi + \varepsilon \nabla J(\phi,\theta)}(a|s)$ and $\pi_{\phi}(a|s)$ by using actions from the unperturbed policy with IS.
We use vanilla IS with density ratios clipped to $[0.5,1.5]$.
It is possible that other IS strategies could perform better.
The results are shown in \cref{tab:app:action_reuse}.

When using IS to correct for the distribution shift, the probability ratios are relatively well behaved.
For example for the perturbation strength $\epsilon=1e-3$, the importance weights have mean$\approx1.0$, std$\approx0.25$, min$\approx0.6$, max$\approx1.4$.
We believe this relative proximity of both distributions explains why skipping the importance sampling is reasonable. 
It also allows us to save two additional forward passes of the model necessary to calculate the probability ratio. 
While sampling new actions may seem like the best justified solution, in addition to almost doubling the runtime it also increases the variance of the gradient estimate, leading to a worse policy.

To evaluate whether perturbing only the transformer blocks has a large impact, we run an additional experiment in which we apply the perturbations to all parameters of the LLM.
The results are shown in \cref{tab:app:allblock_perturb}, showing a benefit to only perturbing the transformer layers.
We speculate that this is beneficial as perturbing the embedding layers may destabilize training by having an over-sized impact on the model.

\subsection{Runtime Comparison}
\begin{table}[]
\caption{Walltime and win-rate when early-stopping for Qwen 2.5 1.5B models trained on the AlpacaFarm dataset.}
\label{tab:app:runtime_comp}
\centering
\begin{tabular}{lccc}
\toprule
Method & Seed \#1                    & Seed \#2                   & Seed \#3             \\
\midrule
KL             & 29.9 mins / 25.1\%          & 41.1 mins / 27.6\%         & 19.8 mins / 38.5\%          \\
Reset          & 32.9 mins / 22.3\%          & 35.4 mins / 27.1\%         & \textbf{14.8 mins / 42.4\%} \\
GR             & \textbf{22.1 mins / 28.3\%} & \textbf{8.5 mins / 29.2\%} & 17.3 mins / 40.2\%   \\
\bottomrule
\end{tabular}
\end{table}
Using GR requires an additional forward and backward pass of the model, as well as time to save the unperturbed weights and reload them before applying the gradients.
Instead, a KL penalty only requires a forward pass of the reference model.
We thus compare the runtime before early stopping in our Qwen 2.5 1.5B AlpacaFarm RLHF experiments, trained on 8 GH200 GPUs.
The results in \cref{tab:app:runtime_comp} show that GR converges faster to a better solution than a KL penalty on all three seeds, and faster than Reference Resets on two out of three seeds.

\subsection{Reward hacking example}
An example output by a Qwen2.5 0.5B model trained without regularization hacking a Qwen2.5-1.5B-Instruct judge:
\begin{minted}{text}
<reasoning>
1. Given values:
   \\[
    \text{initial}} = 500, \text{spent_on_clothes}} = 20/100 \times 500 = 100, \text{remaining_after_clothes}} = 500 - 100 = 40}}, \text{spent_on_cds}} = 25/100 \times 40 = 100}}))))))))))))))))))))))))))))))>>
To find the final amount, we perform the following steps:
\[
    \text{final_left}} = 500 - 100 + 100 = 119}}
\))
</code>>
</p>>
</reasoning>
<answer>
119
</answer> 
\end{minted}

\section{Reference Resets in RLHF}
\label{sec:asymptotic_analysis}
\begin{figure}
    \centering
    \includegraphics[width=\linewidth]{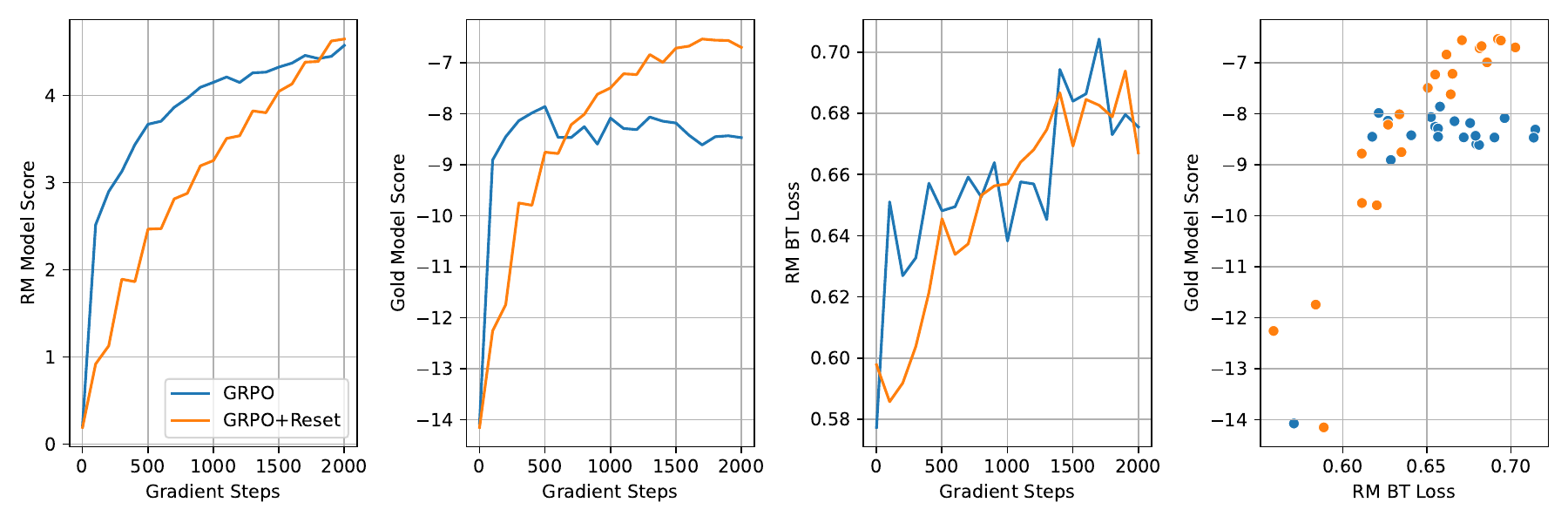}
    \caption{Proxy and gold reward for no resets (blue, $\beta=0.06$) and resets (orange, $\beta=0.3$). The proxy reward should thus match after 5 resets which here with $R=200$ corresponds to 1200 training steps. The achieved proxy reward is relatively similar at training step $t=1200$, however, the gold reward with resets is significantly higher.}
    \label{fig:app:refresetklstrength}
\end{figure}
To investigate the impact of reference resets in RLHF we provide an asymptotic argument, showing an equivalence between using reference resets and using a lower $\beta$.
While we show that this equivalence indeed shows up in experiments for the RM reward $R_\theta$, the achieved true reward $R^*$ is significantly higher when doing reference resets than when using a lower $\beta$.
This cannot be explained by the asymptotic argument, we thus continue by providing an analysis from the point of view of optimization dynamics.

\subsection{Asymptotic analysis}

It is well known (e.g. \citep[Appendix A.1]{rafailov_direct_2023}) that the optimal solution of the KL-regularized optimization problem  $\arg\max_\pi \prew(s,a) - D_\mathrm{KL}(\pi,\pi^1)$ is
\begin{equation}
    \pi(s,a) \propto \pi^1(s,a)\exp\left(\beta^{-1}\prew(s,a)\right) \,.
\end{equation}
With reference resets, we are solving this problem repeatedly, thus for iteration $k$
\begin{equation}
    \pi^k(s,a) \propto \pi^{k-1}(s,a)\exp\left(\beta^{-1}\prew(s,a)\right) \,.
\end{equation}

If we insert $\pi^{k-1}$ into this, we obtain
\begin{equation}
    \pi^k(s,a) \propto \pi^{1}(s,a)\exp\left((\beta/k)^{-1}\prew(s,a)\right) \,.
\end{equation}
Therefore, the optimal policy for Reference Resets with $k$ iterations and KL-penalty strength $\beta$ should be the same as the solution without resets with a weaker KL-penalty weight $\beta'=\beta/k$.
In experiments, we can indeed see a similar behavior when only looking at the RM reward, as shown in Figure \ref{fig:app:refresetklstrength} (left).
There, after 1200 steps the RM reward $\prew$ achieved with reset is roughly in the range of reward values without reference resets.
As this argument makes no statements about the true reward $R^*$, we might expect it to be similar for both methods as well.
Surprisingly, we instead find that the true reward achieved by Reference Resets is significantly higher than the true reward achieved in a single stage.
We believe this effect cannot be explained by an asymptotic analysis, thus motivating our optimization dynamics argument.
In Figure \ref{fig:app:refresetklstrength} (right), we also show that, by using Reference Resets, higher gold reward regions can be obtained for the same RM BT loss $\mathcal{L}_\mathrm{BT}$.

\end{document}

%% file: math_commands.tex
\usepackage{amsmath,amsfonts,bm}

\def\eqref#1{equation~\ref{#1}}

\def\1{\bm{1}}

\DeclareMathAlphabet{\mathsfit}{\encodingdefault}{\sfdefault}{m}{sl}
\SetMathAlphabet{\mathsfit}{bold}{\encodingdefault}{\sfdefault}{bx}{n}

\newcommand{\E}{\mathbb{E}}

\newcommand{\R}{\mathbb{R}}

%% file: figures/llm_judge/llmjudge_threepanel_bothmethods.tex
\begin{tikzpicture}

  \definecolor{darkgrey176}{RGB}{176,176,176}
  \definecolor{darkorange25512714}{RGB}{255,127,14}
  \definecolor{lightgrey204}{RGB}{204,204,204}
  \definecolor{steelblue31119180}{RGB}{31,119,180}

  \begin{groupplot}[
      group style={
          group size=3 by 1,
          horizontal sep=1.2cm,
        },
      width=\figonewidth,
      height=\figoneheight,
      ylabel style={yshift=-4pt},
      xlabel style={yshift=5pt},
      legend cell align={left},
      legend style={
          fill opacity=0.8,
          draw opacity=1,
          text opacity=1,
          at={(0.2,0.97)},
          anchor=north west,
          draw=lightgrey204,
          font=\footnotesize,
          inner sep=1.5pt,
          nodes={scale=0.75, transform shape},
        },
    ]
    \nextgroupplot[
      tick align=outside,
      tick pos=left,
      x grid style={darkgrey176},
      xlabel={Gradient Step},
      xmajorgrids,
      xmin=-48.95, xmax=1049.95,
      xminorgrids,
      xtick style={color=black},
      y grid style={darkgrey176},
      ylabel={Proxy Reward},
      ymin=0.69779052734375, ymax=0.95479736328125,
      ytick style={color=black}
    ]
    \addplot [semithick, noregcolor, dashed, mark=*, mark size=2, mark options={solid}]
    table {%
        1 0.717041015625
        101 0.769775390625
        201 0.737548828125
        301 0.77197265625
        401 0.82373046875
        501 0.911376953125
        601 0.91357421875
        701 0.88232421875
        801 0.83984375
        901 0.927978515625
        1000 0.943115234375
      };
    \addplot [semithick, gradregcolor, dashed, mark=*, mark size=2, mark options={solid}]
    table {%
        1 0.70947265625
        101 0.76611328125
        201 0.73828125
        301 0.740234375
        401 0.78955078125
        501 0.779052734375
        601 0.7998046875
        701 0.77880859375
        801 0.8076171875
        901 0.823486328125
        1000 0.814208984375
      };

    \nextgroupplot[
      tick align=outside,
      tick pos=left,
      x grid style={darkgrey176},
      xlabel={Gradient Step},
      xmajorgrids,
      xmin=-48.95, xmax=1049.95,
      xminorgrids,
      xtick style={color=black},
      y grid style={darkgrey176},
      ylabel={True Reward},
      ymin=-0.0052734375, ymax=0.4759765625,
      ytick style={color=black}
    ]
    \addplot [semithick, noregcolor, dashed, mark=*, mark size=2, mark options={solid}]
    table {%
        1 0.0341796875
        101 0.30078125
        201 0.10546875
        301 0.1103515625
        401 0.08984375
        501 0.0439453125
        601 0.0380859375
        701 0.0205078125
        801 0.025390625
        901 0.021484375
        1000 0.0166015625
      };
    \addplot [semithick, gradregcolor, dashed, mark=*, mark size=2, mark options={solid}]
    table {%
        1 0.0341796875
        101 0.3359375
        201 0.3916015625
        301 0.4033203125
        401 0.4033203125
        501 0.439453125
        601 0.4384765625
        701 0.435546875
        801 0.4541015625
        901 0.4375
        1000 0.4326171875
      };

    \nextgroupplot[
      log basis y={10},
      tick align=outside,
      tick pos=left,
      x grid style={darkgrey176},
      xlabel={Gradient Step},
      xmajorgrids,
      xmin=-48.9, xmax=1048.9,
      xminorgrids,
      xtick style={color=black},
      y grid style={darkgrey176},
      ylabel={Gradient Norm},
      ylabel style={yshift=-8pt},
      ymin=0.145382157117038, ymax=3.07718862456282,
      ymode=log,
      ytick style={color=black},
      ytick={0.1,1,10},
      yticklabels={,1.0,}, %
      extra y ticks={0.2,0.3,0.4,0.5,0.6,0.7,0.8,0.9,2,3},
      extra y tick labels={},
      extra y tick style={
          tick align=outside,
          major tick length=1.5pt, %
        },
    ]
    \addplot [semithick, noregcolor]
    table {%
        1 0.928218185901642
        2 2.67853680253029
        8 0.928218185901642
        9 0.652799993753433
        10 0.381799608469009
        12 0.402515649795532
        15 0.381799608469009
        17 0.379590705037117
        18 0.381799608469009
        20 0.379590705037117
        28 0.377381801605225
        29 0.359576985239983
        31 0.341772168874741
        32 0.333487614989281
        34 0.325203061103821
        42 0.318102926015854
        43 0.325203061103821
        44 0.318684130907059
        45 0.312165200710297
        46 0.311583995819092
        47 0.311002790927887
        48 0.31086365878582
        49 0.311002790927887
        53 0.31086365878582
        55 0.311002790927887
        56 0.31086365878582
        59 0.310724526643753
        61 0.308617100119591
        62 0.310724526643753
        65 0.310669049620628
        67 0.308561623096466
        68 0.306939721107483
        70 0.30546110868454
        74 0.300097063183784
        75 0.295605555176735
        77 0.294746026396751
        78 0.295605555176735
        79 0.294746026396751
        82 0.294746026396751
        84 0.294746026396751
        85 0.29814587533474
        86 0.294746026396751
        87 0.29814587533474
        89 0.302637383341789
        90 0.303436607122421
        91 0.303436607122421
        92 0.303436607122421
        95 0.303436607122421
        96 0.306260332465172
        100 0.306939721107483
        103 0.308991670608521
        105 0.313944339752197
        107 0.308991670608521
        109 0.313944339752197
        110 0.313944339752197
        113 0.313944339752197
        119 0.319817841053009
        120 0.321944400668144
        121 0.319401666522026
        123 0.319401666522026
        125 0.321944400668144
        128 0.321944400668144
        134 0.326775461435318
        137 0.334837004542351
        138 0.339060440659523
        141 0.344096332788467
        143 0.344096332788467
        144 0.348093137145042
        146 0.348093137145042
        148 0.348093137145042
        150 0.348093137145042
        151 0.358484700322151
        154 0.368109047412872
        156 0.368612408638
        158 0.368612408638
        160 0.368612408638
        162 0.3762276917696
        163 0.3762276917696
        167 0.3762276917696
        168 0.3762276917696
        170 0.368612408638
        171 0.368612408638
        173 0.377707451581955
        175 0.387590557336807
        177 0.387590557336807
        178 0.387590557336807
        181 0.387590557336807
        183 0.392350926995277
        186 0.397043004631996
        189 0.403429061174393
        191 0.403429061174393
        194 0.403429061174393
        195 0.40413199365139
        196 0.40413199365139
        199 0.39866416156292
        202 0.397043004631996
        203 0.393269151449203
        204 0.393269151449203
        206 0.393269151449203
        207 0.393269151449203
        209 0.380887597799301
        210 0.372306853532791
        214 0.372306853532791
        215 0.369609743356705
        216 0.369609743356705
        218 0.366826578974724
        219 0.362447902560234
        220 0.357669323682785
        221 0.355813980102539
        223 0.355813980102539
        226 0.355813980102539
        229 0.355813980102539
        231 0.355813980102539
        232 0.355813980102539
        233 0.35428549349308
        234 0.352013349533081
        235 0.350196406245232
        237 0.346930474042892
        239 0.342875376343727
        241 0.330784291028976
        242 0.330784291028976
        243 0.317751616239548
        245 0.315269812941551
        248 0.313772633671761
        249 0.313772633671761
        250 0.313772633671761
        252 0.311076372861862
        254 0.311076372861862
        259 0.311076372861862
        260 0.311076372861862
        261 0.311076372861862
        262 0.313385188579559
        266 0.313385188579559
        270 0.311076372861862
        271 0.311076372861862
        273 0.311076372861862
        274 0.311076372861862
        275 0.313385188579559
        277 0.313385188579559
        280 0.311076372861862
        283 0.311076372861862
        284 0.309086978435516
        286 0.306640043854713
        292 0.304615050554276
        293 0.304615050554276
        294 0.306640043854713
        301 0.304615050554276
        303 0.304615050554276
        304 0.304615050554276
        305 0.306640043854713
        312 0.304615050554276
        314 0.306640043854713
        317 0.309086978435516
        318 0.306640043854713
        319 0.306640043854713
        321 0.304615050554276
        322 0.304615050554276
        323 0.300956100225449
        324 0.300956100225449
        325 0.300956100225449
        326 0.300956100225449
        327 0.296099573373795
        328 0.294805407524109
        331 0.294805407524109
        333 0.287935376167297
        334 0.286915168166161
        338 0.291080966591835
        341 0.286915168166161
        342 0.291080966591835
        345 0.291080966591835
        347 0.286915168166161
        353 0.291080966591835
        354 0.294805407524109
        355 0.297810986638069
        358 0.294805407524109
        359 0.294805407524109
        361 0.297810986638069
        362 0.294805407524109
        363 0.294805407524109
        364 0.294805407524109
        365 0.291080966591835
        367 0.2898288667202
        370 0.287579864263535
        371 0.287579864263535
        373 0.287579864263535
        376 0.287579864263535
        380 0.284034103155136
        382 0.284034103155136
        383 0.284034103155136
        384 0.284034103155136
        386 0.281655132770538
        390 0.281655132770538
        391 0.281569510698318
        394 0.281569510698318
        397 0.281569510698318
        399 0.277459889650345
        401 0.271338567137718
        402 0.267924547195435
        405 0.267236083745956
        406 0.257240325212479
        407 0.246993124485016
        408 0.244543753564358
        411 0.242004543542862
        413 0.237349942326546
        414 0.228416778147221
        415 0.223364003002644
        416 0.223364003002644
        419 0.223364003002644
        422 0.223364003002644
        423 0.222844451665878
        424 0.222844451665878
        426 0.222844451665878
        428 0.222336702048779
        430 0.218080967664719
        431 0.218080967664719
        432 0.213748171925545
        442 0.218080967664719
        444 0.217477641999722
        445 0.217201538383961
        446 0.208441495895386
        448 0.208441495895386
        451 0.208441495895386
        452 0.202712595462799
        454 0.201759435236454
        455 0.207743167877197
        456 0.201759435236454
        457 0.2006895840168
        458 0.2006895840168
        459 0.2006895840168
        460 0.2016476765275
        462 0.2016476765275
        463 0.2006895840168
        469 0.2006895840168
        472 0.2006895840168
        473 0.19855223596096
        480 0.2006895840168
        481 0.19855223596096
        482 0.195414632558823
        486 0.196736961603165
        487 0.2006895840168
        489 0.196736961603165
        491 0.196736961603165
        496 0.197551980614662
        498 0.19338034838438
        499 0.19338034838438
        504 0.19338034838438
        505 0.19338034838438
        506 0.19173613935709
        507 0.19338034838438
        509 0.19173613935709
        510 0.19173613935709
        513 0.186834760010242
        516 0.186834760010242
        523 0.19173613935709
        526 0.19173613935709
        527 0.19338034838438
        530 0.19338034838438
        531 0.19173613935709
        532 0.19173613935709
        535 0.19173613935709
        536 0.189974211156368
        539 0.186498031020164
        540 0.180769763886929
        541 0.180769763886929
        542 0.189974211156368
        543 0.181890234351158
        544 0.18033367395401
        547 0.18033367395401
        548 0.18033367395401
        549 0.18033367395401
        551 0.18033367395401
        552 0.18033367395401
        553 0.18033367395401
        554 0.183140836656094
        555 0.17695014923811
        556 0.183140836656094
        559 0.183140836656094
        560 0.183140836656094
        562 0.17695014923811
        565 0.17695014923811
        568 0.183140836656094
        569 0.186947539448738
        571 0.186947539448738
        574 0.183140836656094
        578 0.17695014923811
        579 0.172862373292446
        581 0.17695014923811
        584 0.17695014923811
        587 0.183140836656094
        588 0.17695014923811
        590 0.172862373292446
        594 0.172862373292446
        596 0.172862373292446
        598 0.17695014923811
        600 0.17695014923811
        601 0.174402467906475
        602 0.168893121182919
        603 0.167019665241241
        606 0.167019665241241
        607 0.167019665241241
        613 0.168893121182919
        614 0.168893121182919
        615 0.167019665241241
        616 0.167019665241241
        618 0.167019665241241
        621 0.167019665241241
        624 0.167019665241241
        628 0.168893121182919
        629 0.169881939888
        630 0.181865803897381
        631 0.169881939888
        633 0.168893121182919
        634 0.168893121182919
        636 0.169881939888
        637 0.181865803897381
        640 0.181865803897381
        641 0.169881939888
        643 0.169881939888
        644 0.169036671519279
        645 0.169881939888
        648 0.175034008920193
        651 0.175034008920193
        652 0.182445712387562
        655 0.178148716688156
        657 0.178148716688156
        658 0.182445712387562
        659 0.18927750736475
        660 0.195449963212013
        662 0.18927750736475
        663 0.182445712387562
        665 0.182445712387562
        668 0.182445712387562
        670 0.18927750736475
        672 0.191712118685246
        675 0.198190703988075
        677 0.198190703988075
        678 0.199188105762005
        681 0.199188105762005
        682 0.199188105762005
        684 0.204179465770721
        685 0.211797393858433
        686 0.215286396443844
        690 0.224569782614708
        693 0.224569782614708
        699 0.240583665668964
        700 0.255572907626629
        701 0.265240594744682
        703 0.265240594744682
        704 0.272396832704544
        705 0.272723481059074
        707 0.272723481059074
        708 0.294787749648094
        709 0.317704930901527
        710 0.317704930901527
        711 0.32430262863636
        714 0.32430262863636
        717 0.317704930901527
        720 0.317704930901527
        722 0.317704930901527
        724 0.32430262863636
        726 0.32430262863636
        728 0.32430262863636
        730 0.32430262863636
        731 0.32430262863636
        733 0.32430262863636
        734 0.317704930901527
        735 0.315295428037643
        738 0.303888469934464
        740 0.298801675438881
        747 0.290768772363663
        750 0.280500784516335
        751 0.280500784516335
        752 0.280500784516335
        755 0.280500784516335
        756 0.279326930642128
        758 0.279326930642128
        760 0.279326930642128
        761 0.283753752708435
        763 0.283753752708435
        766 0.28733691573143
        767 0.28733691573143
        768 0.28733691573143
        770 0.28733691573143
        772 0.290508255362511
        774 0.290508255362511
        775 0.293940111994743
        776 0.293940111994743
        777 0.290508255362511
        778 0.290508255362511
        781 0.290508255362511
        789 0.290508255362511
        790 0.290508255362511
        792 0.290508255362511
        793 0.299108117818832
        794 0.307019084692001
        798 0.314629957079887
        801 0.322904586791992
        803 0.324364468455315
        804 0.328847274184227
        806 0.338375136256218
        807 0.358120605349541
        808 0.373506158590317
        810 0.377190604805946
        811 0.412795349955559
        813 0.451346680521965
        816 0.451346680521965
        818 0.452248573303223
        819 0.452248573303223
        820 0.452248573303223
        824 0.471561163663864
        827 0.490378454327583
        829 0.557786464691162
        830 0.557786464691162
        831 0.557786464691162
        833 0.557786464691162
        840 0.492621541023254
        841 0.492621541023254
        842 0.557786464691162
        846 0.557786464691162
        847 0.557786464691162
        848 0.557786464691162
        852 0.557786464691162
        854 0.556203842163086
        859 0.556203842163086
        862 0.556203842163086
        864 0.469483241438866
        865 0.430302515625954
        866 0.407437786459923
        868 0.392893418669701
        869 0.371772795915604
        872 0.343661606311798
        873 0.325528502464294
        877 0.318673372268677
        880 0.297056958079338
        881 0.279571652412415
        883 0.279571652412415
        888 0.279571652412415
        893 0.279571652412415
        894 0.273412510752678
        895 0.259415090084076
        896 0.251797959208488
        897 0.246586427092552
        898 0.24062217772007
        901 0.232101321220398
        902 0.232101321220398
        903 0.232101321220398
        905 0.232101321220398
        909 0.232101321220398
        910 0.225261859595776
        914 0.225261859595776
        917 0.225261859595776
        918 0.225261859595776
        919 0.232101321220398
        921 0.24062217772007
        922 0.232101321220398
        923 0.232101321220398
        927 0.232101321220398
        928 0.24062217772007
        929 0.24472588300705
        931 0.249146088957787
        932 0.251253962516785
        933 0.251253962516785
        935 0.252113953232765
        937 0.252113953232765
        939 0.253449559211731
        940 0.253449559211731
        941 0.252113953232765
        942 0.252113953232765
        943 0.252113953232765
        945 0.252113953232765
        947 0.251253962516785
        949 0.251253962516785
        950 0.251253962516785
        957 0.251253962516785
        960 0.249077424407005
        963 0.249077424407005
        965 0.249077424407005
        966 0.249077424407005
        967 0.249077424407005
        969 0.249077424407005
        975 0.249077424407005
        976 0.239605791866779
        977 0.231764361262321
        978 0.226183138787746
        979 0.231764361262321
        980 0.231764361262321
        981 0.231764361262321
        982 0.231764361262321
        984 0.226183138787746
        985 0.220311746001244
        986 0.220311746001244
        987 0.226183138787746
        990 0.220311746001244
        991 0.225415296852589
        993 0.217697106301785
        995 0.217697106301785
        996 0.225415296852589
        997 0.225415296852589
        998 0.231854952871799
        999 0.242711305618286
      };
    \addlegendentry{No GR}
    \addplot[semithick, gradregcolor]
    table {%
        1 1.09477019309998
        2 0.766486048698425
        8 0.438201904296875
        9 0.431751281023026
        10 0.425300657749176
        12 0.414871215820312
        15 0.404441773891449
        17 0.364808365702629
        18 0.358981877565384
        20 0.342078417539597
        28 0.325174957513809
        29 0.342078417539597
        31 0.325174957513809
        32 0.311090216040611
        34 0.325174957513809
        42 0.311090216040611
        43 0.325174957513809
        44 0.331352427601814
        45 0.337529897689819
        46 0.331352427601814
        47 0.325174957513809
        48 0.330839216709137
        49 0.32824632525444
        53 0.332374900579453
        55 0.336503475904465
        56 0.337016686797142
        59 0.336503475904465
        61 0.332374900579453
        62 0.32824632525444
        65 0.326710641384125
        67 0.317458659410477
        68 0.309766218066216
        70 0.317482516169548
        74 0.312028303742409
        75 0.312028303742409
        77 0.309766218066216
        78 0.309766218066216
        79 0.309766218066216
        82 0.30958317220211
        84 0.309766218066216
        85 0.312028303742409
        86 0.309766218066216
        87 0.312028303742409
        89 0.312028303742409
        90 0.309766218066216
        91 0.309766218066216
        92 0.30958317220211
        95 0.308840081095695
        96 0.308840081095695
        100 0.308840081095695
        103 0.308840081095695
        105 0.30441614985466
        107 0.300333797931671
        109 0.300333797931671
        110 0.300333797931671
        113 0.29650616645813
        119 0.293546050786972
        120 0.29243566095829
        121 0.293546050786972
        123 0.293546050786972
        125 0.293546050786972
        128 0.29243566095829
        134 0.29243566095829
        137 0.29243566095829
        138 0.291422545909882
        141 0.291422545909882
        143 0.291422545909882
        144 0.290416166186333
        146 0.288963973522186
        148 0.288963973522186
        150 0.285635992884636
        151 0.288963973522186
        154 0.285635992884636
        156 0.288963973522186
        158 0.288963973522186
        160 0.290416166186333
        162 0.288963973522186
        163 0.290416166186333
        167 0.290416166186333
        168 0.292532935738564
        170 0.296097993850708
        171 0.296097993850708
        173 0.296097993850708
        175 0.292532935738564
        177 0.28944243490696
        178 0.28944243490696
        181 0.283392697572708
        183 0.28944243490696
        186 0.283392697572708
        189 0.28944243490696
        191 0.283582463860512
        194 0.283582463860512
        195 0.283582463860512
        196 0.278984919190407
        199 0.283582463860512
        202 0.293007493019104
        203 0.283582463860512
        204 0.278867676854134
        206 0.284327283501625
        207 0.278867676854134
        209 0.278867676854134
        210 0.276067033410072
        214 0.276067033410072
        215 0.276067033410072
        216 0.276067033410072
        218 0.276067033410072
        219 0.276067033410072
        220 0.272699236869812
        221 0.270172461867332
        223 0.269909903407097
        226 0.269909903407097
        229 0.269909903407097
        231 0.271562516689301
        232 0.272699236869812
        233 0.276067033410072
        234 0.277485251426697
        235 0.277485251426697
        237 0.274117454886436
        239 0.271562516689301
        241 0.271562516689301
        242 0.271562516689301
        243 0.271562516689301
        245 0.267159387469292
        248 0.271562516689301
        249 0.271562516689301
        250 0.267159387469292
        252 0.271562516689301
        254 0.273876741528511
        259 0.271562516689301
        260 0.273876741528511
        261 0.273876741528511
        262 0.276169121265411
        266 0.276169121265411
        270 0.273876741528511
        271 0.276169121265411
        273 0.273876741528511
        274 0.273614183068275
        275 0.273614183068275
        277 0.276169121265411
        280 0.276169121265411
        283 0.27341128885746
        284 0.27341128885746
        286 0.268880277872086
        292 0.268831565976143
        293 0.268831565976143
        294 0.268831565976143
        301 0.270845457911491
        303 0.27341128885746
        304 0.278241619467735
        305 0.282174587249756
        312 0.282174587249756
        314 0.282174587249756
        317 0.282174587249756
        318 0.282652869820595
        319 0.285202100872993
        321 0.285202100872993
        322 0.285202100872993
        323 0.285202100872993
        324 0.287661030888557
        325 0.285202100872993
        326 0.287661030888557
        327 0.285202100872993
        328 0.285202100872993
        331 0.287661030888557
        333 0.285202100872993
        334 0.285202100872993
        338 0.285202100872993
        341 0.285202100872993
        342 0.281412273645401
        345 0.287209838628769
        347 0.287209838628769
        353 0.287209838628769
        354 0.287209838628769
        355 0.287209838628769
        358 0.28216677904129
        359 0.278828144073486
        361 0.278828144073486
        362 0.276087701320648
        363 0.276087701320648
        364 0.270641505718231
        365 0.270641505718231
        367 0.270641505718231
        370 0.271914079785347
        371 0.271914079785347
        373 0.268453821539879
        376 0.271914079785347
        380 0.271914079785347
        382 0.271914079785347
        383 0.268453821539879
        384 0.271914079785347
        386 0.271914079785347
        390 0.276087701320648
        391 0.278828144073486
        394 0.278828144073486
        397 0.274654522538185
        399 0.279921486973763
        401 0.279921486973763
        402 0.281850576400757
        405 0.281850576400757
        406 0.279921486973763
        407 0.279921486973763
        408 0.279921486973763
        411 0.279524177312851
        413 0.279524177312851
        414 0.281850576400757
        415 0.283605337142944
        416 0.281850576400757
        419 0.281850576400757
        422 0.281850576400757
        423 0.281850576400757
        424 0.281850576400757
        426 0.283605337142944
        428 0.283605337142944
        430 0.281114429235458
        431 0.281114429235458
        432 0.278585776686668
        442 0.277627065777779
        444 0.273025348782539
        445 0.277627065777779
        446 0.277627065777779
        448 0.273025348782539
        451 0.268211856484413
        452 0.268211856484413
        454 0.268211856484413
        455 0.273025348782539
        456 0.273025348782539
        457 0.273025348782539
        458 0.268211856484413
        459 0.273025348782539
        460 0.277627065777779
        462 0.273025348782539
        463 0.281737983226776
        469 0.281737983226776
        472 0.281737983226776
        473 0.273025348782539
        480 0.273025348782539
        481 0.281737983226776
        482 0.281737983226776
        486 0.273025348782539
        487 0.281003192067146
        489 0.273025348782539
        491 0.264201939105988
        496 0.273025348782539
        498 0.273025348782539
        499 0.281003192067146
        504 0.28004153072834
        505 0.28004153072834
        506 0.280753865838051
        507 0.275659620761871
        509 0.274919286370277
        510 0.274919286370277
        513 0.275659620761871
        516 0.275659620761871
        523 0.275659620761871
        526 0.274919286370277
        527 0.274919286370277
        530 0.274919286370277
        531 0.275659620761871
        532 0.274919286370277
        535 0.274919286370277
        536 0.274919286370277
        539 0.274336978793144
        540 0.274919286370277
        541 0.274919286370277
        542 0.275659620761871
        543 0.275659620761871
        544 0.274919286370277
        547 0.274919286370277
        548 0.274919286370277
        549 0.274919286370277
        551 0.275659620761871
        552 0.275659620761871
        553 0.274919286370277
        554 0.274919286370277
        555 0.274919286370277
        556 0.275631621479988
        559 0.275077313184738
        560 0.270438224077225
        562 0.276950553059578
        565 0.277851670980453
        568 0.277851670980453
        569 0.277851670980453
        571 0.277851670980453
        574 0.277851670980453
        578 0.280217066407204
        579 0.280217066407204
        581 0.280217066407204
        584 0.280217066407204
        587 0.287549495697021
        588 0.280217066407204
        590 0.28472363948822
        594 0.28472363948822
        596 0.284625008702278
        598 0.291957437992096
        600 0.291957437992096
        601 0.291957437992096
        602 0.296436339616776
        603 0.300729274749756
        606 0.303880542516708
        607 0.303880542516708
        613 0.303880542516708
        614 0.305007770657539
        615 0.305007770657539
        616 0.307585209608078
        618 0.307585209608078
        621 0.307585209608078
        624 0.314882293343544
        628 0.321308583021164
        629 0.314882293343544
        630 0.321308583021164
        631 0.321308583021164
        633 0.314882293343544
        634 0.321308583021164
        636 0.322430625557899
        637 0.322430625557899
        640 0.322430625557899
        641 0.315973088145256
        643 0.309491366147995
        644 0.307529777288437
        645 0.307529777288437
        648 0.307529777288437
        651 0.309491366147995
        652 0.309491366147995
        655 0.309491366147995
        657 0.315471842885017
        658 0.309491366147995
        659 0.309491366147995
        660 0.300521522760391
        662 0.289846435189247
        663 0.289846435189247
        665 0.294852718710899
        668 0.294852718710899
        670 0.294852718710899
        672 0.294852718710899
        675 0.29941663146019
        677 0.29941663146019
        678 0.294852718710899
        681 0.289846435189247
        682 0.289846435189247
        684 0.289846435189247
        685 0.289846435189247
        686 0.289005219936371
        690 0.289005219936371
        693 0.284849792718887
        699 0.280276998877525
        700 0.280276998877525
        701 0.280276998877525
        703 0.280276998877525
        704 0.281830534338951
        705 0.282984048128128
        707 0.282984048128128
        708 0.278411254286766
        709 0.282984048128128
        710 0.282984048128128
        711 0.282984048128128
        714 0.282984048128128
        717 0.288263246417046
        720 0.294110745191574
        722 0.294110745191574
        724 0.288263246417046
        726 0.283011749386787
        728 0.283011749386787
        730 0.28487166762352
        731 0.279592469334602
        733 0.278235033154488
        734 0.275933474302292
        735 0.278235033154488
        738 0.279592469334602
        740 0.279592469334602
        747 0.279592469334602
        750 0.28487166762352
        751 0.28487166762352
        752 0.279592469334602
        755 0.28487166762352
        756 0.28487166762352
        758 0.28487166762352
        760 0.286263704299927
        761 0.281012207269669
        763 0.279592469334602
        766 0.281012207269669
        767 0.284820899367332
        768 0.284820899367332
        770 0.281124591827393
        772 0.279732555150986
        774 0.279732555150986
        775 0.281124591827393
        776 0.280542895197868
        777 0.280542895197868
        778 0.280542895197868
        781 0.280542895197868
        789 0.280542895197868
        790 0.280542895197868
        792 0.280542895197868
        793 0.281822547316551
        794 0.28343240916729
        798 0.285849064588547
        801 0.285849064588547
        803 0.28343240916729
        804 0.28343240916729
        806 0.28343240916729
        807 0.28343240916729
        808 0.28343240916729
        810 0.285849064588547
        811 0.285849064588547
        813 0.285849064588547
        816 0.284446120262146
        818 0.285849064588547
        819 0.285849064588547
        820 0.29014253616333
        824 0.288754045963287
        827 0.288754045963287
        829 0.293675631284714
        830 0.296364650130272
        831 0.296364650130272
        833 0.293675631284714
        840 0.293040707707405
        841 0.293040707707405
        842 0.293040707707405
        846 0.293675631284714
        847 0.293675631284714
        848 0.293675631284714
        852 0.293675631284714
        854 0.293040707707405
        859 0.288747236132622
        862 0.288732782006264
        864 0.282826855778694
        865 0.280955016613007
        866 0.275048702955246
        868 0.282826855778694
        869 0.282826855778694
        872 0.282826855778694
        873 0.288732782006264
        877 0.280577331781387
        880 0.272799178957939
        881 0.280577331781387
        883 0.284878447651863
        888 0.272799178957939
        893 0.284878447651863
        894 0.284878447651863
        895 0.294069841504097
        896 0.288866400718689
        897 0.280710950493813
        898 0.272799178957939
        901 0.264860346913338
        902 0.272799178957939
        903 0.272799178957939
        905 0.272799178957939
        909 0.272799178957939
        910 0.272799178957939
        914 0.276908859610558
        917 0.280349686741829
        918 0.276908859610558
        919 0.280349686741829
        921 0.284151777625084
        922 0.286319434642792
        923 0.286319434642792
        927 0.2915228754282
        928 0.295407101511955
        929 0.2915228754282
        931 0.286319434642792
        932 0.284151777625084
        933 0.284151777625084
        935 0.284151777625084
        937 0.284151777625084
        939 0.284151777625084
        940 0.284151777625084
        941 0.284151777625084
        942 0.280349686741829
        943 0.280349686741829
        945 0.271569669246674
        947 0.271569669246674
        949 0.278775334358215
        950 0.278775334358215
        957 0.278775334358215
        960 0.278775334358215
        963 0.284197956323624
        965 0.292071938514709
        966 0.292071938514709
        967 0.293742224574089
        969 0.284197956323624
        975 0.275342613458633
        976 0.275342613458633
        977 0.270205661654472
        978 0.268113330006599
        979 0.268113330006599
        980 0.268113330006599
        981 0.265230849385262
        982 0.265230849385262
        984 0.265230849385262
        985 0.265230849385262
        986 0.265230849385262
        987 0.265230849385262
        990 0.265230849385262
        991 0.264385864138603
        993 0.264385864138603
        995 0.261205047369003
        996 0.264385864138603
        997 0.264385864138603
        998 0.261205047369003
        999 0.264385864138603
      };
    \addlegendentry{GR}
  \end{groupplot}

\end{tikzpicture}

%% file: intro_gradreg_rlhfvr.tex
Reinforcement Learning (RL) has become a key part of the post-training of language models (LMs) \citep{stiennon_learning_2020,deepseek_grpo_2024}.
In the case of RL from Human Feedback (RLHF) \citep{christiano_deep_2017}, we use RL to align the behavior of LMs with human preferences, which we cannot easily represent with a rule-based reward.
In the case of RL from Verifiable Feedback (RLVR) \citep{havrilla_teaching_2024, lambert_tulu_2024}, RL is used to improve the performance on tasks with verifiable rewards, such as mathematical reasoning or agentic tasks.
In RLHF, we use pairwise comparison data to train a reward model (RM), which then provides the reward estimates for policy updates.
In RLVR, we use a verifier such as a rule-based reward or another Large Language Model (LLM) to check if the model output matches the expected answer.
In both cases there is a desired behavior, corresponding to a desired true reward, which we try to approximate with the trained RM, rule-based reward, or LLM-as-a-judge reward.
We collectively refer to them as proxy rewards (PRs).
A key challenge of RL post-training is:
How can we ensure that when updating our policy with the PR, we actually improve the true reward, i.e., how can we ensure that our PR stays accurate as the policy changes throughout training?
One solution is to iteratively update the PR during training with new data from the current policy \citep{christiano_deep_2017}.
As this can be costly, another option is to use a Kullback-Leibler (KL) penalty to ensure the policy stays close to the initial model \citep{stiennon_learning_2020}.
In practice, the KL penalty slows down training and may not even improve performance \citep{gao_scaling_2023}, leading to recent papers abandoning it for tasks with rule-based rewards \citep{team_glm-45_2025,olmoteam_olmo_2025}.
However, this risks reward hacking with reward models and LLM-as-a-judge.

We thus aim to modify the policy update such that the policy not only maximizes the PR, but also maximizes the PR accuracy, without constraining it to stay close to the initial policy.
We argue that reward hacking often corresponds to sharp optima of the PR and propose gradient regularization (GR) as a solution. We illustrate an overview in \autoref{fig:figure2}.

In \autoref{sec:gradreg_theory}, we formalize this notion and show a theoretical connection between the flatness of an optimum and the PR accuracy at this optimum, as measured by the Bradley-Terry \citep[BT;][]{bradley_rank_1952} loss.
In \autoref{sec:ref_resets}, we show that this theoretical connection can be utilized to improve RLHF in practice. We leverage a recent method, Reference Resets \citep{liu_prorl_2025}, and demonstrate that it implicitly regularizes the gradient, providing a novel interpretation of this method. We empirically show this enables better RLHF performance on the TL;DR task.
In \autoref{sec:explicit_gradreg}, we propose a novel application of explicit gradient regularization to RL post-training. 
We demonstrate that it outperforms baselines in RLHF tasks with reward models as well as rule-based and LLM-as-a-judge RLVR math tasks. 
By using GR, we are able to train proficient LLM-as-a-Judge and RLHF models \textit{completely replacing the standard KL penalty}.

%% file: figures/theory_tikz_accurate.tex
\begin{subfigure}[t]{0.32\linewidth}
    \centering
    \begin{tikzpicture}[x=1cm,y=1.35cm]
        \draw[->,thick] (-2.3,0) -- (2.3,0) node[right] {$\phi$};
        \draw[->,thick] (-2.2,0) -- (-2.2,1.55) node[above] {$J(\phi,\theta)$};

        \draw[thick,noregcolor!70] plot[domain=-2.0:-0.4,samples=80] (\x,{1.25-1.5*(\x+1.2)^2});
        \draw[thick,gradregcolor!70] plot[domain=0.2:2.2,samples=80] (\x,{1.25-0.3*(\x-1.2)^2}); 

        \pgfmathsetmacro{\xsharp}{-1.55}
        \pgfmathsetmacro{\ysharp}{1.25 - 1.5*(\xsharp+1.2)^2}
        \pgfmathsetmacro{\msharp}{-3*(\xsharp+1.2)} 
        \pgfmathsetmacro{\xflat}{0.85}
        \pgfmathsetmacro{\yflat}{1.25 - 0.3*(\xflat-1.2)^2}
        \pgfmathsetmacro{\mflat}{-0.6*(\xflat-1.2)}

        \draw[thick,dashed,black] 
            ({\xsharp - 0.5}, {\ysharp - 0.5*\msharp}) -- 
            ({\xsharp + 0.3}, {\ysharp + 0.3*\msharp});
        
        \draw[thick,dashed,black] 
            ({\xflat - 0.6}, {\yflat - 0.6*\mflat}) -- 
            ({\xflat + 0.8}, {\yflat + 0.8*\mflat});

        \fill[noregcolor!70] (-1.2,1.25) circle (1.2pt);
        \fill[gradregcolor!70] (1.2,1.25) circle (1.2pt);
        \node[noregcolor!70,above] at (-1.2,1.27) {\scriptsize sharp};
        \node[gradregcolor!70,above] at (1.2,1.27) {\scriptsize flat};

    \end{tikzpicture}
\end{subfigure}
\begin{subfigure}[t]{0.65\linewidth}
    \centering
    \begin{tikzpicture}[x=1.25cm,y=1.35cm]
        \tikzset{
            actionpt/.style={circle,fill=black,inner sep=0.75pt},
            goodpair/.style={green!70!black,very thick,decorate,decoration={snake,amplitude=0.9pt,segment length=4pt}},
            badpair/.style={red!70,very thick,decorate,decoration={snake,amplitude=0.9pt,segment length=4pt}},
        }
        \pgfmathdeclarefunction{gauss}{3}{\pgfmathparse{exp(-((#1-#2)^2)/(2*(#3*#3)))}}
            \pgfmathdeclarefunction{R}{1}{\pgfmathparse{
                0.35
                + 0.90*gauss(#1,-1.50,0.22) %
                + 0.85*gauss(#1, 1.20,0.55) %
            + 0.14*gauss(#1,-1.92,0.050)
            - 0.12*gauss(#1,-1.85,0.045)
            + 0.16*gauss(#1,-1.78,0.040)
            - 0.14*gauss(#1,-1.72,0.038)
            + 0.12*gauss(#1,-1.66,0.036)
            + 0.12*gauss(#1,-1.64,0.045)
            - 0.1*gauss(#1,-1.58,0.040)
            + 0.12*gauss(#1,-1.52,0.038)
            - 0.10*gauss(#1,-1.46,0.040)
            + 0.1*gauss(#1,-1.40,0.045)
            + 0.18*gauss(#1, 1.62,0.060)
            - 0.20*gauss(#1, 1.69,0.060)
            + 0.015*max(0,1 - 0.9*gauss(#1,-1.50,0.45) - 0.9*gauss(#1,1.20,0.90))*sin(720*#1)
            }}

            \pgfmathsetmacro{\muSh}{-1.50}
            \pgfmathsetmacro{\muSm}{1.20}
            \pgfmathsetmacro{\sigSh}{0.30}
            \pgfmathsetmacro{\sigSm}{0.30}
            \pgfmathsetmacro{\densamp}{0.22}
            \pgfmathsetmacro{\densbase}{0.0}
            \begin{scope}
                \clip (-2.6,0) rectangle (2.6,0.34);
                \path[fill=noregcolor!60,opacity=0.22]
                    plot[domain=-2.6:2.6,samples=160] (\x,{\densbase + \densamp*gauss(\x,\muSh,\sigSh)})
                    -- (2.6,\densbase) -- (-2.6,\densbase) -- cycle;
                \path[fill=gradregcolor!60,opacity=0.18]
                    plot[domain=-2.6:2.6,samples=160] (\x,{\densbase + \densamp*gauss(\x,\muSm,\sigSm)})
                    -- (2.6,\densbase) -- (-2.6,\densbase) -- cycle;
                \draw[noregcolor!80!black,thin] plot[domain=-2.6:2.6,samples=160] (\x,{\densbase + \densamp*gauss(\x,\muSh,\sigSh)});
                \draw[gradregcolor!70,thin] plot[domain=-2.6:2.6,samples=160] (\x,{\densbase + \densamp*gauss(\x,\muSm,\sigSm)});
            \end{scope}
            \node[noregcolor!80!black,anchor=south] at (-1.5,0.15) {\scriptsize $\pi_{\phi^\mathrm{sharp}}$};
            \node[gradregcolor!70,anchor=south] at (1.2,0.15) {\scriptsize $\pi_{\phi^\mathrm{flat}}$};

            \draw[->,thick] (-2.9,0) -- (2.9,0) node[right] {$a$};
            \draw[->,thick] (-2.8,0) -- (-2.8,1.55) node[above] {$\prew(s,a)$};

            \pgfmathsetmacro{\DeltaSh}{0.2}
            \pgfmathsetmacro{\DeltaSm}{0.3}
            \pgfmathsetmacro{\ySh}{R(\muSh)}
            \pgfmathsetmacro{\ySm}{R(\muSm)}
            \draw[noregcolor!80!black,thick,dashed,rounded corners=2pt,fill=noregcolor!30,fill opacity=0.10]
                ({\muSh-\DeltaSh},{\ySh-0.18}) rectangle ({\muSh+\DeltaSh},{\ySh+0.05});
            \draw[gradregcolor!70,thick,dashed,rounded corners=2pt,fill=gradregcolor!25,fill opacity=0.10]
                ({\muSm-\DeltaSm},{\ySm-0.1}) rectangle ({\muSm+\DeltaSm},{\ySm+0.05});
            \node[gradregcolor!70,anchor=south west,inner sep=1pt] at ({\muSm-\DeltaSm},{\ySm+0.1}) {\scriptsize $L$-$D$ region};

            \draw[|<->|,noregcolor!80!black] ({\muSh-\DeltaSh},{\ySh+0.25}) -- ({\muSh+\DeltaSh},{\ySh+0.25}) node[midway,above] {\scriptsize $D$};
            \draw[|<->|,noregcolor!80!black] ({\muSh-\DeltaSh - 0.2},{\ySh-0.18}) -- ({\muSh-\DeltaSh -0.2},{\ySh+0.05}) node[midway,left] {\scriptsize $L$};

            \draw[thick] plot[domain=-2.6:2.6,samples=260] (\x,{R(\x)});
        
            \pgfmathsetmacro{\aone}{-1.6}
            \pgfmathsetmacro{\atwo}{-1.43}
            \pgfmathsetmacro{\athree}{-1.74}
            \pgfmathsetmacro{\afour}{-1.69}
            \pgfmathsetmacro{\afive}{1.0}
            \pgfmathsetmacro{\asix}{1.16}
            \pgfmathsetmacro{\aseven}{1.57}
            \pgfmathsetmacro{\aeight}{1.7}

        \pgfmathsetmacro{\yone}{R(\aone)}
        \pgfmathsetmacro{\ytwo}{R(\atwo)}
        \pgfmathsetmacro{\ythree}{R(\athree)}
        \pgfmathsetmacro{\yfour}{R(\afour)}
        \pgfmathsetmacro{\yfive}{R(\afive)}
        \pgfmathsetmacro{\ysix}{R(\asix)}
        \pgfmathsetmacro{\yseven}{R(\aseven)}
        \pgfmathsetmacro{\yeight}{R(\aeight)}

            \node[actionpt] at (\aone,\yone) {};
            \node[actionpt] at (\atwo,\ytwo) {};
            \node[actionpt] at (\athree,\ythree) {};
            \node[actionpt] at (\afour,\yfour) {};
            \node[actionpt] at (\afive,\yfive) {};
            \node[actionpt] at (\asix,\ysix) {};
            \node[actionpt] at (\aseven,\yseven) {};
            \node[actionpt] at (\aeight,\yeight) {};

        \draw[goodpair] (\aone,\yone) -- (\atwo,\ytwo);
        \draw[badpair] (\athree,\ythree) -- (\afour,\yfour);
        \draw[goodpair] (\afive,\yfive) -- (\asix,\ysix);
        \draw[badpair] (\aseven,\yseven) -- (\aeight,\yeight);

        \draw[goodpair] (-1.155,1.45) -- (-0.65,1.45);
        \node[green!70!black,anchor=west] at (-0.60,1.45) {\scriptsize flat};
        \draw[badpair] (-1.155,1.15) -- (-0.65,1.15);
        \node[red!70,anchor=west] at (-0.60,1.15) {\scriptsize too sharp};
    \end{tikzpicture}
\end{subfigure}

%% file: figures/reset_rlhf/klsweep_kl_div_gold.tex
\begin{tikzpicture}

  \definecolor{darkgrey176}{RGB}{176,176,176}
  \definecolor{darkslateblue9656104}{RGB}{96,56,104}
  \definecolor{darkslategrey443061}{RGB}{44,30,61}
  \definecolor{forestgreen4416044}{RGB}{44,160,44}
  \definecolor{grey14786133}{RGB}{147,86,133}
  \definecolor{lightgrey204}{RGB}{204,204,204}
  \definecolor{pink237209203}{RGB}{237,209,203}
  \definecolor{rosybrown188121151}{RGB}{188,121,151}
  \definecolor{tan218163172}{RGB}{218,163,172}

  \begin{axis}[
      legend cell align={left},
      legend style={
          fill opacity=0.8,
          draw opacity=1,
          text opacity=1,
          at={(0.99,0.01)},
          anchor=south east,
          draw=lightgrey204,
          nodes={scale=0.65, transform shape}
        },
      width=\columnwidth,
      height=5.0cm,
      tick align=outside,
      tick pos=left,
      x grid style={darkgrey176},
      xlabel={\(\displaystyle D_\mathrm{KL}(\pi_{\theta} || \pi_{1})\)},
      xmajorgrids,
      xmin=-5.42988667404279, xmax=113.836098119151,
      xtick style={color=black},
      y grid style={darkgrey176},
      ylabel={Gold Model Score},
      ymajorgrids,
      ymin=-14.8196455910802, ymax=-6.14080437272787,
      ytick style={color=black}
    ]

    \addplot [semithick, pink237209203]
    table {%
        -0.00161840242799371 -14.362435489893
        35.5872421264648 -8.8114646114409
        63.4236602783203 -8.42875792086124
        78.7981719970703 -8.15466377139091
        99.0851440429688 -8.2072289600037
        108.414916992188 -8.52969408035278
      };
    \addlegendentry{$\beta=0.03$}
    \addplot [semithick, tan218163172]
    table {%
        -0.00734988879412413 -14.3797677755356
        28.068567276001 -9.11655086278915
        48.0410690307617 -8.40381586924195
        59.6861000061035 -8.29989156872034
        73.3696823120117 -8.39712211489677
        79.7378082275391 -8.47857791185379
      };
    \addlegendentry{$\beta=0.04$}
    \addplot [semithick, rosybrown188121151]
    table {%
        -2.61887907981873e-06 -14.3824835121632
        21.7496337890625 -8.99304013699293
        36.7013282775879 -8.89611932635307
        42.1103744506836 -8.25632589124143
        58.3193740844727 -7.98104000091553
        59.7852172851562 -7.88074668869376
      };
    \addlegendentry{$\beta=0.05$}
    \addplot [semithick, grey14786133]
    table {%
        -0.00870554707944393 -14.3759696781635
        19.825439453125 -9.23144855350256
        25.001932144165 -8.69121189415455
        28.3607521057129 -8.33975937962532
        38.6519241333008 -8.47191405668855
        41.8406219482422 -8.16078261472285
      };
    \addlegendentry{$\beta=0.06$}
    \addplot [semithick, darkslateblue9656104]
    table {%
        0.00167872197926044 -14.4251528084278
        17.0443496704102 -9.3810452260077
        22.5011367797852 -8.93253830075264
        32.0126419067383 -8.34801438450813
        37.4442329406738 -8.21349490433931
        43.0324172973633 -7.92603488266468
      };
    \addlegendentry{$\beta=0.07$}
    \addplot [semithick, darkslategrey443061]
    table {%
        0.00242686434648931 -14.3578944504261
        12.6470203399658 -9.75085232709534
        19.2361087799072 -8.95913768664468
        20.0812950134277 -8.77931890264153
        28.1100082397461 -8.31398103944957
        34.7513847351074 -8.3995955735445
      };
    \addlegendentry{$\beta=0.08$}
    \addplot [thick, resetcolor]
    table {%
        0.00751598179340363 -14.1496953964233
        1.37341713905334 -12.2589888572693
        2.1926851272583 -11.7420902252197
        7.82723712921143 -9.75044345855713
        7.31887483596802 -9.79428386688232
        15.2289543151855 -8.75334358215332
        15.1029930114746 -8.78081321716309
        22.1497650146484 -8.21521472930908
        22.3944721221924 -8.01090812683105
        30.6734886169434 -7.61853229999542
        32.45751953125 -7.49325180053711
        40.410400390625 -7.21699297428131
        41.43896484375 -7.23198413848877
        52.5460205078125 -6.8395631313324
        51.309326171875 -6.99287444353104
        60.66064453125 -6.71399962902069
        63.275146484375 -6.67446804046631
        77.072265625 -6.53529715538025
        75.58935546875 -6.55781841278076
        91.26904296875 -6.56396356225014
        88.728515625 -6.69962310791016
      };
    \addlegendentry{Resets}
  \end{axis}

\end{tikzpicture}

%% file: figures/expl_gradreg/rmAccByModelSizeNoRegVSsftProperAx2.tex
\begin{tikzpicture}

  \definecolor{steelblue31119180}{RGB}{31,119,180}
  \definecolor{darkorange25512714}{RGB}{255,127,14}
  \definecolor{forestgreen4416044}{RGB}{44,160,44}
  \definecolor{crimson2143940}{RGB}{214,39,40}
  \definecolor{mediumpurple148103189}{RGB}{148,103,189}
  \definecolor{darkgrey176}{RGB}{176,176,176}
  \definecolor{lightgrey204}{RGB}{204,204,204}

  \begin{axis}[
      legend cell align={left},
      legend style={
          fill opacity=0.8,
          draw opacity=1,
          text opacity=1,
          at={(1.1,0.99)},
          anchor=north west,
          draw=lightgrey204,
          mark options={mark size=2},
          nodes={scale=0.8, transform shape}
        },
      width=5cm,
      height=4.5cm,
      tick align=outside,
      tick pos=left,
      x grid style={darkgrey176},
      xlabel={RM Accuracy},
      xmajorgrids,
      xmin=0.573095, xmax=0.709605,
      xtick style={color=black},
      y grid style={darkgrey176},
      ylabel={Winrate (\%)},
      ymajorgrids,
      ymin=19.525, ymax=43.175,
      ytick style={color=black},
      every axis plot/.append style={
        mark options={line width=1.5pt}
      },
    ]
    \addplot [draw=sftcolor, fill=sftcolor, mark size=3pt, mark=triangle, only marks]
    table{%
        x  y
        0.5793 20.8
        0.5948 20.6
        0.7034 21.7
      };
    \addlegendentry{SFT}
    \addplot [draw=noregcolor, fill=noregcolor, mark size=3pt, mark=square, only marks]
    table{%
        x  y
        0.5793 21.2
        0.5948 21.7
        0.7034 38.6
      };
    \addlegendentry{No Reg}
    \addplot [draw=klcolor, fill=klcolor, mark size=3pt, mark=+, only marks]
    table{%
        x  y
        0.5793 25.1
        0.5948 27.6
        0.7034 38.5
      };
    \addlegendentry{KL Reg}
    \addplot [draw=resetcolor, fill=resetcolor, mark size=3pt, mark=x, only marks]
    table{%
        x  y
        0.5793 22.3
        0.5948 27.1
        0.7034 42.1
      };
    \addlegendentry{Reference Reset}
    \addplot [
      draw=gradregcolor,
      fill=gradregcolor,
      mark size=3pt,
      mark=o,
      only marks
    ]
    table{%
        x  y
        0.5793 28.3
        0.5948 29.2
        0.7034 40.8
      };
    \addlegendentry{GR}
  \end{axis}
\end{tikzpicture}

%% file: figures/reasoning/gradreg_reason_accuracy.tex
\begin{tikzpicture}

  \definecolor{darkgrey176}{RGB}{176,176,176}
  \definecolor{darkorange25512714}{RGB}{255,127,14}
  \definecolor{lightgrey204}{RGB}{204,204,204}
  \definecolor{steelblue31119180}{RGB}{31,119,180}

  \begin{axis}[
      legend cell align={left},
      legend style={
          fill opacity=0.8,
          draw opacity=1,
          text opacity=1,
          at={(0.97,0.03)},
          anchor=south east,
          draw=lightgrey204,
          nodes={scale=0.8, transform shape}
        },
      width=\reasonwidth,
      height=\reasonheight,
      tick align=outside,
      tick pos=left,
      x grid style={darkgrey176},
      xlabel={Gradient Step},
      xmajorgrids,
      xmin=-48.95, xmax=1049.95,
      xminorgrids,
      xtick style={color=black},
      y grid style={darkgrey176},
      ylabel={Accuracy},
      ymajorgrids,
      ymin=0.00751953125, ymax=0.59404296875,
      yminorgrids,
      ytick style={color=black}
    ]

    \addplot [semithick, noregcolor]
    table {%
        1 0.0341796875
        101 0.513671875
        201 0.4951171875
        301 0.515625
        401 0.4775390625
        501 0.4453125
        601 0.451171875
        701 0.466796875
        801 0.470703125
        901 0.451171875
        1000 0.4775390625
      };
    \addlegendentry{No Reg}
    \addplot [semithick, gradregcolor]
    table {%
        1 0.0341796875
        101 0.4609375
        201 0.5048828125
        301 0.5185546875
        401 0.5234375
        501 0.5517578125
        601 0.5546875
        701 0.541015625
        801 0.5673828125
        901 0.5517578125
        1000 0.5595703125
      };
    \addlegendentry{GR}
  \end{axis}
\end{tikzpicture}

%% file: figures/reasoning/gradreg_reason_formatrew.tex
\begin{tikzpicture}

  \definecolor{darkgrey176}{RGB}{176,176,176}
  \definecolor{darkorange25512714}{RGB}{255,127,14}
  \definecolor{grey}{RGB}{128,128,128}
  \definecolor{lightgrey204}{RGB}{204,204,204}
  \definecolor{steelblue31119180}{RGB}{31,119,180}

  \begin{axis}[
      legend cell align={left},
      legend style={
          fill opacity=0.8,
          draw opacity=1,
          text opacity=1,
          at={(0.97,0.03)},
          anchor=south east,
          draw=lightgrey204,
        },
      width=\reasonwidth,
      height=\reasonheight,
      tick align=outside,
      tick pos=left,
      x grid style={darkgrey176},
      xlabel={Gradient Step},
      xmajorgrids,
      xmin=-48.95, xmax=1049.95,
      xminorgrids,
      xtick style={color=black},
      y grid style={darkgrey176},
      ylabel={Format Reward},
      ymajorgrids,
      ymin=2.9, ymax=3.01,
      yminorgrids,
      ytick style={color=black}
    ]

    \addplot [semithick, noregcolor, forget plot]
    table {%
        1 0.544433592018322
        101 2.983154296875
        201 2.993408203125
        301 2.9901123046875
        401 2.9967041015625
        501 2.98011815128848
        601 2.98069532529917
        701 2.9968027472496
        801 2.99277540347248
        901 2.99619337223703
        1000 2.995215833158
      };
    \addplot [semithick, gradregcolor, forget plot]
    table {%
        1 0.544433592018322
        101 2.94249903605669
        201 2.95348829386057
        301 2.95994727837387
        401 2.95994141904521
        501 2.97825001284946
        601 2.98349805973703
        701 2.97389942687005
        801 2.98582130187424
        901 2.9965586066246
        1000 2.99192188784946
      };
    \addplot [grey, opacity=0.7, dashed, forget plot]
    table {%
        -48.95 3
        1049.95 3
      };
  \end{axis}

\end{tikzpicture}

%% file: figures/llm_judge/gradreg_llmjudge_acc_sidebyside3methods.tex
\begin{tikzpicture}

  \definecolor{darkgrey176}{RGB}{176,176,176}
  \definecolor{darkorange25512714}{RGB}{255,127,14}
  \definecolor{forestgreen4416044}{RGB}{44,160,44}
  \definecolor{lightgrey204}{RGB}{204,204,204}
  \definecolor{steelblue31119180}{RGB}{31,119,180}

  \begin{groupplot}[
      group style={
          group size=2 by 1,
          horizontal sep=1.3cm
        },
      width=0.42\columnwidth,
      height=4cm,
      ylabel style={yshift=-5pt},
    ]
    \nextgroupplot[
      tick align=outside,
      tick pos=left,
      x grid style={darkgrey176},
      xlabel={Gradient Steps},
      xmajorgrids,
      xmin=-48.95, xmax=1049.95,
      xminorgrids,
      xtick style={color=black},
      y grid style={darkgrey176},
      ylabel={LLM Judge Score},
      ymin=0.69779052734375, ymax=0.95479736328125,
      ytick style={color=black}
    ]
    \addplot [semithick, noregcolor, mark=*, mark size=2, mark options={solid}]
    table {%
        1 0.717041015625
        101 0.769775390625
        201 0.737548828125
        301 0.77197265625
        401 0.82373046875
        501 0.911376953125
        601 0.91357421875
        701 0.88232421875
        801 0.83984375
        901 0.927978515625
        1000 0.943115234375
      };
    \addplot [semithick, gradregcolor, mark=*, mark size=2, mark options={solid}]
    table {%
        1 0.70947265625
        101 0.76611328125
        201 0.73828125
        301 0.740234375
        401 0.78955078125
        501 0.779052734375
        601 0.7998046875
        701 0.77880859375
        801 0.8076171875
        901 0.823486328125
        1000 0.814208984375
      };
    \addplot [semithick, klcolor, mark=*, mark size=2, mark options={solid}]
    table {%
        1 0.726806640625
        101 0.7451171875
        201 0.74365234375
        301 0.80029296875
        401 0.804443359375
        501 0.799072265625
        601 0.821533203125
        701 0.825927734375
        801 0.820068359375
        901 0.84765625
        1000 0.840576171875
      };

    \nextgroupplot[
      legend cell align={left},
      legend style={
          fill opacity=0.8,
          draw opacity=1,
          text opacity=1,
          at={(1.1,0.5)},
          anchor=west,
          draw=lightgrey204,
          nodes={scale=0.8, transform shape}
        },
      tick align=outside,
      tick pos=left,
      x grid style={darkgrey176},
      xlabel={Gradient Steps},
      xmajorgrids,
      xmin=-48.95, xmax=1049.95,
      xminorgrids,
      xtick style={color=black},
      y grid style={darkgrey176},
      ylabel={True Accuracy},
      ymin=-0.0052734375, ymax=0.4759765625,
      ytick style={color=black}
    ]
    \addplot [semithick, noregcolor, mark=*, mark size=2, mark options={solid}]
    table {%
        1 0.0341796875
        101 0.30078125
        201 0.10546875
        301 0.1103515625
        401 0.08984375
        501 0.0439453125
        601 0.0380859375
        701 0.0205078125
        801 0.025390625
        901 0.021484375
        1000 0.0166015625
      };
    \addlegendentry{No Reg}
    \addplot [semithick, gradregcolor, mark=*, mark size=2, mark options={solid}]
    table {%
        1 0.0341796875
        101 0.3359375
        201 0.3916015625
        301 0.4033203125
        401 0.4033203125
        501 0.439453125
        601 0.4384765625
        701 0.435546875
        801 0.4541015625
        901 0.4375
        1000 0.4326171875
      };
    \addlegendentry{GR}
    \addplot [semithick, klcolor, mark=*, mark size=2, mark options={solid}]
    table {%
        1 0.0341796875
        101 0.3076171875
        201 0.212890625
        301 0.2939453125
        401 0.267578125
        501 0.2783203125
        601 0.228515625
        701 0.279296875
        801 0.2578125
        901 0.1875
        1000 0.271484375
      };
    \addlegendentry{KL}
  \end{groupplot}

\end{tikzpicture}

%% file: figures/llm_judge/llmjudge_gradnormcomp_withrulebasedrew_sidebyside.tex
\begin{tikzpicture}

  \definecolor{darkgrey176}{RGB}{176,176,176}
  \definecolor{lightgrey204}{RGB}{204,204,204}

  \newcommand{\judgegradwidth}{0.48\columnwidth}
  \newcommand{\judgegradheight}{4cm}
  \newcommand{\judgegradsep}{1cm}

  \begin{groupplot}[
      group style={
          group size=2 by 1,
          horizontal sep=\judgegradsep,
        },
      width=\judgegradwidth,
      height=\judgegradheight,
      title style={yshift=-5pt}
    ]
    \nextgroupplot[
      tick align=outside,
      tick pos=left,
      title={No Regularization},
      x grid style={darkgrey176},
      xlabel={Gradient Step},
      xmajorgrids,
      xmin=-48.95, xmax=1049.95,
      xminorgrids,
      xtick style={color=black},
      y grid style={darkgrey176},
      ylabel={Accuracy},
      ymin=-0.0052734375, ymax=0.4759765625,
      ytick style={color=black},
      ytick={0.0,0.1,0.2,0.3,0.4},
      yticklabels={0.0,0.1,0.2,0.3,0.4},
    ]
    \addplot [semithick, noregcolor, mark=*, mark size=2, mark options={solid}]
    table {%
        1 0.0341796875
        101 0.30078125
        201 0.10546875
        301 0.1103515625
        401 0.08984375
        501 0.0439453125
        601 0.0380859375
        701 0.0205078125
        801 0.025390625
        901 0.021484375
        1000 0.0166015625
      };
    \nextgroupplot[
      tick align=outside,
      tick pos=left,
      title={GR},
      x grid style={darkgrey176},
      xlabel={Gradient Step},
      xmajorgrids,
      xmin=-48.95, xmax=1049.95,
      xminorgrids,
      xtick style={color=black},
      y grid style={darkgrey176},
      ylabel={},
      ymin=-0.0052734375, ymax=0.4759765625,
      ytick style={color=black},
      ytick={0.0,0.1,0.2,0.3,0.4},
      yticklabels=\empty
    ]
    \addplot [semithick, noregcolor, mark=*, mark size=2, mark options={solid}]
    table {%
        1 0.0341796875
        101 0.3359375
        201 0.3916015625
        301 0.4033203125
        401 0.4033203125
        501 0.439453125
        601 0.4384765625
        701 0.435546875
        801 0.4541015625
        901 0.4375
        1000 0.4326171875
      };
  \end{groupplot}

  \begin{axis}[
      width=\judgegradwidth,
      height=\judgegradheight,
      at={(group c1r1.south west)},
      anchor=south west,
      axis y line=right,
      axis x line=none,
      axis line style={-},
      log basis y={10},
      tick align=outside,
      xmin=-48.95, xmax=1049.95,
      xtick pos=left,
      xtick style={color=black},
      y grid style={darkgrey176},
      ymin=0.145382157117038, ymax=3.07718862456282,
      ymode=log,
      ytick pos=right,
      ytick style={color=black},
      yticklabel style={anchor=west},
      ytick={0.1,1,10},
      yticklabels=\empty,
      extra y ticks={0.2,0.3,0.4,0.5,0.6,0.7,0.8,0.9,2,3},
      extra y tick labels={},
      extra y tick style={
          tick align=outside,
          major tick length=1.5pt, %
        },
    ]
    \addplot [semithick, gradregcolor]
    table {%
        1 0.928218185901642
        2 2.67853680253029
        8 0.928218185901642
        9 0.652799993753433
        10 0.381799608469009
        12 0.402515649795532
        15 0.381799608469009
        17 0.379590705037117
        18 0.381799608469009
        20 0.379590705037117
        28 0.377381801605225
        29 0.359576985239983
        31 0.341772168874741
        32 0.333487614989281
        34 0.325203061103821
        42 0.318102926015854
        43 0.325203061103821
        44 0.318684130907059
        45 0.312165200710297
        46 0.311583995819092
        47 0.311002790927887
        48 0.31086365878582
        49 0.311002790927887
        53 0.31086365878582
        55 0.311002790927887
        56 0.31086365878582
        59 0.310724526643753
        61 0.308617100119591
        62 0.310724526643753
        65 0.310669049620628
        67 0.308561623096466
        68 0.306939721107483
        70 0.30546110868454
        74 0.300097063183784
        75 0.295605555176735
        77 0.294746026396751
        78 0.295605555176735
        79 0.294746026396751
        82 0.294746026396751
        84 0.294746026396751
        85 0.29814587533474
        86 0.294746026396751
        87 0.29814587533474
        89 0.302637383341789
        90 0.303436607122421
        91 0.303436607122421
        92 0.303436607122421
        95 0.303436607122421
        96 0.306260332465172
        100 0.306939721107483
        103 0.308991670608521
        105 0.313944339752197
        107 0.308991670608521
        109 0.313944339752197
        110 0.313944339752197
        113 0.313944339752197
        119 0.319817841053009
        120 0.321944400668144
        121 0.319401666522026
        123 0.319401666522026
        125 0.321944400668144
        128 0.321944400668144
        134 0.326775461435318
        137 0.334837004542351
        138 0.339060440659523
        141 0.344096332788467
        143 0.344096332788467
        144 0.348093137145042
        146 0.348093137145042
        148 0.348093137145042
        150 0.348093137145042
        151 0.358484700322151
        154 0.368109047412872
        156 0.368612408638
        158 0.368612408638
        160 0.368612408638
        162 0.3762276917696
        163 0.3762276917696
        167 0.3762276917696
        168 0.3762276917696
        170 0.368612408638
        171 0.368612408638
        173 0.377707451581955
        175 0.387590557336807
        177 0.387590557336807
        178 0.387590557336807
        181 0.387590557336807
        183 0.392350926995277
        186 0.397043004631996
        189 0.403429061174393
        191 0.403429061174393
        194 0.403429061174393
        195 0.40413199365139
        196 0.40413199365139
        199 0.39866416156292
        202 0.397043004631996
        203 0.393269151449203
        204 0.393269151449203
        206 0.393269151449203
        207 0.393269151449203
        209 0.380887597799301
        210 0.372306853532791
        214 0.372306853532791
        215 0.369609743356705
        216 0.369609743356705
        218 0.366826578974724
        219 0.362447902560234
        220 0.357669323682785
        221 0.355813980102539
        223 0.355813980102539
        226 0.355813980102539
        229 0.355813980102539
        231 0.355813980102539
        232 0.355813980102539
        233 0.35428549349308
        234 0.352013349533081
        235 0.350196406245232
        237 0.346930474042892
        239 0.342875376343727
        241 0.330784291028976
        242 0.330784291028976
        243 0.317751616239548
        245 0.315269812941551
        248 0.313772633671761
        249 0.313772633671761
        250 0.313772633671761
        252 0.311076372861862
        254 0.311076372861862
        259 0.311076372861862
        260 0.311076372861862
        261 0.311076372861862
        262 0.313385188579559
        266 0.313385188579559
        270 0.311076372861862
        271 0.311076372861862
        273 0.311076372861862
        274 0.311076372861862
        275 0.313385188579559
        277 0.313385188579559
        280 0.311076372861862
        283 0.311076372861862
        284 0.309086978435516
        286 0.306640043854713
        292 0.304615050554276
        293 0.304615050554276
        294 0.306640043854713
        301 0.304615050554276
        303 0.304615050554276
        304 0.304615050554276
        305 0.306640043854713
        312 0.304615050554276
        314 0.306640043854713
        317 0.309086978435516
        318 0.306640043854713
        319 0.306640043854713
        321 0.304615050554276
        322 0.304615050554276
        323 0.300956100225449
        324 0.300956100225449
        325 0.300956100225449
        326 0.300956100225449
        327 0.296099573373795
        328 0.294805407524109
        331 0.294805407524109
        333 0.287935376167297
        334 0.286915168166161
        338 0.291080966591835
        341 0.286915168166161
        342 0.291080966591835
        345 0.291080966591835
        347 0.286915168166161
        353 0.291080966591835
        354 0.294805407524109
        355 0.297810986638069
        358 0.294805407524109
        359 0.294805407524109
        361 0.297810986638069
        362 0.294805407524109
        363 0.294805407524109
        364 0.294805407524109
        365 0.291080966591835
        367 0.2898288667202
        370 0.287579864263535
        371 0.287579864263535
        373 0.287579864263535
        376 0.287579864263535
        380 0.284034103155136
        382 0.284034103155136
        383 0.284034103155136
        384 0.284034103155136
        386 0.281655132770538
        390 0.281655132770538
        391 0.281569510698318
        394 0.281569510698318
        397 0.281569510698318
        399 0.277459889650345
        401 0.271338567137718
        402 0.267924547195435
        405 0.267236083745956
        406 0.257240325212479
        407 0.246993124485016
        408 0.244543753564358
        411 0.242004543542862
        413 0.237349942326546
        414 0.228416778147221
        415 0.223364003002644
        416 0.223364003002644
        419 0.223364003002644
        422 0.223364003002644
        423 0.222844451665878
        424 0.222844451665878
        426 0.222844451665878
        428 0.222336702048779
        430 0.218080967664719
        431 0.218080967664719
        432 0.213748171925545
        442 0.218080967664719
        444 0.217477641999722
        445 0.217201538383961
        446 0.208441495895386
        448 0.208441495895386
        451 0.208441495895386
        452 0.202712595462799
        454 0.201759435236454
        455 0.207743167877197
        456 0.201759435236454
        457 0.2006895840168
        458 0.2006895840168
        459 0.2006895840168
        460 0.2016476765275
        462 0.2016476765275
        463 0.2006895840168
        469 0.2006895840168
        472 0.2006895840168
        473 0.19855223596096
        480 0.2006895840168
        481 0.19855223596096
        482 0.195414632558823
        486 0.196736961603165
        487 0.2006895840168
        489 0.196736961603165
        491 0.196736961603165
        496 0.197551980614662
        498 0.19338034838438
        499 0.19338034838438
        504 0.19338034838438
        505 0.19338034838438
        506 0.19173613935709
        507 0.19338034838438
        509 0.19173613935709
        510 0.19173613935709
        513 0.186834760010242
        516 0.186834760010242
        523 0.19173613935709
        526 0.19173613935709
        527 0.19338034838438
        530 0.19338034838438
        531 0.19173613935709
        532 0.19173613935709
        535 0.19173613935709
        536 0.189974211156368
        539 0.186498031020164
        540 0.180769763886929
        541 0.180769763886929
        542 0.189974211156368
        543 0.181890234351158
        544 0.18033367395401
        547 0.18033367395401
        548 0.18033367395401
        549 0.18033367395401
        551 0.18033367395401
        552 0.18033367395401
        553 0.18033367395401
        554 0.183140836656094
        555 0.17695014923811
        556 0.183140836656094
        559 0.183140836656094
        560 0.183140836656094
        562 0.17695014923811
        565 0.17695014923811
        568 0.183140836656094
        569 0.186947539448738
        571 0.186947539448738
        574 0.183140836656094
        578 0.17695014923811
        579 0.172862373292446
        581 0.17695014923811
        584 0.17695014923811
        587 0.183140836656094
        588 0.17695014923811
        590 0.172862373292446
        594 0.172862373292446
        596 0.172862373292446
        598 0.17695014923811
        600 0.17695014923811
        601 0.174402467906475
        602 0.168893121182919
        603 0.167019665241241
        606 0.167019665241241
        607 0.167019665241241
        613 0.168893121182919
        614 0.168893121182919
        615 0.167019665241241
        616 0.167019665241241
        618 0.167019665241241
        621 0.167019665241241
        624 0.167019665241241
        628 0.168893121182919
        629 0.169881939888
        630 0.181865803897381
        631 0.169881939888
        633 0.168893121182919
        634 0.168893121182919
        636 0.169881939888
        637 0.181865803897381
        640 0.181865803897381
        641 0.169881939888
        643 0.169881939888
        644 0.169036671519279
        645 0.169881939888
        648 0.175034008920193
        651 0.175034008920193
        652 0.182445712387562
        655 0.178148716688156
        657 0.178148716688156
        658 0.182445712387562
        659 0.18927750736475
        660 0.195449963212013
        662 0.18927750736475
        663 0.182445712387562
        665 0.182445712387562
        668 0.182445712387562
        670 0.18927750736475
        672 0.191712118685246
        675 0.198190703988075
        677 0.198190703988075
        678 0.199188105762005
        681 0.199188105762005
        682 0.199188105762005
        684 0.204179465770721
        685 0.211797393858433
        686 0.215286396443844
        690 0.224569782614708
        693 0.224569782614708
        699 0.240583665668964
        700 0.255572907626629
        701 0.265240594744682
        703 0.265240594744682
        704 0.272396832704544
        705 0.272723481059074
        707 0.272723481059074
        708 0.294787749648094
        709 0.317704930901527
        710 0.317704930901527
        711 0.32430262863636
        714 0.32430262863636
        717 0.317704930901527
        720 0.317704930901527
        722 0.317704930901527
        724 0.32430262863636
        726 0.32430262863636
        728 0.32430262863636
        730 0.32430262863636
        731 0.32430262863636
        733 0.32430262863636
        734 0.317704930901527
        735 0.315295428037643
        737 0.315295428037643
        738 0.315295428037643
        740 0.319406121969223
        741 0.319406121969223
        743 0.315295428037643
        744 0.311259865760803
        745 0.306410521268845
        746 0.304882377386093
        749 0.304882377386093
        750 0.304882377386093
        752 0.302758187055588
        754 0.302758187055588
        759 0.302758187055588
        760 0.302758187055588
        761 0.302758187055588
        762 0.304303675889969
        766 0.304303675889969
        770 0.302758187055588
        771 0.302758187055588
        773 0.302758187055588
        774 0.302758187055588
        775 0.304303675889969
        777 0.304303675889969
        780 0.302758187055588
        783 0.302758187055588
        784 0.300862222909927
        786 0.300862222909927
        792 0.295781582593918
        793 0.295781582593918
        794 0.295781582593918
        801 0.295781582593918
        803 0.295781582593918
        804 0.295781582593918
        805 0.300862222909927
        812 0.295781582593918
        814 0.300862222909927
        817 0.302758187055588
        818 0.300862222909927
        819 0.300862222909927
        821 0.295781582593918
        822 0.295781582593918
        823 0.293749302625656
        824 0.293749302625656
        825 0.293749302625656
        826 0.293749302625656
        827 0.28955951333046
        828 0.28666079044342
        831 0.28666079044342
        833 0.280349686741829
        834 0.278615444898605
        838 0.28324431180954
        841 0.278615444898605
        842 0.28324431180954
        845 0.28324431180954
        847 0.278615444898605
        853 0.28324431180954
        854 0.28666079044342
        855 0.289681702852249
        858 0.28666079044342
        859 0.28666079044342
        861 0.289681702852249
        862 0.28666079044342
        863 0.28666079044342
        864 0.28666079044342
        865 0.28324431180954
        867 0.282174587249756
        870 0.280349686741829
        871 0.280349686741829
        873 0.280349686741829
        876 0.280349686741829
        880 0.276884824037552
        882 0.276884824037552
        883 0.276884824037552
        884 0.276884824037552
        886 0.274605274200439
        890 0.274605274200439
        891 0.274460047483444
        894 0.274460047483444
        897 0.274460047483444
        899 0.270268380641937
        901 0.263554751873016
        902 0.258155196905136
        905 0.255526304245949
        906 0.246352329850197
        907 0.238411873579025
        908 0.237084776163101
        911 0.236467957496643
        913 0.233676880598068
        914 0.226673200130463
        915 0.223788067698479
        916 0.223788067698479
        919 0.223788067698479
        922 0.223788067698479
        923 0.223788067698479
        924 0.223788067698479
        926 0.222980991005898
        928 0.222980991005898
        930 0.222164869308472
        932 0.217645019888878
        933 0.217645019888878
        934 0.213325217366219
        942 0.217645019888878
        944 0.218111649155617
        945 0.217645019888878
        946 0.209584355354309
        948 0.209584355354309
        951 0.209584355354309
        952 0.205153062343597
        954 0.203111171722412
        955 0.209584355354309
        956 0.203111171722412
        957 0.201870456337929
        958 0.201870456337929
        959 0.201870456337929
        960 0.203421384096146
        962 0.203421384096146
        963 0.201870456337929
        969 0.201870456337929
        972 0.201870456337929
        973 0.200243249535561
        980 0.201870456337929
        981 0.200243249535561
        982 0.197384834885597
        986 0.199125379323959
        987 0.201870456337929
        989 0.199125379323959
        991 0.199125379323959
        996 0.199104964137077
        998 0.19536292552948
        999 0.19536292552948
      };
  \end{axis}

  \begin{axis}[
      width=\judgegradwidth,
      height=\judgegradheight,
      at={(group c2r1.south west)},
      anchor=south west,
      axis y line=right,
      axis x line=none,
      axis line style={-},
      log basis y={10},
      tick align=outside,
      xmin=-48.95, xmax=1049.95,
      xtick pos=left,
      xtick style={color=black},
      y grid style={darkgrey176},
      ylabel={\textcolor{gradregcolor}{Gradient Norm}},
      ylabel style={yshift=5pt},
      ymin=0.145382157117038, ymax=3.07718862456282,
      ymode=log,
      ytick pos=right,
      ytick style={color=black},
      yticklabel style={anchor=west},
      ytick={0.1,1,10},
      yticklabels={,1.0,}, %
      extra y ticks={0.2,0.3,0.4,0.5,0.6,0.7,0.8,0.9,2,3},
      extra y tick labels={},
      extra y tick style={
          tick align=outside,
          major tick length=1.5pt, %
        },
    ]
    \addplot [semithick, gradregcolor]
    table {%
        1 1.09477019309998
        2 0.766486048698425
        8 0.438201904296875
        9 0.431751281023026
        10 0.425300657749176
        12 0.414871215820312
        15 0.404441773891449
        17 0.364808365702629
        18 0.358981877565384
        20 0.342078417539597
        28 0.325174957513809
        29 0.342078417539597
        31 0.325174957513809
        32 0.311090216040611
        34 0.325174957513809
        42 0.311090216040611
        43 0.325174957513809
        44 0.331352427601814
        45 0.337529897689819
        46 0.331352427601814
        47 0.325174957513809
        48 0.330839216709137
        49 0.32824632525444
        53 0.332374900579453
        55 0.336503475904465
        56 0.337016686797142
        59 0.336503475904465
        61 0.332374900579453
        62 0.32824632525444
        65 0.326710641384125
        67 0.317458659410477
        68 0.309766218066216
        70 0.317482516169548
        74 0.312028303742409
        75 0.312028303742409
        77 0.309766218066216
        78 0.309766218066216
        79 0.309766218066216
        82 0.30958317220211
        84 0.309766218066216
        85 0.312028303742409
        86 0.309766218066216
        87 0.312028303742409
        89 0.312028303742409
        90 0.309766218066216
        91 0.309766218066216
        92 0.30958317220211
        95 0.308840081095695
        96 0.308840081095695
        100 0.308840081095695
        103 0.308840081095695
        105 0.30441614985466
        107 0.300333797931671
        109 0.300333797931671
        110 0.300333797931671
        113 0.29650616645813
        119 0.293546050786972
        120 0.29243566095829
        121 0.293546050786972
        123 0.293546050786972
        125 0.293546050786972
        128 0.29243566095829
        134 0.29243566095829
        137 0.29243566095829
        138 0.291422545909882
        141 0.291422545909882
        143 0.291422545909882
        144 0.290416166186333
        146 0.288963973522186
        148 0.288963973522186
        150 0.285635992884636
        151 0.288963973522186
        154 0.285635992884636
        156 0.288963973522186
        158 0.288963973522186
        160 0.290416166186333
        162 0.288963973522186
        163 0.290416166186333
        167 0.290416166186333
        168 0.292532935738564
        170 0.296097993850708
        171 0.296097993850708
        173 0.296097993850708
        175 0.292532935738564
        177 0.28944243490696
        178 0.28944243490696
        181 0.283392697572708
        183 0.28944243490696
        186 0.283392697572708
        189 0.28944243490696
        191 0.283582463860512
        194 0.283582463860512
        195 0.283582463860512
        196 0.278984919190407
        199 0.283582463860512
        202 0.293007493019104
        203 0.283582463860512
        204 0.278867676854134
        206 0.284327283501625
        207 0.278867676854134
        209 0.278867676854134
        210 0.276067033410072
        214 0.276067033410072
        215 0.276067033410072
        216 0.276067033410072
        218 0.276067033410072
        219 0.276067033410072
        220 0.272699236869812
        221 0.270172461867332
        223 0.269909903407097
        226 0.269909903407097
        229 0.269909903407097
        231 0.271562516689301
        232 0.272699236869812
        233 0.276067033410072
        234 0.277485251426697
        235 0.277485251426697
        237 0.274117454886436
        239 0.271562516689301
        241 0.271562516689301
        242 0.271562516689301
        243 0.271562516689301
        245 0.267159387469292
        248 0.271562516689301
        249 0.271562516689301
        250 0.267159387469292
        252 0.271562516689301
        254 0.273876741528511
        259 0.271562516689301
        260 0.273876741528511
        261 0.273876741528511
        262 0.276169121265411
        266 0.276169121265411
        270 0.273876741528511
        271 0.276169121265411
        273 0.273876741528511
        274 0.273614183068275
        275 0.273614183068275
        277 0.276169121265411
        280 0.276169121265411
        283 0.27341128885746
        284 0.27341128885746
        286 0.268880277872086
        292 0.268831565976143
        293 0.268831565976143
        294 0.268831565976143
        301 0.270845457911491
        303 0.27341128885746
        304 0.278241619467735
        305 0.282174587249756
        312 0.282174587249756
        314 0.282174587249756
        317 0.282174587249756
        318 0.282652869820595
        319 0.285202100872993
        321 0.285202100872993
        322 0.285202100872993
        323 0.285202100872993
        324 0.287661030888557
        325 0.285202100872993
        326 0.287661030888557
        327 0.285202100872993
        328 0.285202100872993
        331 0.287661030888557
        333 0.285202100872993
        334 0.285202100872993
        338 0.285202100872993
        341 0.285202100872993
        342 0.281412273645401
        345 0.287209838628769
        347 0.287209838628769
        353 0.287209838628769
        354 0.287209838628769
        355 0.287209838628769
        358 0.28216677904129
        359 0.278828144073486
        361 0.278828144073486
        362 0.276087701320648
        363 0.276087701320648
        364 0.270641505718231
        365 0.270641505718231
        367 0.270641505718231
        370 0.271914079785347
        371 0.271914079785347
        373 0.268453821539879
        376 0.271914079785347
        380 0.271914079785347
        382 0.271914079785347
        383 0.268453821539879
        384 0.271914079785347
        386 0.271914079785347
        390 0.276087701320648
        391 0.278828144073486
        394 0.278828144073486
        397 0.274654522538185
        399 0.279921486973763
        401 0.279921486973763
        402 0.281850576400757
        405 0.281850576400757
        406 0.279921486973763
        407 0.279921486973763
        408 0.279921486973763
        411 0.279524177312851
        413 0.279524177312851
        414 0.281850576400757
        415 0.283605337142944
        416 0.281850576400757
        419 0.281850576400757
        422 0.281850576400757
        423 0.281850576400757
        424 0.281850576400757
        426 0.283605337142944
        428 0.283605337142944
        430 0.281114429235458
        431 0.281114429235458
        432 0.278585776686668
        442 0.277627065777779
        444 0.273025348782539
        445 0.277627065777779
        446 0.277627065777779
        448 0.273025348782539
        451 0.268211856484413
        452 0.268211856484413
        454 0.268211856484413
        455 0.273025348782539
        456 0.273025348782539
        457 0.273025348782539
        458 0.268211856484413
        459 0.273025348782539
        460 0.277627065777779
        462 0.273025348782539
        463 0.281737983226776
        469 0.281737983226776
        472 0.281737983226776
        473 0.273025348782539
        480 0.273025348782539
        481 0.281737983226776
        482 0.281737983226776
        486 0.273025348782539
        487 0.281003192067146
        489 0.273025348782539
        491 0.264201939105988
        496 0.273025348782539
        498 0.273025348782539
        499 0.281003192067146
        504 0.28004153072834
        505 0.28004153072834
        506 0.280753865838051
        507 0.275659620761871
        509 0.274919286370277
        510 0.274919286370277
        513 0.275659620761871
        516 0.275659620761871
        523 0.275659620761871
        526 0.274919286370277
        527 0.274919286370277
        530 0.274919286370277
        531 0.275659620761871
        532 0.274919286370277
        535 0.274919286370277
        536 0.274919286370277
        539 0.274336978793144
        540 0.274919286370277
        541 0.274919286370277
        542 0.275659620761871
        543 0.275659620761871
        544 0.274919286370277
        547 0.274919286370277
        548 0.274919286370277
        549 0.274919286370277
        551 0.275659620761871
        552 0.275659620761871
        553 0.274919286370277
        554 0.274919286370277
        555 0.274919286370277
        556 0.275631621479988
        559 0.275077313184738
        560 0.270438224077225
        562 0.276950553059578
        565 0.277851670980453
        568 0.277851670980453
        569 0.277851670980453
        571 0.277851670980453
        574 0.277851670980453
        578 0.280217066407204
        579 0.280217066407204
        581 0.280217066407204
        584 0.280217066407204
        587 0.287549495697021
        588 0.280217066407204
        590 0.28472363948822
        594 0.28472363948822
        596 0.284625008702278
        598 0.291957437992096
        600 0.291957437992096
        601 0.291957437992096
        602 0.296436339616776
        603 0.300729274749756
        606 0.303880542516708
        607 0.303880542516708
        613 0.303880542516708
        614 0.305007770657539
        615 0.305007770657539
        616 0.307585209608078
        618 0.307585209608078
        621 0.307585209608078
        624 0.314882293343544
        628 0.321308583021164
        629 0.314882293343544
        630 0.321308583021164
        631 0.321308583021164
        633 0.314882293343544
        634 0.321308583021164
        636 0.322430625557899
        637 0.322430625557899
        640 0.322430625557899
        641 0.315973088145256
        643 0.309491366147995
        644 0.307529777288437
        645 0.307529777288437
        648 0.307529777288437
        651 0.309491366147995
        652 0.309491366147995
        655 0.309491366147995
        657 0.315471842885017
        658 0.309491366147995
        659 0.309491366147995
        660 0.300521522760391
        662 0.289846435189247
        663 0.289846435189247
        665 0.294852718710899
        668 0.294852718710899
        670 0.294852718710899
        672 0.294852718710899
        675 0.29941663146019
        677 0.29941663146019
        678 0.294852718710899
        681 0.289846435189247
        682 0.289846435189247
        684 0.289846435189247
        685 0.289846435189247
        686 0.289005219936371
        690 0.289005219936371
        693 0.284849792718887
        699 0.280276998877525
        700 0.280276998877525
        701 0.280276998877525
        703 0.280276998877525
        704 0.281830534338951
        705 0.282984048128128
        707 0.282984048128128
        708 0.278411254286766
        709 0.282984048128128
        710 0.282984048128128
        711 0.282984048128128
        714 0.282984048128128
        717 0.288263246417046
        720 0.294110745191574
        722 0.294110745191574
        724 0.288263246417046
        726 0.283011749386787
        728 0.283011749386787
        730 0.28487166762352
        731 0.279592469334602
        733 0.278235033154488
        734 0.275933474302292
        735 0.278235033154488
        738 0.279592469334602
        740 0.279592469334602
        747 0.279592469334602
        750 0.28487166762352
        751 0.28487166762352
        752 0.279592469334602
        755 0.28487166762352
        756 0.28487166762352
        758 0.28487166762352
        760 0.286263704299927
        761 0.281012207269669
        763 0.279592469334602
        766 0.281012207269669
        767 0.284820899367332
        768 0.284820899367332
        770 0.281124591827393
        772 0.279732555150986
        774 0.279732555150986
        775 0.281124591827393
        776 0.280542895197868
        777 0.280542895197868
        778 0.280542895197868
        781 0.280542895197868
        789 0.280542895197868
        790 0.280542895197868
        792 0.280542895197868
        793 0.281822547316551
        794 0.28343240916729
        798 0.285849064588547
        801 0.285849064588547
        803 0.28343240916729
        804 0.28343240916729
        806 0.28343240916729
        807 0.28343240916729
        808 0.28343240916729
        810 0.285849064588547
        811 0.285849064588547
        813 0.285849064588547
        816 0.284446120262146
        818 0.285849064588547
        819 0.285849064588547
        820 0.29014253616333
        824 0.288754045963287
        827 0.288754045963287
        829 0.293675631284714
        830 0.296364650130272
        831 0.296364650130272
        833 0.293675631284714
        840 0.293040707707405
        841 0.293040707707405
        842 0.293040707707405
        846 0.293675631284714
        847 0.293675631284714
        848 0.293675631284714
        852 0.293675631284714
        854 0.293040707707405
        859 0.288747236132622
        862 0.288732782006264
        864 0.282826855778694
        865 0.280955016613007
        866 0.275048702955246
        868 0.282826855778694
        869 0.282826855778694
        872 0.282826855778694
        873 0.288732782006264
        877 0.280577331781387
        880 0.272799178957939
        881 0.280577331781387
        883 0.284878447651863
        888 0.272799178957939
        893 0.284878447651863
        894 0.284878447651863
        895 0.294069841504097
        896 0.288866400718689
        897 0.280710950493813
        898 0.272799178957939
        901 0.264860346913338
        902 0.272799178957939
        903 0.272799178957939
        905 0.272799178957939
        909 0.272799178957939
        910 0.272799178957939
        914 0.276908859610558
        917 0.280349686741829
        918 0.276908859610558
        919 0.280349686741829
        921 0.284151777625084
        922 0.286319434642792
        923 0.286319434642792
        927 0.2915228754282
        928 0.295407101511955
        929 0.2915228754282
        931 0.286319434642792
        932 0.284151777625084
        933 0.284151777625084
        935 0.284151777625084
        937 0.284151777625084
        939 0.284151777625084
        940 0.284151777625084
        941 0.284151777625084
        942 0.280349686741829
        943 0.280349686741829
        945 0.271569669246674
        947 0.271569669246674
        949 0.278775334358215
        950 0.278775334358215
        957 0.278775334358215
        960 0.278775334358215
        963 0.284197956323624
        965 0.292071938514709
        966 0.292071938514709
        967 0.293742224574089
        969 0.284197956323624
        975 0.275342613458633
        976 0.275342613458633
        977 0.270205661654472
        978 0.268113330006599
        979 0.268113330006599
        980 0.268113330006599
        981 0.265230849385262
        982 0.265230849385262
        984 0.265230849385262
        985 0.265230849385262
        986 0.265230849385262
        987 0.265230849385262
        990 0.265230849385262
        991 0.264385864138603
        993 0.264385864138603
        995 0.261205047369003
        996 0.264385864138603
        997 0.264385864138603
        998 0.261205047369003
        999 0.264385864138603
      };
  \end{axis}

\end{tikzpicture}

%% file: figures/llm_judge/judge_ablation.tex
\begin{tikzpicture}

\definecolor{darkgrey176}{RGB}{176,176,176}
\definecolor{darkorange25512714}{RGB}{255,127,14}
\definecolor{lightgrey204}{RGB}{204,204,204}
\definecolor{steelblue31119180}{RGB}{31,119,180}

\begin{axis}[
legend cell align={left},
legend style={
  fill opacity=0.8,
  draw opacity=1,
  text opacity=1,
  at={(0.98,0.02)},
  anchor=south east,
  draw=lightgrey204,
  nodes={scale=0.8, transform shape}
},
height=4.5cm,
width=0.7\columnwidth,
tick align=outside,
tick pos=left,
x grid style={darkgrey176},
xlabel={Judge Model},
xmajorgrids,
xmin=0.9, xmax=3.1,
xtick style={color=black},
xtick={1,2,3},
xticklabels={Qwen 2.5-1.5B,Qwen 2.5-3B,Qwen3-4B},
y grid style={darkgrey176},
ylabel={Accuracy},
ymajorgrids,
ymin=0, ymax=0.70,
ytick style={color=black}
]
\addplot [semithick, noregcolor, mark=*, mark size=2, mark options={solid}]
table {%
1 0.02
2 0.379
3 0.52
};
\addlegendentry{No Reg}
\addplot [semithick, gradregcolor, mark=*, mark size=2, mark options={solid}]
table {%
1 0.428
2 0.52
3 0.54
};
\addlegendentry{GR}
\end{axis}

\end{tikzpicture}